\documentclass[10pt]{article}

\usepackage[margin=1in]{geometry}

\usepackage[utf8]{inputenc}
\usepackage[T1]{fontenc}
\usepackage{hyperref}
\usepackage{url}
\usepackage[backend=biber, style=numeric, sorting=none, url=true, doi=true, eprint=true]{biblatex}
\usepackage{booktabs}
\usepackage{amsfonts}
\usepackage{nicefrac}
\usepackage{microtype}
\usepackage[dvipsnames]{xcolor}
\usepackage{enumitem}
\usepackage{subcaption}

\usepackage{graphicx}
\usepackage{tikz}
\usetikzlibrary{shapes.geometric}
\usepackage{mathtools}
\usepackage{amsthm}
\usepackage{thmtools,thm-restate}
\usepackage{aligned-overset}
\usepackage{enumitem}
\usepackage{caption}
\usepackage{wrapfig}
\usepackage{bm}
\usepackage{algorithm}
\usepackage{algpseudocode}
\usepackage{listings}

\definecolor{codegreen}{rgb}{0,0.6,0}
\definecolor{codegray}{rgb}{0.5,0.5,0.5}
\definecolor{codepurple}{rgb}{0.58,0,0.82}
\definecolor{backcolour}{rgb}{0.95,0.95,0.92}

\lstdefinestyle{mystyle}{
    backgroundcolor=\color{backcolour},   
    commentstyle=\color{codegreen},
    keywordstyle=\color{magenta},
    numberstyle=\tiny\color{codegray},
    stringstyle=\color{codepurple},
    basicstyle=\ttfamily\scriptsize,
    breakatwhitespace=false,         
    breaklines=true,                 
    captionpos=b,                    
    keepspaces=true,                 
    numbers=left,
    numbersep=3pt,                  
    showspaces=false,                
    showstringspaces=false,
    showtabs=false,                  
    tabsize=2,
    aboveskip=3pt,
    belowskip=3pt
}

\newcommand{\R}{\mathbb{R}}
\newcommand{\EE}{\mathbb{E}}
\newcommand{\F}{\mathcal{F}}

\newcommand{\norm}[1]{\left\lVert #1\right\rVert}
\newcommand{\op}{\mathrm{op}}
\newcommand{\polar}{\mathrm{polar}}

\newcommand{\tr}{\mathrm{tr}}

\usepackage{amsmath, amssymb, amsfonts, amsthm, mathtools}

\usepackage{bbm}

\usepackage{bm}

\usepackage{nicefrac}

\usepackage{hyperref}
\hypersetup{
    colorlinks=true,
    linkcolor=Blue,
    citecolor=teal,
    urlcolor=Blue
}
\usepackage[nameinlink, capitalize, noabbrev]{cleveref}

\newtheorem{theorem}{Theorem}[section]
\newtheorem{lemma}{Lemma}[section]
\newtheorem{corollary}{Corollary}[section]

\newtheorem{assumption}{Assumption}[section]

\newtheorem{remark}{Remark}
\newtheorem{hypothesis}{Hypothesis}[section]

\renewcommand{\leq}{\leqslant}
\renewcommand{\geq}{\geqslant}
\renewcommand{\le}{\leqslant}
\renewcommand{\ge}{\geqslant}

\title{Anytime Training with Schedule-Free Spectral Optimization}

\author{
  Anuj Apte\thanks{\texttt{anuj.apte@jpmchase.com}} \quad
  Pranav Deshpande \quad 
  Niraj Kumar \\
  Shouvanik Chakrabarti \quad
  Junhyung Lyle Kim\thanks{\texttt{lyle.kim@jpmchase.com}} \\[0.5em]
  \normalsize Global Technology Applied Research, JPMorganChase\\
  \normalsize New York, NY 10001, USA
}

\date{May 2026}

\begin{document}
\maketitle

\begin{abstract}
Standard neural network training relies on learning-rate schedules tied to a fixed horizon, leading to strong path dependence and costly re-tuning as data availability changes. Schedule-Free (SF) methods address this by removing explicit schedules, yet SF-AdamW, the current state-of-the-art anytime optimizer, consistently underperforms well-tuned AdamW baselines. We propose SF-NorMuon, a schedule-free spectral optimizer that closes this gap: with a single hyperparameter configuration, SF-NorMuon matches or exceeds tuned AdamW on 125M and 772M parameter language models across $1$--$8\times$ Chinchilla horizons, and trails horizon-aware, cosine-scheduled NorMuon by only $\sim 0.03$ nats, substantially narrowing the gap between anytime and scheduled spectral optimizers.
On the theoretical side, we prove a stationarity guarantee for schedule-free spectral dynamics and identify weight decay at the fast iterate as essential for long-horizon stability. SF-NorMuon enables practitioners to obtain high-quality checkpoints at any point during training without committing to a horizon in advance.
Fully matching horizon-aware spectral optimizers without a schedule remains an open problem, which we highlight as a direction for future work.
By closing the performance gap with tuned AdamW baselines, SF-NorMuon makes horizon-free optimization more practical, taking a step towards truly open-ended, continual learning.
\end{abstract}

\section{Introduction} 


In the standard machine learning paradigm, after curating a dataset, a model is trained and then evaluated on a held-out, previously unseen portion of the data \cite{Hastie2009,Goodfellow-et-al-2016,Bishop2024}. However, models are increasingly being deployed in scenarios where large amounts of data comparable to or even exceeding the training data are generated during the deployment period. Examples of this include large language models (LLMs) \cite{brown2020language, touvron2023, Wen2023, Gao2024} interacting with billions of users, and Vision-Language-Action (VLA) models controlling robots in industries \cite{Brohan2023, King2024, Finn2024}. As a result, judiciously leveraging the a priori unknown amount of data available at deployment time is crucial for improving model capabilities through continual learning \cite{Parisi2019,wang2024continuallearning}. However, beyond addressing catastrophic forgetting \cite{Kirkpatrick2017, Delange2021}, there is also an optimization obstacle to overcome \cite{degris2024step, meterez2026anytime, defazio2024road, kasimbeg2025far}.

Most existing training methods are not designed to be anytime, as they depend on learning rate schedules tied to a fixed training horizon and require extensive hyperparameter tuning under a predetermined compute budget \cite{loshchilov2017cosine, hoffmann2022, hu2024WSD}. For example, the standard cosine decay schedule has a strong path dependence. Consider training an LLaMA-2-style transformer with cosine decay as shown in \cref{fig:validation-loss}. The black dashed lines highlight the performance discrepancy: at the same token count, a run tuned for a shorter horizon significantly outperforms one tuned for a longer horizon, despite seeing identical data (e.g.: 2$\times$ vs 4$\times$ Chinchilla at 31B tokens for 772M model). This occurs because cosine decay ties the learning rate to the training horizon. At any intermediate point, runs with longer horizons have barely begun decaying, while shorter-horizon runs have already annealed to near zero, resulting in the big discrepancy in performances. Ideally, we want the model to learn as much as possible from a given set of training data regardless of how much data awaits in the future. 

\begin{figure}[htb]
    \centering
    \includegraphics[width=0.82\linewidth]
    {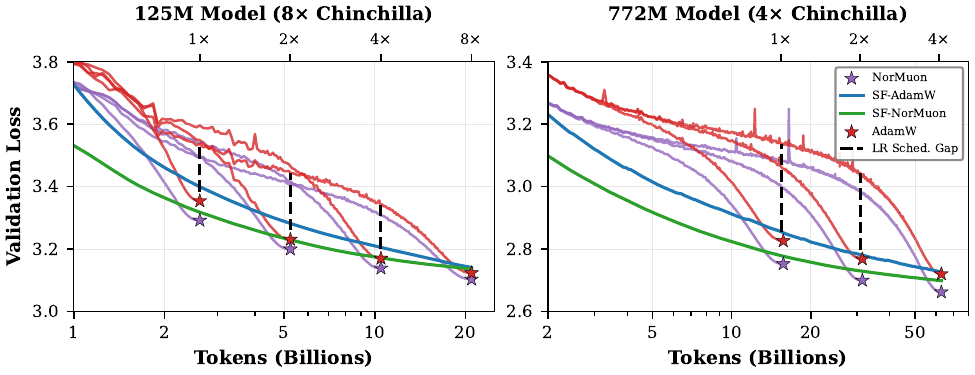}
    \caption{
    \textbf{Comparison of \textcolor{Green}{SF-NorMuon} (\textbf{this work}), \textcolor{RoyalBlue}{SF-AdamW}, and \textcolor{Red}{\textit{tuned} AdamW}, and \textcolor{Purple}{\textit{tuned} NorMuon} baselines for LLaMA-2-style transformers trained on FineWeb-100B.} \textit{Left:} 125M model with AdamW learning rate tuned per horizon with cosine schedule. \textit{Right:} 772M model with a single optimized AdamW configuration across horizons. 
    Dashed lines highlight that learning rate schedule for a long horizon is sub-optimal for a smaller token budget. 
    The grid of hyperparameters are summarized in \cref{tab:hyperparam-space}, and validation losses for different learning rates are illustrated in \cref{fig:lr-sweep-adamw}.
    \textcolor{Green}{SF-NorMuon} consistently outperforms \textcolor{RoyalBlue}{SF-AdamW}, and matches  \textcolor{Red}{\textit{tuned} AdamW} baselines, while horizon-aware \textcolor{Purple}{NorMuon} (cosine) remains ahead.
    }
    \label{fig:validation-loss}
\end{figure}

To that end, \textit{Schedule-Free} (SF) methods \cite{defazio2024road} were recently proposed, removing the need for an explicit learning-rate schedule. 
The SF framework maintains three iterate sequences.
In its most basic form, SF-SGD iterates as follows:
\begin{align}
    y_t &= \beta x_t + (1-\beta) z_t, \label{eq:SF-y} \\
    z_{t+1} &= z_t - \eta \nabla \mathcal{L}(y_t;\xi_t),  \label{eq:SF-z} \\
    x_{t+1} &= (1-c_{t+1}) x_t + c_{t+1} z_{t+1}, \label{eq:SF-x}
\end{align}
where $z_1 = x_1$ and $c_{t+1} = 1/(t+1)$. The $y$ sequence is where the gradient is evaluated at interpolated points, in spirit of Nesterov's accelerated method \cite{nesterov1983method, lan2012optimal}, the $z$ sequence is the \textit{fast iterate} taking the (stochastic) gradient direction with a constant step size $\eta$ throughout the training, and the $x$ sequence is the \textit{evaluation}
sequence where validation loss is computed.
Rewriting the update, we obtain 
\begin{equation}\label{eq:lr_effective_xz}
    x_{t+1} = x_{t} - \eta c_{t+1} \nabla \mathcal{L}(y_t;\xi_t) + c_{t+1} (z_{t+1} - x_{t}).
\end{equation}
Thus, even though $z$ moves at a constant rate, the effective learning rate for $x$ is $\eta_{\textrm{eff}}=\eta/(t+1)$, which decays over the duration of the training. Note that since $x$ is a linear combination of $y$ and $x$ it can be computed on the fly, and so it does not need additional memory overhead. 
An alternative but closely related approach for horizon-free training is simply training with a constant learning rate, and returning an some moving average\footnote{For instance, with $\beta=0$ in \eqref{eq:SF-y}, SF-SGD recovers constant SGD with uniform averaging.} of the weights \cite{hagele2024scalinglawscomputeoptimaltraining, li2025modelmergingpretraininglarge, meterez2026anytime, izmailov2018averaging}. 
Notably, SF-AdamW, the practical variant of SF-SGD, won the first prize in the self-tuning track of the \texttt{AlgoPerf} challenge \cite{Kasimbeg2025AlgoPerfResults}, demonstrating a new level of effectiveness for anytime training. For completeness, \cref{app:sf-review} reviews SF-SGD and SF-AdamW, and summarizes the convergence guarantee for SF-SGD.

Despite this success, as observed in \cite{hagele2024scaling, song2025through, semenov2026benchmarking},
SF-AdamW is unable to match the performance of well-tuned AdamW \cite{kingma2015adam,loshchilov2018decoupled}, the primary workhorse for optimizing neural networks. 
This is also confirmed in \cref{fig:validation-loss}, where SF-AdamW (blue) consistently lies above the well-tuned AdamW baselines (red stars, tuning grid in \cref{tab:hyperparam-space}) across all Chinchilla horizons.
Since the schedule-free averaging mechanism in \eqref{eq:SF-y}--\eqref{eq:SF-x} is agnostic to the base update rule applied at \eqref{eq:SF-z}, a  natural question is whether a different choice of base optimizer can close this gap.

In modern neural networks, the vast majority of learnable parameters reside in weight matrices that act linearly on incoming activations. However, AdamW treats weight matrices as flat vectors, discarding how they act on activations. 
The spectral norm provides a natural measure of the effect a weight matrix update has on the output activation, motivating recent interest in spectral optimization methods such as 
Muon \cite{jordan2024muon, jordan2024muon_repo} and its variants \cite{si2025adamuonadaptivemuonoptimizer, li2025normuonmakingmuonefficient, ahn2025dion,ahn2025dion2simplemethodshrink,an2025asgoadaptivestructuredgradient,ren2026rethinkinglanguagemodelscaling,amsel2026polarexpressoptimalmatrix,muon2boostingmuonadaptive,khaled2025muonbpfastermuonblockperiodic,kravatskiy2025kyfannormsbeyond}. 
Muon performs steepest descent under the spectral norm via the polar decomposition of the gradient computed by Newton-Schulz iterations, and has been successfully scaled to trillions of tokens \cite{liu2025muonscalablellmtraining, kimiteam2026kimik25visualagentic, glm5team2026glm5vibecodingagentic, singh2026arceetrinitylargetechnical}. Neural network loss surfaces are characterized by many nearly-flat directions corresponding to small Hessian eigenvalues \cite{Sagun2017, ghorbani19b}, and thus taking uniformly sized steps in all spectral directions via the polar decomposition enables Muon to extract more information from the training data per step \cite{davis2026spectralgradientupdateshelp}. Indeed, in our setting, cosine-scheduled NorMuon outperforms tuned AdamW at every horizon and both scales (\cref{fig:validation-loss}), making the spectral geometry a natural candidate base optimizer for schedule-free methods. This motivates the question:

\begin{center}
\textit{Does optimization under the spectral norm improve upon existing schedule-free methods, \\and how close can it come to well-tuned, horizon-dependent AdamW and NorMuon baselines?}    
\end{center}

Apart from the choice of base geometry, schedule-free methods exhibit a stability challenge that becomes critical for long horizon training. 
As shown in \cref{fig:weight-decay-temp}, the validation loss increases for long runs in absence of weight decay. This is in sharp contrast to the convex optimization theory underlying schedule-free methods, where convergence holds without regularization \cite{defazio2024road, schaipp2025surprising}. Moreover, weight decay at $y$ as originally proposed in \cite{defazio2024road}, is inadequate for the spectral case, where the instability is significantly amplified (right panel of \cref{fig:weight-decay-temp}). Therefore, in schedule-free methods weight decay plays a qualitatively different role: it is not merely a regularizer that improves generalization \cite{Krogh1991, dangelo2024, qiu2026}, but is \emph{necessary} for long-horizon stability.

Combining these two insights, spectral geometry to speed up training and weight decay at the fast iterate to ensure stability, our contributions are as follows:
\begin{itemize}[leftmargin=*, align=left]
\item We propose SF-SpectralSGD (with momentum), a schedule-free spectral method for matrix optimization, and prove a stationarity guarantee matching the $\widetilde{O}(T^{-1/4})$ rate achieved by recent work on convergence of Muon \cite{chang2026convergencemuon, shen2026convergenceanalysismuon, sato2025convergenceboundcriticalbatch, li2025noteconvergencemuon} (\cref{thm:sf_spectral_descent} and \cref{cor:main_sf_spectral_descent}).

\item For practical deep learning, we incorporate the per-neuron normalization of NorMuon \cite{li2025normuonmakingmuonefficient, si2025adamuonadaptivemuonoptimizer} and obtain SF-NorMuon, as summarized in \cref{algo:sf-normuon}. SF-NorMuon substantially improves over the state-of-the-art anytime optimizer SF-AdamW, matches tuned AdamW across training horizons, and comes within $0.025$--$0.036$ nats of cosine-scheduled NorMuon, which has access to the horizon in advance (c.f., \cref{fig:validation-loss}, \cref{tab:main-results}). Closing this remaining gap in a schedule-free manner remains an open problem. In particular, including explicit momentum buffer and per-neuron normalization is crucial in achieving good practical performance (c.f., \cref{fig:ablation}).  
\item We identify the key role of weight decay in schedule-free optimization. Unlike in standard optimizers, the fast iterate $z_t$ in \eqref{eq:SF-z} continues to move at constant learning rate throughout training, so weight decay is a necessity for long-horizon stability. For spectral optimization, the aggressive polar updates amplify the instability (\cref{fig:weight-decay-temp}), but weight decay directly applied to $z_t$ leads to stable training. We prove boundedness of all iterates (\cref{lem:Z-bounded}), and derive a closed-form steady-state characterization (\cref{lem:Z-steady-state}) that closely matches empirical training dynamics (\cref{fig:wd-rms}).

\end{itemize}

\section{Schedule-Free Spectral Optimization}\label{sec:sf-spec}

\begin{figure}
    \centering
    \includegraphics[width=0.82\linewidth]{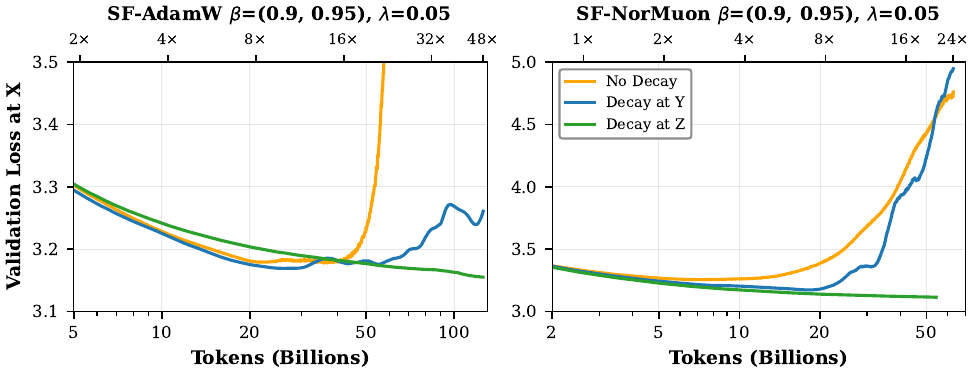}
    \caption{\textbf{Weight decay strategies for schedule-free optimizers.} 
    \textit{Left:} SF-AdamW with no decay (orange) diverges; decay at $Y$ (blue) exhibits the best performance up to $\sim$30B tokens (c.f., \cref{tab:main-results}), but eventually diverges.  
    Decay at $Z$ (green) yield stable training, but is suboptimal in the early phase.
    \textit{Right:} SF-NorMuon with no decay (orange) diverges faster; decay at $Y$ (blue) is also unstable. Decay at $Z$ (green) maintains stable and performs the best throughout, motivating the analysis in \cref{sec:sf-wd}.
    }
    \vspace{-10pt}
    \label{fig:weight-decay-temp}
\end{figure}

The fundamental design choice in a schedule-free optimizer is the geometry used for the fast update. When inputs and outputs are measured in the Euclidean norm \(\|\cdot\|_{2}\),
the induced operator norm on a matrix \(A \in \mathbb{R}^{m \times n}\) is
\begin{equation}
\|A\|_{2 \to 2}
\;:=\; \sup_{\|u\|_2 \le 1} \|Au\|_2
\;=\; \sigma_{\max}(A)
\;=\; \|A\|_{\op}.
\end{equation}
where \(\sigma_{\max}(A)\) denotes the largest singular value of \(A\), which is precisely the spectral norm. Geometrically, it measures the maximum factor by
which \(A\) can stretch any unit vector, making it the natural measure of how
much a weight perturbation can alter the network's activations. The analysis for other induced norms, and their corresponding optimizers is presented in \cref{app:app-metrized}.

For matrix-valued weights acting as linear maps on activations, the spectral norm is the natural metric, since it equals the operator norm governing the largest change an update induces on activations. The steepest-descent direction under a spectral-norm constraint solves \cite{bernstein2024oldoptimizernewnorm, pethick2025trainingdeeplearningmodels}
\begin{equation}
\textrm{argmin}_{\|\Delta W\|_{\op} \le \eta} \langle \nabla f(W), \Delta W \rangle_F.
\end{equation}
Let $\nabla f(W) = U \Sigma V^\top$ be the SVD with $U \in \mathbb{R}^{m \times r}$, $V \in \mathbb{R}^{n \times r}$ orthonormal and $\Sigma = \operatorname{diag}(\sigma_1, \ldots, \sigma_r)$. Decomposing $\Delta W = U Q V^\top + R$ where $R$ is orthogonal to the column/row space of the gradient, only $Q$ contributes:
\begin{equation}
\langle \nabla f(W), \Delta W \rangle_F = \tr(\Sigma^\top Q) = \sum_{i=1}^{r} \sigma_i Q_{ii}.
\end{equation}
Since $\|Q\|_{\op} \le \eta$ implies $|Q_{ii}| \le \eta$, the minimum is achieved at $Q^\star = -\eta I_r$. Thus, the polar factor is the steepest-descent direction under spectral norm:
\begin{equation}
\Delta W^\star = -\eta \, U V^\top = -\eta \, \polar(\nabla f(W)).
\end{equation}

Unlike Adam \cite{kingma2015adam} or Lion \cite{chen2023symbolic} which act entry-wise and distort
the gradient's singular-vector structure, the polar factor \(UV^\top\) preserves
the left and right singular subspaces. The polar update extracts learning signal from all singular directions of the gradient and not just those with the largest entries, and by making all the singular values unity, it  makes uniform progress along all singular directions. 
This uniform treatment is more aggressive than coordinate-wise updates, so we include an explicit momentum buffer \(M_t\) to smooth the gradient before taking its polar factor. 
We build on this spectral-norm geometry to design a schedule-free optimizer that respects how weight matrices act on activations, while obtaining implicit learning-rate decay through online averaging.

\subsection{Schedule-Free Spectral Descent with Momentum}

For parameters \(\beta , \mu \in [0,1) \) and step-size \(\eta>0\), update rules for SF-Spectral Descent with Momentum are given by
\begin{align}
Y_t &= \beta X_t + (1-\beta) Z_t, &
G_t &= \nabla f(Y_t;\xi_{t}), &
M_t &= \mu M_{t-1} + (1-\mu) G_t, \notag\\
P_t &= \polar(M_t), &
Z_{t+1} &= Z_t - \eta P_t, &
X_{t+1} &= (1 - c_{t+1}) X_t + c_{t+1} Z_{t+1},
\label{eq:sf_spectral_updates}
\end{align}
with \(c_{t+1} = 1 /(t+1) \), \(X_1=Z_1\) and \(M_1=G_1\). Here \(Z_t\) is the fast
sequence with polar update, \(X_t\) is the returned schedule-free average, \(Y_t\) is the gradient-evaluation point, and \(M_t\) is an explicit momentum buffer, 
which is crucial in practical performance (c.f., \cref{fig:ablation}).

Assume the existence of \(W^\star\in\mathbb{R}^{m\times n}\) such that
\(f(W^\star)=f^\star\) and \(\nabla f(W^\star)=0\), and define
\begin{equation}
r:=\min\{m,n\}, \qquad \Delta:=f(Y_1)-f^\star.
\end{equation}

Our convergence analysis relies on the following standard assumptions:
\begin{assumption}[Frobenius Lipschitz smoothness]
\label{assu:F-smooth-main}
There is a constant $L_F>0$ such that for all $U,V\in\R^{m\times n}$,
\begin{equation}
\norm{\nabla f(U)-\nabla f(V)}_F \le L_F \norm{U-V}_F.
\end{equation}    
\end{assumption}

\begin{assumption}[Unbiased gradient and bounded-variance noise]
\label{assu:noise-assumption-main}
Let $\F_t$ be the sigma-field generated by all randomness up to the construction of $Y_t$. We assume, for some batch size $B\ge 1$ and $\sigma^2\ge 0$,
\begin{equation}
\label{eq:noise-assumption-main}
\EE\!\bigl[G_t\mid \F_t\bigr] = \nabla f(Y_t),
\qquad
\EE\!\left[\norm{G_t-\nabla f(Y_t)}_F^2\mid \F_t\right] \le \sigma^2/B.
\end{equation}    
\end{assumption}

\begin{assumption}[Bounded diameter]
\label{assu:bounded-diameter-main}
There exists a constant $D_F>0$
such that
\begin{equation}
\label{eq:bounded-diameter}
\|W^\star\|_F \le D_F/2,
\qquad
\|Z_t\|_F \le D_F/2
\qquad \text{for all } t.
\end{equation}
\end{assumption}

Note that bounded iterate assumption is typically employed in the convergence analysis of spectral methods (e.g., see \cite[Theorem 7]{gupta2018shampoopreconditionedstochastictensor} and \cite[Theorem 1]{an2025asgoadaptivestructuredgradient}), and if necessary it can be enforced algorithmically by considering a projected variant that keep the fast iterates $Z_t$ inside the corresponding norm ball. 
Further, with weight decay at $Z$, \cref{assu:bounded-diameter-main} can be lifted, as can be seen in \cref{lem:Z-bounded}.

Under these assumptions, we obtain the following convergence guarantee:
\begin{theorem}[Stationarity of SF-Spectral Descent with Momentum]
\label{thm:sf_spectral_descent}
Under \cref{assu:F-smooth-main}--\ref{assu:bounded-diameter-main},
for every \(T\ge 1\), SF-Sepctral Descent with Momentum in \eqref{eq:sf_spectral_updates} satisfies the following:
\begin{align}
\frac1T\sum_{t=1}^{T}\mathbb{E}\|\nabla f(Y_t)\|_*
\le {}&
\frac{\Delta + 2L_F D_F^2\!\left(\log(eT)+\pi^2/6\right)}
{(1-\beta)T\eta}
+ \frac{2\sqrt r\,\sigma}{(1-\beta)\sqrt B}\left(\sqrt{\frac{1-\mu}{1+\mu}} + \frac{1}{(1-\mu)T}\right)
\notag\\
&\quad
+ \frac{2\sqrt r\,L_F D_F\log(eT)}{(1-\beta)T}
\left(\frac{1+\mu}{1-\mu}\right)
+ \frac{L_F r\,\eta}{(1-\beta)} \left(\frac{1+3\mu}{1-\mu}\right).
\label{eq:sf_spectral_descent_bound}
\end{align}
\end{theorem}

The first term is the descent term, the next captures the stochastic noise, the third term is the price of schedule-free drift, and the final term reflects the constant spectral
step and momentum tracking. The explicit momentum
stabilizes the aggressive polar steps at the cost of a tracking error. By optimizing the terms that appear on the right hand side of the theorem above, we obtain the following corollary for any \( \beta \in [0,1) \), both with and without noise.

\begin{corollary}[Optimized convergence rates for SF-Spectral Descent]
\label{cor:main_sf_spectral_descent}
Under \cref{assu:F-smooth-main}--\ref{assu:bounded-diameter-main}, for any fixed
\(\beta\in[0,1)\) there exist choices \(\mu_T\in[0,1)\) and \(\eta_T>0\) such
that
\begin{equation}
\frac1T\sum_{t=1}^{T}\mathbb{E}\|\nabla f(Y_t)\|_*
=
\widetilde{O}(T^{-1/4})
~~(\sigma > 0); 
~~\text{and}~~
\frac1T\sum_{t=1}^{T}\|\nabla f(Y_t)\|_*
=
\widetilde{O}(T^{-1/2})
~~(\sigma=0).
\end{equation}
\end{corollary}
These convergence rates match the ones derived in recent works on convergence of Muon \cite{chang2026convergencemuon, shen2026convergenceanalysismuon, sato2025convergenceboundcriticalbatch, li2025noteconvergencemuon}. 
An analogous result can also be proven under Lipschitz smoothness measured in
the spectral norm rather than the Frobenius norm. Since the Frobenius-smooth
version is more directly comparable to existing matrix-optimization methods, we
state it here and defer the spectral-smooth variant, together with the full
proofs of both results, to the \cref{app:convergence}.

\subsection{Schedule-Free NorMuon for Neural Network Training}

\begin{algorithm}[ht]
\caption{Schedule-Free NorMuon (SF-NorMuon)}
\label{algo:sf-normuon}
\small
\begin{algorithmic}[1]
\State \textbf{Input:} base learning rate $\eta$, interpolation parameter $\beta_1$, variance parameter $\beta_2$, momentum parameter $\mu$, perturbation parameter $\varepsilon$, weight decay $\lambda$, warmup steps $T_{\text{warmup}}$
\For{$t=1$ \textbf{to} $T$}
    \State $Y_t \gets (1-\beta_1)Z_t + \beta_1X_t$ \Comment{Interpolate weights}
    \State $G_t \gets \nabla_{\mathbf{W}}\mathcal{L}(Y_t, \zeta_t)$ \Comment{Compute gradient at interpolation point}
    \State $M_t \gets \mu M_{t-1} + (1-\mu) G_t$ \Comment{Explicit momentum buffer}
    \State \rule{\linewidth}{0.4pt}
    \State $P_t \gets \text{polar}(M_t)$ \Comment{Polar factor of momentum}
    \State $v_t \gets \beta_2v_{t-1} + (1-\beta_2)\operatorname{mean}_{\text{cols}}(P_t \odot P_t)$ \Comment{Row-wise second moment}
    \State $V_t \gets \mathrm{ExpandRows}(v_t)$ \Comment{Broadcast to matrix}
    \State $\widehat{P}_t \gets P_t \oslash (\sqrt{V_t} + \varepsilon)$ \Comment{Adaptive row-wise normalization}
    \State \rule{\linewidth}{0.4pt}
    \State $\eta_t \gets \eta \min(1, t/T_{\text{warmup}})$ \Comment{Learning rate with warmup}
    \State $\hat{\eta}_t \gets 0.2 \eta_t \sqrt{mn}/\|\widehat{P}_t\|_F$ \Comment{Match Adam RMS scaling}
    \State $s_t \gets s_{t-1} + \eta_t^2$
    \State $c_{t+1} \gets \eta_t^2 / s_t$ \Comment{Averaging coefficient}
    \State \rule{\linewidth}{0.4pt}
    \State $Z_{t+1} \gets Z_t -\eta \lambda Z_t - \hat{\eta}_t\widehat{P}_t$ \Comment{Update with weight decay at $Z_t$}
    \State $X_{t+1} \gets (1-c_{t+1})X_t + c_{t+1}Z_{t+1}$ \Comment{Update averaged iterate}
\EndFor
\State \textbf{Return:} $X_T$
\end{algorithmic}
\end{algorithm}

To make schedule-free spectral descent practical for neural network training, we incorporate three refinements below. The resulting algorithm is SF-NorMuon, which we summarize in \cref{algo:sf-normuon}. A PyTorch implementation of this algorithm is provided in \cref{app:implementation}.

\textbf{Learning rate warm up and scaling.}
We warm up the learning rate linearly, which stabilizes early training \cite{goyal2017accurate} and enables larger learning rates \cite{kalra2024warmup}. After computing the adaptively normalized direction $\widehat{P}_t$, the algorithm scales the base learning rate by $\hat{\eta}_t = 0.2\eta_t\sqrt{\max(m,n)}/\|\widehat{P}_t\|_F$. This scaling serves two purposes \cite{liu2025muonscalablellmtraining}: (i) It normalizes the effective step size to be comparable to Adam's RMS-normalized updates, enabling $\eta$ and $\lambda$ to be shared with SF-AdamW; and (ii) it provides additional stability by accounting for the overall magnitude of the normalized update direction.

\textbf{Weight decay at the fast iterate $Z$.} 
Weight decay is applied to $Z_t$, rather than $Y_t$ (as originally proposed in \cite{defazio2023learning} for SF-AdamW). 
This choice treats weight decay as a modification to the base optimizer dynamics rather than as an $\ell_2$ regularizer in the loss, which would correspond to decay at $Y_t$. The scaled weight decay term $\eta_t\lambda Z_t$ is subtracted from $Z_t$ in the base update. 
Weight decay at $Z_t$ is necessary for long-horizon training with schedule-free methods, as we explain in \cref{sec:sf-wd}. 

\textbf{Row-wise Normalization.}
We incorporate row-wise adaptive normalization \cite{li2025normuonmakingmuonefficient, si2025adamuonadaptivemuonoptimizer}. The polar factor can produce orthogonalized updates with high variance in step sizes across individual neurons, as different rows of a weight matrix may have vastly different update magnitudes. 
To accomodate, SF-NorMuon maintains a running average $v_t \in \mathbb{R}^m$ of the squared row norms of the polar factor $P_t$, capturing the typical magnitude of updates for each neuron and enabling per-neuron step size adaptation. The exponential moving average with parameter $\beta_2$ 
provides stability while remaining responsive to changes in gradient statistics, analogous to the second-moment estimation in Adam, but operating at the neuron level rather than element-wise.

\begin{figure}[htb]
    \centering
    \includegraphics[width=0.8\linewidth]{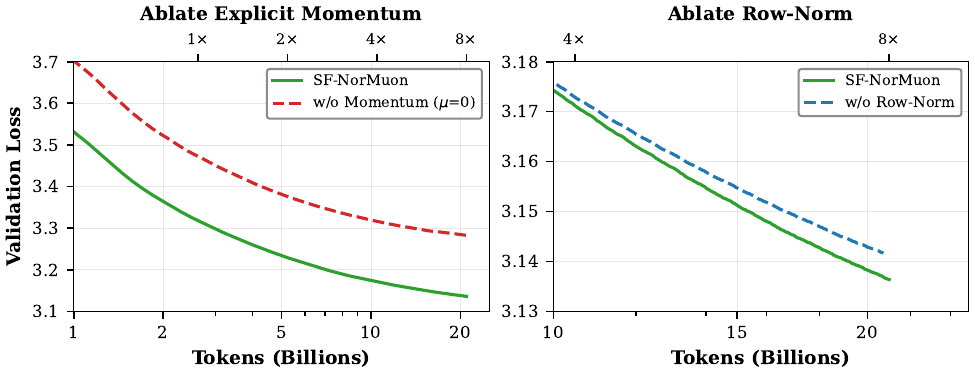}
    \caption{%
    \textbf{Importance of explicit momentum and row-wise normalization for SF-NorMuon.}
    \textit{Left:} Ablating explicit momentum ($\mu=0$) significantly degrades performance, with final loss increasing from 3.14 to 3.28. This validates the importance of smoothing the gradient before computing the polar factor, as discussed in \cref{sec:sf-spec}.
    \textit{Right:} Ablating row-wise normalization leads to a smaller gap, with SF-NorMuon reaching the same loss approximately 12\% faster.
    }
    \vspace{-5pt}
    \label{fig:ablation}
\end{figure}

In \cref{fig:ablation}, we ablate the two key aspects of SF-NorMuon: (i) the explicit momentum buffer (by setting $\mu = 0$), which reduces to direct polar decomposition of the gradient, and (ii) the row-wise adaptive normalization (by removing the $v_t$ computation and normalization step), which reduces to uniform spectral steps. These ablations help isolate the contribution of each component to the overall performance. We find that removing the row-wise adaptive normalization slows down convergence by around $12 \%$ at 8x Chinchilla ratio, but removing momentum leads to a significant drop in performance.
Note that the original SF-SGD or SF-AdamW does not have an explicit momentum buffer. Therefore, the above ablation study illustrates how direct application of schedule-free dynamics to spectral optimzer (e.g., replacing $\nabla \mathcal{L}(y_t;\xi_t)$ in \eqref{eq:SF-z} with polar$(\nabla \mathcal{L}(y_t;\xi_t))$) would perform.

\begin{remark}[Memory footprint]\label{rem:memory}
Consider a matrix parameter of shape $m \times n$. For SF methods, we need to track $Y$ and $Z$ (since $X$ can be computed on the fly). For SF-AdamW the second moment buffer (per parameter) takes up memory $m n$, whereas SF-NorMuon uses a row-wise second-moment buffer of size $m$. Thus, there is room for an explicit momentum buffer matrix $M$ of size $m \times n$ buffer unlike SF-AdamW. Hence, the total persistent memory is $3mn$ for SF-AdamW and $3mn+m$ for SF-NorMuon, compared with $3mn$ for AdamW and $2mn+m$ for NorMuon.     
\end{remark}

\section{Necessity of Weight Decay in Long Horizon Training}
\label{sec:sf-wd}

Modern language models are trained for very long horizons \cite{shazeer2017outrageously,lepikhin2021gshard,fedus2022switch}, and the training tokens for Mixture-of-Experts models can exceed 
30 times the Chinchilla ratio to the number of active parameters \cite{glm5team2026glm5vibecodingagentic,singh2026arceetrinitylargetechnical,kimiteam2026kimik25visualagentic}. 
Thus, to assess long-horizon behavior of schedule-free methods, we train the 125M LLaMa-2-style model for up to 48$\times$ Chinchilla with the default hyper-parameters.

Interestingly, for SF methods, the validation loss at $X$ diverges without weight decaying \textit{specifically at} $Z$. This is in sharp contrast to convex optimization theory, where convergence holds without regularization \cite{defazio2024road, schaipp2025surprising}. 
In particular, for SF-AdamW in \cref{fig:weight-decay-temp} (left panel), decay at $Y$ performs the best up to around 30B tokens, but eventually leads to divergence. 
Note that the original paper \cite{defazio2024road} recommended decaying at $Y$, which is sensible from the perspective of $\ell_2$ regularization. In contrast, for SF-NorMuon in \cref{fig:weight-decay-temp} (right panel), decaying at $Z$ is not only stable (see also \cref{lem:Z-bounded}) but also maintains the best validation loss throughout the long training exceeding 50B tokens. 

To better understand this phenomenon, we consider a quasi steady-state analysis, following \cite{vanlaarhoven2017l2regularizationversusbatch, defazio2025gradientsrapidlyincreasenear}. Recall from \eqref{eq:lr_effective_xz} that the schedule-free method can be rewritten as
\begin{equation}
X_{t+1} = X_t - c_{t+1}\eta U_t + c_{t+1} (Z_{t+1} - X_t),
\end{equation}
where $c_{t+1} = 1/(t+1)$, and $U_t$ is
either gradient divided by second moment buffer for SF-AdamW or polar transform of momentum buffer for SF-NorMuon, respectively. As can be seen, the evaluation sequence $X_t$ evolves with an \emph{effective} learning rate $\eta_t^{\mathrm{eff}} = c_{t+1}\eta = \eta/(t+1)$ that decays to zero as $t \to \infty$, which makes the divergence surprising. 

The culprit is in the fast $Z_t$ sequence: it evolves with constant learning rate $\eta$ throughout training, and in fact, we observe in practice that validation loss at $Z_t$ often diverges. Since $X_t$ averages all past $\{Z_s\}_{s \leq t}$, $Z_t$ eventually contaminates the $X_t$ sequence.

For SF-NorMuon, applying weight decay directly at $Z_t$ addresses this issue, while maintaining good performance. 
To make this precise, we prove that weight decay at $Z_t$ ensures all iterates to remain bounded, and characterize the precise steady-state norm to which $Z_t$ converges.

\begin{lemma}
\label{lem:Z-bounded}
    Consider the update 
    $Z_{t+1} \gets Z_t -\eta \lambda Z_t - \hat{\eta}_t\widehat{P}_t$
    and assume $0 < \eta \lambda < 1$ with $\hat{\eta} = \kappa\, \eta \sqrt{mn}/\|\widehat{P}_t\|_F$ for a constant $\kappa>0$. Then, we have for all $t$,
    \begin{equation} \label{eq:Z-bounded}
    \|Z_t\|_F \leq  \|Z_0\|_F (1- \eta \lambda)^{t}+  \frac{\kappa \sqrt{mn}}{\lambda}  \implies 
    \|Z_t\|_F \leq \frac{\kappa \sqrt{mn}}{\lambda} ~\textrm{as} ~ t \to \infty.
    \end{equation}
Furthermore, since $X_t$ and $Y_t$ are convex combinations of $\{Z_s\}_{s \leq t}$, they satisfy the same bound. 
\end{lemma}

In \cref{fig:weight-decay-temp}, we see that weight decay at $Y_t$ eventually leads to divergence, for both SF-NorMuon and SF-AdamW. We make the following remark regarding this observation:
\begin{remark}
For weight decay at $Y_t$, the coupled dynamics $(Z_t, D_t)$ where $D_t = Z_t - X_t$ satisfy:
\begin{equation}
\label{eq:wd-y-triangle}
\begin{bmatrix} \|Z_{t+1}\|_F \\ \|D_{t+1}\|_F \end{bmatrix} 
\leq 
\begin{bmatrix} 1-\eta\lambda & \eta\lambda\beta \\ (1-c_{t+1})\eta\lambda & (1-c_{t+1})(1+\eta\lambda\beta) \end{bmatrix}
\begin{bmatrix} \|Z_t\|_F \\ \|D_t\|_F \end{bmatrix} + \kappa\,\eta\sqrt{mn}\begin{bmatrix} 1 \\ 1-c_{t+1} \end{bmatrix}.
\end{equation}
In the long-horizon limit $c_t \to 0$, the iteration matrix above has the leading eigenvalue
\begin{equation}
    \lambda_1 = 1 + \frac{\eta\lambda}{2}\left(\sqrt{1 + 6\beta + \beta^2} - (1-\beta)\right) > 1 
\end{equation}
for all $\beta > 0$, and so the triangle-inequality upper bound similar to \eqref{eq:Z-bounded} diverges. 
\end{remark}

\begin{figure}[htb]
    \centering
    \includegraphics[width=0.82\linewidth]{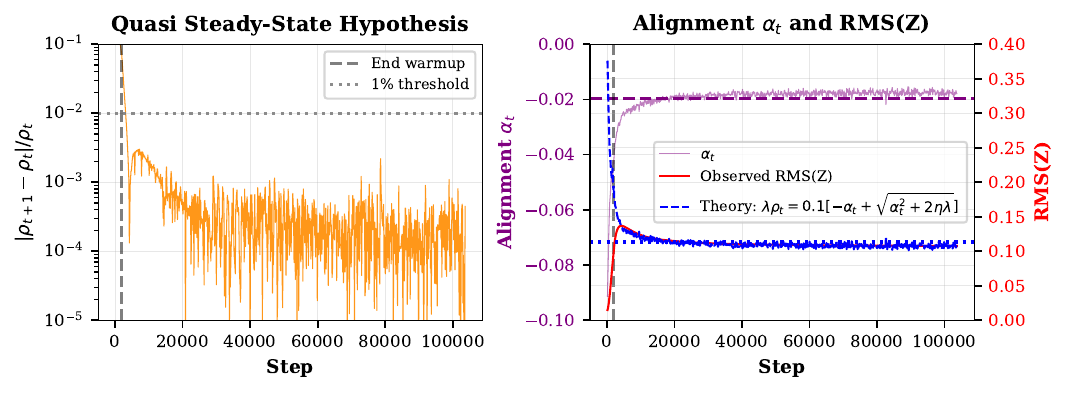}
    \caption{%
    \textbf{Quasi steady-state analysis for SF-NorMuon with decay at $\mathbf{Z}$ (averaged over layers).}
    \textit{Left:} The ratio $|\rho_{t+1} - \rho_t|/\rho_t$ remains below 1\% after warmup, validating the quasi steady-state hypothesis (\cref{assu:F-smooth}).
    \textit{Right:} Alignment $\alpha_t$ (purple) and RMS$(Z)$ (red) over training. The theoretical prediction $\lambda\rho_t = 0.1[-\alpha_t + \sqrt{\alpha_t^2 + 2\eta\lambda}]$ from \cref{lem:Z-steady-state} with $\kappa=0.2$ (blue dashed) closely tracks the observed values. This training run corresponds to $\eta = 0.01, \lambda = 0.05$. 
    }
    \label{fig:wd-rms}
\end{figure}

Having established iterate boundedness in \cref{lem:Z-bounded}, we now study how the norm of $Z$ evolves during the training. For this purpose, the analysis greatly simplifies if we make a quasi steady-state hypothesis following  \cite{vanlaarhoven2017l2regularizationversusbatch, defazio2025gradientsrapidlyincreasenear}, which states that change in norm of $Z$ between two iterations is much smaller than the norm itself. 

\begin{hypothesis}[Quasi Steady-State]
\label{assu:F-smooth}
There exists an iteration count $\tau$ and a constant $\varepsilon\geq 0$ such that for all $t\geq \tau$,
\begin{equation}
\label{eq:steady}
\left| \|Z_{t+1}\|_F - \|Z_{t}\|_F \right| \leq \varepsilon \|Z_{t}\|_F~.
\end{equation}    
\end{hypothesis}

To enable uniform comparison across layers of different shapes, we define the RMS norm $\rho_t := \|Z_t\|_F / \sqrt{mn}$ for an $m \times n$ weight matrix, and the alignment $\alpha_t := \langle \widehat{P}_t, Z_t \rangle / (\|\widehat{P}_t\|_F \|Z_t\|_F) \in [-1, 1]$, which is the cosine of the angle between the preconditioned update and the current iterate. The left panel of \cref{fig:wd-rms} shows that $|\rho_{t+1} - \rho_t|/\rho_t$ remains small after warmup, validating the hypothesis. Based on this, we derive the steady-state value of $\rho_t$ :

\begin{lemma}[Quasi Steady-State of $Z$]
\label{lem:Z-steady-state}
Consider the update 
$Z_{t+1} \gets Z_t -\eta \lambda Z_t - \hat{\eta}_t\widehat{P}_t$
and assume $0 < \eta \lambda \ll 1$ with $\hat{\eta} = \kappa\, \eta \sqrt{mn}/\|\widehat{P}_t\|_F$ for a constant $\kappa>0$.
Then, the steady-state satisfies
\begin{equation}
\lambda\rho_t = \frac{\kappa}{2}\Bigl[-\alpha_t + \sqrt{\alpha_t^2 + 2\eta\lambda}\Bigr].
\end{equation}
\end{lemma}
The constant $\kappa$ enters both lemmas only as a multiplicative factor, so the results hold for any RMS-matching choice, including the exact factor matching AdamW's update RMS; our experiments use the heuristic value $\kappa=0.2$ from \cite{liu2025muonscalablellmtraining}.

After the end of the warm-up period, the alignment $\alpha_t$ approaches a small and negative value that persists for the rest of the training period as shown by the purple curve in right panel of \cref{fig:wd-rms}. Based on the \cref{lem:Z-steady-state}, we expect that the RMS norm of $Z$ (depicted in red) will asymptote to a constant value, which is also confirmed by the red curve sitting on top of the theoretical prediction shown in blue. Crucially, since the $Z$ norm reaches a fixed value, the fast iterate remains controlled throughout training, and stabilizes other sequences as shown in \cref{lem:Z-bounded}. 

\section{Experiments on Training Language Models}

\begin{table}[ht]
\centering
\caption{\textbf{Validation loss at $X$ across Chinchilla horizons.} \emph{Speedup} (in \textbf{bold}) is the percentage of steps saved by SF-NorMuon to reach the final loss of SF-AdamW. SF methods use hyperparameters tuned at $2\times$ Chinchilla for both models. AdamW uses cosine schedule tuned per-horizon (c.f., \cref{tab:hyperparam-space}) for 125M model, and the configuration for $8\times$ was reused for the 772M model. SF-AdamW uses weight decay at $Y$ as it performs the best in these horizons (c.f., \cref{fig:weight-decay-temp}) and is the originally proposed option in \cite{defazio2024road}. NorMuon (cosine) is a horizon-aware scheduled reference, tuned per-horizon for the 125M model and with a single configuration reused across horizons for the 772M model.
}
\label{tab:main-results}
\small
\resizebox{\textwidth}{!}{%
\begin{tabular}{@{}c ccccc ccccc@{}}
\toprule
& \multicolumn{5}{c}{\textbf{125M Model}} & \multicolumn{5}{c}{\textbf{772M Model}} \\
\cmidrule(lr){2-6} \cmidrule(lr){7-11}
Chinchilla & SF-NorMuon & SF-AdamW & AdamW & NorMuon & & SF-NorMuon & SF-AdamW & AdamW & NorMuon & \\
 Ratio & (decay@$Z$) & (decay@$Y$) & (cosine) & (cosine) & Speedup & (decay@$Z$) & (decay@$Y$) & (cosine) & (cosine) & Speedup \\
\midrule
$1\times$ & 3.318 & 3.399 & 3.354 & 3.292 & \textbf{35\%} & 2.777 & 2.852 & 2.826 & 2.752 & \textbf{50\%} \\
$2\times$ & 3.229 & 3.282 & 3.231 & 3.200 & \textbf{35\%} & 2.729 & 2.780 & 2.768 & 2.700 & \textbf{52\%} \\
$4\times$ & 3.170 & 3.205 & 3.170 & 3.139 & \textbf{36\%} & 2.698 & 2.726 & 2.720 & 2.662 & \textbf{47\%} \\
$8\times$ & 3.136 & 3.141 & 3.124 & 3.103 & \textbf{11\%} & --- & --- & --- & --- & --- \\
\bottomrule
\end{tabular}%
}
\end{table}

We validate SF-NorMuon on auto-regressive language modeling by training decoder only transformers at two scales: a 125M parameter model and a 772M parameter model. Both architectures follow the LLaMA-2 design \cite{touvron2023llama2openfoundation}, with tied input-output embeddings, rotary positional embeddings (RoPE) \cite{su2023roformerenhancedtransformerrotary}, RMSNorm for pre-normalization \cite{zhang2019rootmeansquarelayer}, and squared ReLU activations in the MLP blocks \cite{zhang2024relu2}. Additional architecture details are provided in \cref{app:arch-hypers}. Our training code is based on the NanoGPT repository \cite{Karpathy2022, modded_nanogpt_2024}. 

\textbf{Dataset and training setup.}
We train on the FineWeb-100B dataset \cite{penedo2024finewebdatasetsdecantingweb} tokenized with the GPT-2 tokenizer \cite{radford2019language}. All experiments use batch size of 512 sequences and sequence length of 1024 tokens, yielding approximately 524K tokens per gradient step. We adopt the Chinchilla scaling \cite{hoffmann2022}, where $N\times$ Chinchilla denotes training on $\sim 20N$ tokens per model parameter. For our 125M model, this corresponds to 2.5B tokens at $1\times$ (5,000 steps), scaling up to 21B tokens at $8\times$ (40,000 steps). 

\textbf{Hyperparameter tuning.}
We start with the 125M model.
For schedule-free optimizers, we use 2,000 warm up steps and tune hyper-parameters at $2\times$ Chinchilla, and reuse these for all settings.
The best hyperparameters for SF-AdamW are $\eta=0.01$, $(\beta_1, \beta_2)=(0.95,0.99)$, and $\lambda=0.05$.  
For SF-NorMuon, we reuse the same learning rate and weight decay parameter, following the $0.2 \eta_t\sqrt{mn}/\|\widehat{P_t}\|_F$ rescaling factor from \cite{kimiteam2026kimik25visualagentic}. The other hyper-parameters were set to $(\beta_1, \beta_2)=(0.9, 0.95)$, momentum $\mu=0.8$. For the non-matrix parameters (embeddings, layer norms), both schedule-free methods run SF-AdamW. 
For the AdamW baseline, we tune the learning rate \textit{separately at each horizon}, using 2,500 warm up steps followed by cosine decay to zero; comparison of these runs can be found in \cref{fig:lr-sweep-adamw} in \cref{app:arch-hypers}. Other hyperparameters were tuned at $2 \times$ Chinchilla.
This represents a well-tuned baseline with per-horizon optimization, using optimized $\eta \in \{0.004, 0.008\}$, and $\beta_1=0.9$, $\beta_2=0.95$, and $\lambda=0.1$. Additional hyper-parameter tuning details and learning rate sweeps are provided in 
\cref{tab:hyperparam-space} in \cref{app:arch-hypers}. 

For the 772M parameter model, we reuse the optimal $8\times$ Chinchilla hyper-parameters from AdamW and the $2\times$-tuned hyper-parameters for schedule-free methods. This approach is motivated by the observation that losses near the optimal learning rates exhibit only small variation, and recent work has shown similar optimal learning rates across small and large Llama-2 model scales (see Table 7 and Table 35 in \cite{semenov2026benchmarking}). For scheduled NorMuon on the 772M model, we use a single configuration across horizons ($\eta=0.004$, $\lambda=0.1$, $\mu=0.95$, $\beta_2=0.95$, 2,000 warm up steps, cosine decay to zero). We omit the $8\times$ Chinchilla runs for the large model due to compute limitations. 


\textbf{Results.}
\cref{tab:main-results} reports validation losses (evaluated at $X$) across training horizons. Note that the SF-NorMuon uses exactly the same learning rate and weight decay as SF-AdamW. For both models, SF-NorMuon with weight decay at $Z$ consistently outperforms SF-AdamW across all horizons. Crucially, SF-NorMuon matches the per-horizon tuned AdamW baselines (except at $8\times$ Chinchilla where it lags marginally), while SF-AdamW lags behind consistently. This demonstrates that SF-NorMuon delivers competitive performance with AdamW at any stopping point, without a horizon-specific learning rate schedule. For the 772M model, the speedup advantage of SF-NorMuon over SF-AdamW is even more pronounced: 50\% at $1\times$ and 52\% at $2\times$ Chinchilla. 

We also report performance of NorMuon with a cosine schedule (\cref{tab:main-results}, \cref{fig:validation-loss}), which outperforms all other methods. The gap between it and SF-NorMuon is $0.025$--$0.036$ nats, which is comparable to the gap between AdamW and SF-AdamW. This gap is expected, since NorMuon knows the horizon in advance and benefits from a final learning-rate cool-down \cite{you2019doeslearningratedecay}, whereas SF-NorMuon relies only on implicit decay through online averaging. Matching the performance of horizon-aware NorMuon (or, more generally, the best scheduled spectral optimizer) with a truly schedule-free method remains an open problem; SF-NorMuon closes the analogous gap for AdamW. 

We additionally compare against NorMuon with a warmup-stable-decay (WSD) schedule \cite{hu2024WSD, hagele2024scalinglawscomputeoptimaltraining} in \cref{app:wsd}: WSD trails cosine at every horizon and, when its stable branch is extended to $2\times$ Chinchilla, is outperformed by SF-NorMuon as well. A constant-learning-rate NorMuon with EMA weights \cite{zhang2025doescriticalbatchsize}, which uses the same $\sim 3mn$ memory as SF-NorMuon (\cref{rem:memory}), also remains behind SF-NorMuon (\cref{app:ema}).
All reported gaps are several times larger than the seed-to-seed standard deviation ($\leq 0.006$ nats; \cref{app:seeds}), and SF-NorMuon incurs only $\sim 2\%$ wall-clock overhead per step relative to AdamW (\cref{app:implementation}).


\begin{figure}[htb]
    \centering
    \includegraphics[width=0.5\textwidth]{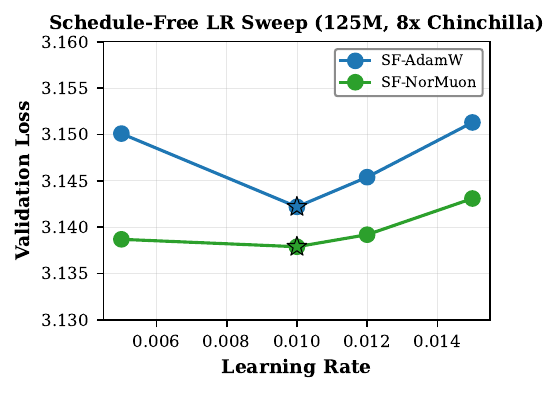}
    \caption{Learning rate sweep comparison between SF-AdamW and SF-NorMuon at $8\times$ Chinchilla horizon. SF-NorMuon leads to better validation loss across all tested learning rates.}
    \label{fig:lr-sweep-sf}
\end{figure}

\textbf{Improvement across learning rates.}
A key practical advantage of SF-NorMuon is robustness to the choice of learning rate. \cref{fig:lr-sweep-sf} shows a learning-rate sweep at $8\times$ Chinchilla for both schedule-free methods: SF-NorMuon achieves lower validation loss than SF-AdamW across all learning rates tested.
Importantly, SF-NorMuon's performance degrades gracefully away from the optimum, confirming that the spectral geometry provides a favorable optimization landscape even without careful per-run tuning.



\textbf{Other findings.}
We present the full convergence proofs under both Frobenius and spectral smoothness in \cref{app:convergence}. Additional material includes: a self-contained derivation of spectral and sign update rules from first principles (\cref{app:app-metrized}), a review of SF-SGD and SF-AdamW (\cref{app:sf-review}), detailed model architectures and hyperparameter search spaces (\cref{app:arch-hypers}), an ablation study of the explicit momentum and row normalization, comparisons with WSD-scheduled and EMA-averaged NorMuon, and seed-to-seed variability (\cref{app:ablation}), experiments on training ViT-B/16 on ImageNet-1K, where SF-NorMuon again outperforms both SF-AdamW and tuned AdamW at every milestone (\cref{app:vit-hypers}), and a PyTorch implementation of SF-NorMuon in \cref{algo:sf-normuon} (\cref{app:implementation}).

\section{Limitations}
Horizon-aware methods such as cosine or WSD schedules assume the stopping time is known (or that a cool-down branch can be run before each deployment) and exploit a final learning-rate anneal to obtain a good model at the end of training. In contrast, horizon-free methods such as SF-NorMuon make no such assumption and return a deployable iterate at every step, at the cost of relying on implicit decay through averaging. Consistent with this trade-off, horizon-tuned NorMuon remains $\sim 0.03$ nats ahead of SF-NorMuon, and fully matching scheduled spectral optimizers without a schedule remains an open problem. Our study focuses on the Muon family; schedule-free variants of other matrix pre-conditioned optimizers such as Shampoo \cite{gupta2018shampoopreconditionedstochastictensor} and SOAP \cite{vyas2025soapimprovingstabilizingshampoo} are a natural extension. All language-model runs use a single batch size of $\sim$0.5M tokens, and we did not sweep batch size, which is known to affect schedule-free methods. 


\section{Conclusion}
In this work, we proposed SF-NorMuon, a schedule-free spectral optimizer that matches per-horizon tuned AdamW across Chinchilla horizons without requiring schedule or horizon specification in advance. We validated SF-NorMuon on both auto-regressive language modeling (125M and 772M parameter LLaMA-2-style models), demonstrating consistent improvements over SF-AdamW. We proved a $\widetilde{O}(T^{-1/4})$ stationarity guarantee for the underlying schedule-free spectral dynamics for smooth non-convex functions and identified weight decay at the fast iterate as essential for long-horizon stability, which is a departure from the convex optimization theory. Extending the validation of SF-NorMuon to larger scales and multi-stage continual learning pipelines is a natural direction for future work. Since weight decay at the fast iterate drives the weight norm to a steady state (\cref{lem:Z-steady-state}), combining schedule-free spectral updates with explicitly norm-constrained training \cite{wen2026fantasticpretrainingoptimizersii} is another promising direction.

\newpage

\section*{Acknowledgments}
We thank Anthony Ashmore for insightful discussions on the importance of learning rate scheduling. We thank Aaron Defazio for valuable guidance on schedule-free methods. We are grateful to Pragna Subrahmanya for assistance with computing infrastructure. We thank Rob Otter for the executive support of the work, and our colleagues at the Global Technology Applied Research center of JPMorganChase
for support.

\section*{Author Contributions}
A.A. and J.L.K. conceived the algorithm. A.A. developed the codebase, conducted the language modeling experiments, and established the convergence and steady-state theory. P.D. performed vision transformer experiments. 
J.L.K. carried out exploratory experiments, shaped the overall research direction and oversaw the presentation of the material. S.C. and N.K. provided feedback and contributed to discussions throughout the project. All authors contributed to writing the manuscript.

\section*{Disclaimer}
This paper was prepared for informational purposes with contributions from the Global Technology Applied Research center of JPMorgan Chase \& Co. This paper is not a product of the Research Department of JPMorgan Chase \& Co. or its affiliates. Neither JPMorgan Chase \& Co. nor any of its affiliates makes any explicit or implied representation or warranty and none of them accept any liability in connection with this paper, including, without limitation, with respect to the completeness, accuracy, or reliability of the information contained herein and the potential legal, compliance, tax, or accounting effects thereof. This document is not intended as investment research or investment advice, or as a recommendation, offer, or solicitation for the purchase or sale of any security, financial instrument, financial product or service, or to be used in any way for evaluating the merits of participating in any transaction.

\newpage 
\printbibliography

@article{lan2012optimal,
  title={An optimal method for stochastic composite optimization},
  author={Lan, Guanghui},
  journal={Mathematical Programming},
  volume={133},
  number={1},
  pages={365--397},
  year={2012},
  publisher={Springer}
}

@inproceedings{nesterov1983method,
  title={A method for solving the convex programming problem with convergence rate O (1/k2)},
  author={Nesterov, Yurii},
  booktitle={Dokl akad nauk Sssr},
  volume={269},
  pages={543},
  year={1983}
}

@article{hagele2024scaling,
  title={Scaling laws and compute-optimal training beyond fixed training durations},
  author={H{\"a}gele, Alexander and Bakouch, Elie and Kosson, Atli and Allal, Loubna B and Von Werra, Leandro and Jaggi, Martin},
  journal={Advances in Neural Information Processing Systems},
  volume={37},
  pages={76232--76264},
  year={2024}
}

@article{schaipp2025surprising,
  title={The surprising agreement between convex optimization theory and learning-rate scheduling for large model training},
  author={Schaipp, Fabian and H{\"a}gele, Alexander and Taylor, Adrien and Simsekli, Umut and Bach, Francis},
  journal={arXiv preprint arXiv:2501.18965},
  year={2025}
}

@article{kasimbeg2025far,
  title={How far away are truly hyperparameter-free learning algorithms?},
  author={Kasimbeg, Priya and Roulet, Vincent and Agarwal, Naman and Medapati, Sourabh and Pedregosa, Fabian and Agarwala, Atish and Dahl, George E},
  journal={arXiv preprint arXiv:2505.24005},
  year={2025}
}

@article{izmailov2018averaging,
  title={Averaging weights leads to wider optima and better generalization},
  author={Izmailov, Pavel and Podoprikhin, Dmitrii and Garipov, Timur and Vetrov, Dmitry and Wilson, Andrew Gordon},
  journal={arXiv preprint arXiv:1803.05407},
  year={2018}
}

@article{degris2024step,
  title={Step-size optimization for continual learning},
  author={Degris, Thomas and Javed, Khurram and Sharifnassab, Arsalan and Liu, Yuxin and Sutton, Richard},
  journal={arXiv preprint arXiv:2401.17401},
  year={2024}
}

@article{meterez2026anytime,
  title={Anytime Pretraining: Horizon-Free Learning-Rate Schedules with Weight Averaging},
  author={Meterez, Alexandru and Nair, Pranav Ajit and Morwani, Depen and Pehlevan, Cengiz and Kakade, Sham},
  journal={arXiv preprint arXiv:2602.03702},
  year={2026}
}

@article{defazio2024road,
  title={The road less scheduled},
  author={Defazio, Aaron and Yang, Xingyu and Mehta, Harsh and Mishchenko, Konstantin and Khaled, Ahmed and Cutkosky, Ashok},
  journal={Advances in Neural Information Processing Systems},
  volume={37},
  pages={9974--10007},
  year={2024}
}

@inproceedings{defazio2023learning,
  title={Learning-rate-free learning by d-adaptation},
  author={Defazio, Aaron and Mishchenko, Konstantin},
  booktitle={International Conference on Machine Learning},
  pages={7449--7479},
  year={2023},
  organization={PMLR}
}

@book{Goodfellow-et-al-2016,
    title={Deep Learning},
    author={Ian Goodfellow and Yoshua Bengio and Aaron Courville},
    publisher={MIT Press},
    note={\url{http://www.deeplearningbook.org}},
    year={2016}
}

@book{Hastie2009,
  title = {The Elements of Statistical Learning},
  ISBN = {9780387848587},
  ISSN = {2197-568X},
  url = {http://dx.doi.org/10.1007/978-0-387-84858-7},
  DOI = {10.1007/978-0-387-84858-7},
  journal = {Springer Series in Statistics},
  publisher = {Springer New York},
  author = {Hastie,  Trevor and Tibshirani,  Robert and Friedman,  Jerome},
  year = {2009}
}

@book{Bishop2024,
  title = {Deep Learning: Foundations and Concepts},
  ISBN = {9783031454684},
  url = {http://dx.doi.org/10.1007/978-3-031-45468-4},
  DOI = {10.1007/978-3-031-45468-4},
  publisher = {Springer International Publishing},
  author = {Bishop,  Christopher M. and Bishop,  Hugh},
  year = {2024}
}

@misc{Gao2024,
Author = {Shervin Minaee and Tomas Mikolov and Narjes Nikzad and Meysam Chenaghlu and Richard Socher and Xavier Amatriain and Jianfeng Gao},
Title = {Large Language Models: A Survey},
Year = {2024},
Eprint = {arXiv:2402.06196},
}

@misc{Wen2023,
Author = {Wayne Xin Zhao and Kun Zhou and Junyi Li and Tianyi Tang and Xiaolei Wang and Yupeng Hou and Yingqian Min and Beichen Zhang and Junjie Zhang and Zican Dong and Yifan Du and Chen Yang and Yushuo Chen and Zhipeng Chen and Jinhao Jiang and Ruiyang Ren and Yifan Li and Xinyu Tang and Zikang Liu and Peiyu Liu and Jian-Yun Nie and Ji-Rong Wen},
Title = {A Survey of Large Language Models},
Year = {2023},
Eprint = {arXiv:2303.18223},
}

@article{brown2020language,
  title={Language models are few-shot learners},
  author={Brown, Tom and Mann, Benjamin and Ryder, Nick and Subbiah, Melanie and Kaplan, Jared D and Dhariwal, Prafulla and Neelakantan, Arvind and Shyam, Pranav and Sastry, Girish and Askell, Amanda and others},
  journal={Advances in neural information processing systems},
  volume={33},
  pages={1877--1901},
  year={2020}
}

@misc{touvron2023,
      title={LLaMA: Open and Efficient Foundation Language Models}, 
      author={Hugo Touvron and Thibaut Lavril and Gautier Izacard and Xavier Martinet and Marie-Anne Lachaux and Timothée Lacroix and Baptiste Rozière and Naman Goyal and Eric Hambro and Faisal Azhar and Aurelien Rodriguez and Armand Joulin and Edouard Grave and Guillaume Lample},
      year={2023},
      eprint={2302.13971},
      archivePrefix={arXiv},
      primaryClass={cs.CL},
      url={https://arxiv.org/abs/2302.13971}, 
}

@misc{Brohan2023,
Author = {Anthony Brohan and Noah Brown and Justice Carbajal and Yevgen Chebotar and Xi Chen and Krzysztof Choromanski and Tianli Ding and Danny Driess and Avinava Dubey and Chelsea Finn and Pete Florence and Chuyuan Fu and Montse Gonzalez Arenas and Keerthana Gopalakrishnan and Kehang Han and Karol Hausman and Alexander Herzog and Jasmine Hsu and Brian Ichter and Alex Irpan and Nikhil Joshi and Ryan Julian and Dmitry Kalashnikov and Yuheng Kuang and Isabel Leal and Lisa Lee and Tsang-Wei Edward Lee and Sergey Levine and Yao Lu and Henryk Michalewski and Igor Mordatch and Karl Pertsch and Kanishka Rao and Krista Reymann and Michael Ryoo and Grecia Salazar and Pannag Sanketi and Pierre Sermanet and Jaspiar Singh and Anikait Singh and Radu Soricut and Huong Tran and Vincent Vanhoucke and Quan Vuong and Ayzaan Wahid and Stefan Welker and Paul Wohlhart and Jialin Wu and Fei Xia and Ted Xiao and Peng Xu and Sichun Xu and Tianhe Yu and Brianna Zitkovich},
Title = {RT-2: Vision-Language-Action Models Transfer Web Knowledge to Robotic Control},
Year = {2023},
Eprint = {arXiv:2307.15818},
}

@misc{King2024,
Author = {Yueen Ma and Zixing Song and Yuzheng Zhuang and Jianye Hao and Irwin King},
Title = {A Survey on Vision-Language-Action Models for Embodied AI},
Year = {2024},
Eprint = {arXiv:2405.14093},
}

@misc{Finn2024,
Author = {Moo Jin Kim and Karl Pertsch and Siddharth Karamcheti and Ted Xiao and Ashwin Balakrishna and Suraj Nair and Rafael Rafailov and Ethan Foster and Grace Lam and Pannag Sanketi and Quan Vuong and Thomas Kollar and Benjamin Burchfiel and Russ Tedrake and Dorsa Sadigh and Sergey Levine and Percy Liang and Chelsea Finn},
Title = {OpenVLA: An Open-Source Vision-Language-Action Model},
Year = {2024},
Eprint = {arXiv:2406.09246},
}

@article{Parisi2019,
  title = {Continual lifelong learning with neural networks: A review},
  volume = {113},
  ISSN = {0893-6080},
  url = {http://dx.doi.org/10.1016/j.neunet.2019.01.012},
  DOI = {10.1016/j.neunet.2019.01.012},
  journal = {Neural Networks},
  publisher = {Elsevier BV},
  author = {Parisi,  German I. and Kemker,  Ronald and Part,  Jose L. and Kanan,  Christopher and Wermter,  Stefan},
  year = {2019},
  month = may,
  pages = {54–71}
}

@misc{wang2024continuallearning,
      title={A Comprehensive Survey of Continual Learning: Theory, Method and Application}, 
      author={Liyuan Wang and Xingxing Zhang and Hang Su and Jun Zhu},
      year={2024},
      eprint={2302.00487},
      archivePrefix={arXiv},
      primaryClass={cs.LG},
      url={https://arxiv.org/abs/2302.00487}, 
}

@article{Kirkpatrick2017,
  title = {Overcoming catastrophic forgetting in neural networks},
  volume = {114},
  ISSN = {1091-6490},
  url = {http://dx.doi.org/10.1073/pnas.1611835114},
  DOI = {10.1073/pnas.1611835114},
  number = {13},
  journal = {Proceedings of the National Academy of Sciences},
  publisher = {Proceedings of the National Academy of Sciences},
  author = {Kirkpatrick,  James and Pascanu,  Razvan and Rabinowitz,  Neil and Veness,  Joel and Desjardins,  Guillaume and Rusu,  Andrei A. and Milan,  Kieran and Quan,  John and Ramalho,  Tiago and Grabska-Barwinska,  Agnieszka and Hassabis,  Demis and Clopath,  Claudia and Kumaran,  Dharshan and Hadsell,  Raia},
  year = {2017},
  month = mar,
  pages = {3521–3526}
}

@article{Delange2021,
   title={A continual learning survey: Defying forgetting in classification tasks},
   ISSN={1939-3539},
   url={http://dx.doi.org/10.1109/TPAMI.2021.3057446},
   DOI={10.1109/tpami.2021.3057446},
   journal={IEEE Transactions on Pattern Analysis and Machine Intelligence},
   publisher={Institute of Electrical and Electronics Engineers (IEEE)},
   author={Delange, Matthias and Aljundi, Rahaf and Masana, Marc and Parisot, Sarah and Jia, Xu and Leonardis, Ales and Slabaugh, Greg and Tuytelaars, Tinne},
   year={2021},
}

@inproceedings{
loshchilov2017cosine,
title={{SGDR}: Stochastic Gradient Descent with Warm Restarts},
author={Ilya Loshchilov and Frank Hutter},
booktitle={International Conference on Learning Representations},
year={2017},
url={https://openreview.net/forum?id=Skq89Scxx}
}

@misc{hu2024WSD,
      title={MiniCPM: Unveiling the Potential of Small Language Models with Scalable Training Strategies}, 
      author={Shengding Hu and Yuge Tu and Xu Han and Chaoqun He and Ganqu Cui and Xiang Long and Zhi Zheng and Yewei Fang and Yuxiang Huang and Weilin Zhao and Xinrong Zhang and Zheng Leng Thai and Kaihuo Zhang and Chongyi Wang and Yuan Yao and Chenyang Zhao and Jie Zhou and Jie Cai and Zhongwu Zhai and Ning Ding and Chao Jia and Guoyang Zeng and Dahai Li and Zhiyuan Liu and Maosong Sun},
      year={2024},
      eprint={2404.06395},
      archivePrefix={arXiv},
      primaryClass={cs.CL},
      url={https://arxiv.org/abs/2404.06395}, 
}

@misc{dremov2025trainingdynamicscooldownstage,
      title={Training Dynamics of the Cooldown Stage in Warmup-Stable-Decay Learning Rate Scheduler}, 
      author={Aleksandr Dremov and Alexander Hägele and Atli Kosson and Martin Jaggi},
      year={2025},
      eprint={2508.01483},
      archivePrefix={arXiv},
      primaryClass={cs.LG},
      url={https://arxiv.org/abs/2508.01483}, 
}

@misc{ma2026wsqdhorizonfreelearningrate,
      title={WSqD: A Horizon-Free Learning Rate Schedule for Large Model Training}, 
      author={Jianhao Ma and Yuxin Chen},
      year={2026},
      eprint={2607.10959},
      archivePrefix={arXiv},
      primaryClass={cs.LG},
      url={https://arxiv.org/abs/2607.10959}, 
}

@misc{kimiteam2026kimik3openfrontier,
      title={Kimi K3: Open Frontier Intelligence}, 
      author={Kimi Team},
      year={2026},
      eprint={2607.24653},
      archivePrefix={arXiv},
      primaryClass={cs.CL},
      url={https://arxiv.org/abs/2607.24653}, 
}

@misc{zhang2025doescriticalbatchsize,
      title={How Does Critical Batch Size Scale in Pre-training?}, 
      author={Hanlin Zhang and Depen Morwani and Nikhil Vyas and Jingfeng Wu and Difan Zou and Udaya Ghai and Dean Foster and Sham Kakade},
      year={2025},
      eprint={2410.21676},
      archivePrefix={arXiv},
      primaryClass={cs.LG},
      url={https://arxiv.org/abs/2410.21676}, 
}

@misc{wen2026fantasticpretrainingoptimizersii,
      title={Fantastic Pretraining Optimizers and Where to Find Them II: Hyperball Optimization}, 
      author={Kaiyue Wen and Xingyu Dang and Kaifeng Lyu and Tengyu Ma and Percy Liang},
      year={2026},
      eprint={2606.16899},
      archivePrefix={arXiv},
      primaryClass={cs.LG},
      url={https://arxiv.org/abs/2606.16899}, 
}

@misc{vyas2025soapimprovingstabilizingshampoo,
      title={SOAP: Improving and Stabilizing Shampoo using Adam}, 
      author={Nikhil Vyas and Depen Morwani and Rosie Zhao and Mujin Kwun and Itai Shapira and David Brandfonbrener and Lucas Janson and Sham Kakade},
      year={2025},
      eprint={2409.11321},
      archivePrefix={arXiv},
      primaryClass={cs.LG},
      url={https://arxiv.org/abs/2409.11321}, 
}

@misc{hoffmann2022,
      title={Training Compute-Optimal Large Language Models}, 
      author={Jordan Hoffmann and Sebastian Borgeaud and Arthur Mensch and Elena Buchatskaya and Trevor Cai and Eliza Rutherford and Diego de Las Casas and Lisa Anne Hendricks and Johannes Welbl and Aidan Clark and Tom Hennigan and Eric Noland and Katie Millican and George van den Driessche and Bogdan Damoc and Aurelia Guy and Simon Osindero and Karen Simonyan and Erich Elsen and Jack W. Rae and Oriol Vinyals and Laurent Sifre},
      year={2022},
      eprint={2203.15556},
      archivePrefix={arXiv},
      primaryClass={cs.CL},
      url={https://arxiv.org/abs/2203.15556}, 
}

@misc{li2025modelmergingpretraininglarge,
      title={Model Merging in Pre-training of Large Language Models}, 
      author={Yunshui Li and Yiyuan Ma and Shen Yan and Chaoyi Zhang and Jing Liu and Jianqiao Lu and Ziwen Xu and Mengzhao Chen and Minrui Wang and Shiyi Zhan and Jin Ma and Xunhao Lai and Deyi Liu and Yao Luo and Xingyan Bin and Hongbin Ren and Mingji Han and Wenhao Hao and Bairen Yi and LingJun Liu and Bole Ma and Xiaoying Jia and Xun Zhou and Siyuan Qiao and Liang Xiang and Yonghui Wu},
      year={2025},
      eprint={2505.12082},
      archivePrefix={arXiv},
      primaryClass={cs.CL},
      url={https://arxiv.org/abs/2505.12082}, 
}

@inproceedings{Kasimbeg2025AlgoPerfResults,
title           = {Accelerating neural network training: An analysis of the {AlgoPerf} competition},
author          = {Kasimbeg, Priya and Schneider, Frank and Eschenhagen, Runa and Bae, Juhan and Sastry, Chandramouli Shama and Saroufim, Mark and Boyuan, Feng and Wright, Less and Yang, Edward Z. and Nado, Zachary and Medapati, Sourabh and Hennig, Philipp and Rabbat, Michael and Dahl, George E.},
booktitle       = {The Thirteenth International Conference on Learning Representations},
year            = {2025},
url             = {https://openreview.net/forum?id=CtM5xjRSfm}
}

@misc{hagele2024scalinglawscomputeoptimaltraining,
      title={Scaling Laws and Compute-Optimal Training Beyond Fixed Training Durations}, 
      author={Alexander Hägele and Elie Bakouch and Atli Kosson and Loubna Ben Allal and Leandro Von Werra and Martin Jaggi},
      year={2024},
      eprint={2405.18392},
      archivePrefix={arXiv},
      primaryClass={cs.LG},
      url={https://arxiv.org/abs/2405.18392}, 
}

@misc{
semenov2026benchmarking,
title={Benchmarking Optimizers for Large Language Model Pretraining},
author={Andrei Semenov and Matteo Pagliardini and Martin Jaggi},
year={2026},
url={https://openreview.net/forum?id=Jw7khYzYzl}
}

@inproceedings{
song2025through,
title={Through the River: Understanding the Benefit of Schedule-Free Methods for Language Model Training},
author={Minhak Song and Beomhan Baek and Kwangjun Ahn and Chulhee Yun},
booktitle={High-dimensional Learning Dynamics 2025},
year={2025},
url={https://openreview.net/forum?id=b5HYeRzG9M}
}

@inproceedings{kingma2015adam,
  author    = {Kingma, Diederik P. and Ba, Jimmy},
  title     = {Adam: A Method for Stochastic Optimization},
  booktitle = {International Conference on Learning Representations (ICLR)},
  year      = {2015},
  url       = {https://arxiv.org/abs/1412.6980}
}

@inproceedings{
loshchilov2018decoupled,
title={Decoupled Weight Decay Regularization},
author={Ilya Loshchilov and Frank Hutter},
booktitle={International Conference on Learning Representations},
year={2019},
url={https://openreview.net/forum?id=Bkg6RiCqY7},
}

@misc{dangelo2024,
      title={Why Do We Need Weight Decay in Modern Deep Learning?}, 
      author={Francesco D'Angelo and Maksym Andriushchenko and Aditya Varre and Nicolas Flammarion},
      year={2024},
      eprint={2310.04415},
      archivePrefix={arXiv},
      primaryClass={cs.LG},
      url={https://arxiv.org/abs/2310.04415}, 
}

@misc{qiu2026,
      title={Hyperparameter Transfer Enables Consistent Gains of Matrix-Preconditioned Optimizers Across Scales}, 
      author={Shikai Qiu and Zixi Chen and Hoang Phan and Qi Lei and Andrew Gordon Wilson},
      year={2026},
      eprint={2512.05620},
      archivePrefix={arXiv},
      primaryClass={cs.LG},
      url={https://arxiv.org/abs/2512.05620}, 
}

@misc{liu2025muonscalablellmtraining,
      title={Muon is Scalable for LLM Training}, 
      author={Jingyuan Liu and Jianlin Su and Xingcheng Yao and Zhejun Jiang and Guokun Lai and Yulun Du and Yidao Qin and Weixin Xu and Enzhe Lu and Junjie Yan and Yanru Chen and Huabin Zheng and Yibo Liu and Shaowei Liu and Bohong Yin and Weiran He and Han Zhu and Yuzhi Wang and Jianzhou Wang and Mengnan Dong and Zheng Zhang and Yongsheng Kang and Hao Zhang and Xinran Xu and Yutao Zhang and Yuxin Wu and Xinyu Zhou and Zhilin Yang},
      year={2025},
      eprint={2502.16982},
      archivePrefix={arXiv},
      primaryClass={cs.LG},
      url={https://arxiv.org/abs/2502.16982}, 
}

@misc{kimiteam2026kimik25visualagentic,
      title={Kimi K2.5: Visual Agentic Intelligence}, 
      author={Kimi Team},
      year={2026},
      eprint={2602.02276},
      archivePrefix={arXiv},
      primaryClass={cs.CL},
      url={https://arxiv.org/abs/2602.02276}, 
}

@misc{singh2026arceetrinitylargetechnical,
      title={Arcee Trinity Large Technical Report}, 
      author={Varun Singh and Lucas Krauss and Sami Jaghouar and Matej Sirovatka and Charles Goddard and Fares Obied and Jack Min Ong and Jannik Straube and Fern and Aria Harley and Conner Stewart and Colin Kealty and Maziyar Panahi and Simon Kirsten and Anushka Deshpande and Anneketh Vij and Arthur Bresnu and Pranav Veldurthi and Raghav Ravishankar and Hardik Bishnoi and DatologyAI Team and Arcee AI Team and Prime Intellect Team and Mark McQuade and Johannes Hagemann and Lucas Atkins},
      year={2026},
      eprint={2602.17004},
      archivePrefix={arXiv},
      primaryClass={cs.LG},
      url={https://arxiv.org/abs/2602.17004}, 
}

@misc{glm5team2026glm5vibecodingagentic,
      title={GLM-5: from Vibe Coding to Agentic Engineering}, 
      author={GLM-5-Team},
      year={2026},
      eprint={2602.15763},
      archivePrefix={arXiv},
      primaryClass={cs.LG},
      url={https://arxiv.org/abs/2602.15763}, 
}

@misc{jordan2024muon,
  author       = {Jordan, Keller and others},
  title        = {Muon: An optimizer for hidden layers in neural networks},
  year         = {2024},
  howpublished = {\url{https://kellerjordan.github.io/posts/muon/}},
  note         = {Accessed: 2026-01-25}
}

@misc{jordan2024muon_repo,
  author       = {Jordan, Keller and others},
  title        = {Muon (GitHub repository): An optimizer for hidden layers in neural networks},
  year         = {2024},
  url          = {https://github.com/KellerJordan/Muon},
  note         = {GitHub repository, master branch},
  urldate      = {2026-01-25}
}

@article{ahn2025dion,
  title={Dion: Distributed Orthonormalized Updates},
  author={Ahn, Kwangjun and Xu, Byron and Abreu, Natalie and Langford, John},
  journal={arXiv preprint: 2504.05295},
  year={2025}
}

@misc{bernstein2024oldoptimizernewnorm,
      title={Old Optimizer, New Norm: An Anthology}, 
      author={Jeremy Bernstein and Laker Newhouse},
      year={2024},
      eprint={2409.20325},
      archivePrefix={arXiv},
      primaryClass={cs.LG},
      url={https://arxiv.org/abs/2409.20325}, 
}

@misc{gupta2018shampoopreconditionedstochastictensor,
      title={Shampoo: Preconditioned Stochastic Tensor Optimization}, 
      author={Vineet Gupta and Tomer Koren and Yoram Singer},
      year={2018},
      eprint={1802.09568},
      archivePrefix={arXiv},
      primaryClass={cs.LG},
      url={https://arxiv.org/abs/1802.09568}, 
}

@misc{davis2026spectralgradientupdateshelp,
      title={When do spectral gradient updates help in deep learning?}, 
      author={Damek Davis and Dmitriy Drusvyatskiy},
      year={2026},
      eprint={2512.04299},
      archivePrefix={arXiv},
      primaryClass={cs.LG},
      url={https://arxiv.org/abs/2512.04299}, 
}

@misc{Sagun2017,
Author = {Levent Sagun and Utku Evci and V. Ugur Guney and Yann Dauphin and Leon Bottou},
Title = {Empirical Analysis of the Hessian of Over-Parametrized Neural Networks},
Year = {2017},
Eprint = {arXiv:1706.04454},
}

@InProceedings{ghorbani19b,
  title = 	 {An Investigation into Neural Net Optimization via Hessian Eigenvalue Density},
  author =       {Ghorbani, Behrooz and Krishnan, Shankar and Xiao, Ying},
  booktitle = 	 {Proceedings of the 36th International Conference on Machine Learning},
  pages = 	 {2232--2241},
  year = 	 {2019},
  editor = 	 {Chaudhuri, Kamalika and Salakhutdinov, Ruslan},
  volume = 	 {97},
  series = 	 {Proceedings of Machine Learning Research},
  publisher =    {PMLR},
  url = 	 {https://proceedings.mlr.press/v97/ghorbani19b.html},
}

@misc{chang2026convergencemuon,
      title={On the Convergence of Muon and Beyond}, 
      author={Da Chang and Yongxiang Liu and Ganzhao Yuan},
      year={2026},
      eprint={2509.15816},
      archivePrefix={arXiv},
      primaryClass={cs.LG},
      url={https://arxiv.org/abs/2509.15816}, 
}

@misc{shen2026convergenceanalysismuon,
      title={On the Convergence Analysis of Muon}, 
      author={Wei Shen and Ruichuan Huang and Minhui Huang and Cong Shen and Jiawei Zhang},
      year={2026},
      eprint={2505.23737},
      archivePrefix={arXiv},
      primaryClass={stat.ML},
      url={https://arxiv.org/abs/2505.23737}, 
}

@misc{sato2025convergenceboundcriticalbatch,
      title={Convergence Bound and Critical Batch Size of Muon Optimizer}, 
      author={Naoki Sato and Hiroki Naganuma and Hideaki Iiduka},
      year={2025},
      eprint={2507.01598},
      archivePrefix={arXiv},
      primaryClass={cs.LG},
      url={https://arxiv.org/abs/2507.01598}, 
}

@misc{li2025noteconvergencemuon,
      title={A Note on the Convergence of Muon}, 
      author={Jiaxiang Li and Mingyi Hong},
      year={2025},
      eprint={2502.02900},
      archivePrefix={arXiv},
      primaryClass={math.OC},
      url={https://arxiv.org/abs/2502.02900}, 
}

@misc{li2025normuonmakingmuonefficient,
      title={NorMuon: Making Muon more efficient and scalable}, 
      author={Zichong Li and Liming Liu and Chen Liang and Weizhu Chen and Tuo Zhao},
      year={2025},
      eprint={2510.05491},
      archivePrefix={arXiv},
      primaryClass={cs.LG},
      url={https://arxiv.org/abs/2510.05491}, 
}

@misc{si2025adamuonadaptivemuonoptimizer,
      title={AdaMuon: Adaptive Muon Optimizer}, 
      author={Chongjie Si and Debing Zhang and Wei Shen},
      year={2025},
      eprint={2507.11005},
      archivePrefix={arXiv},
      primaryClass={cs.LG},
      url={https://arxiv.org/abs/2507.11005}, 
}

@article{chen2023symbolic,
  title={Symbolic discovery of optimization algorithms},
  author={Chen, Xiangning and Liang, Chen and Huang, Da and Real, Esteban and Wang, Kaiyuan and Pham, Hieu and Dong, Xuanyi and Luong, Thang and Hsieh, Cho-Jui and Lu, Yifeng and others},
  journal={Advances in neural information processing systems},
  volume={36},
  pages={49205--49233},
  year={2023}
}

@article{goyal2017accurate,
  title={Accurate, large minibatch sgd: Training imagenet in 1 hour},
  author={Goyal, Priya and Doll{\'a}r, Piotr and Girshick, Ross and Noordhuis, Pieter and Wesolowski, Lukasz and Kyrola, Aapo and Tulloch, Andrew and Jia, Yangqing and He, Kaiming},
  journal={arXiv preprint arXiv:1706.02677},
  year={2017}
}

@article{kalra2024warmup,
  title={Why warmup the learning rate? underlying mechanisms and improvements},
  author={Kalra, Dayal Singh and Barkeshli, Maissam},
  journal={Advances in Neural Information Processing Systems},
  volume={37},
  pages={111760--111801},
  year={2024}
}

@inproceedings{Krogh1991,
author = {Krogh, Anders and Hertz, John A.},
title = {A simple weight decay can improve generalization},
year = {1991},
isbn = {1558602224},
publisher = {Morgan Kaufmann Publishers Inc.},
address = {San Francisco, CA, USA},
booktitle = {Proceedings of the 5th International Conference on Neural Information Processing Systems},
pages = {950–957},
numpages = {8},
location = {Denver, Colorado},
series = {NIPS'91}
}

@misc{an2025asgoadaptivestructuredgradient,
      title={ASGO: Adaptive Structured Gradient Optimization}, 
      author={Kang An and Yuxing Liu and Rui Pan and Yi Ren and Shiqian Ma and Donald Goldfarb and Tong Zhang},
      year={2025},
      eprint={2503.20762},
      archivePrefix={arXiv},
      primaryClass={cs.LG},
      url={https://arxiv.org/abs/2503.20762}, 
}

@misc{touvron2023llama2openfoundation,
      title={Llama 2: Open Foundation and Fine-Tuned Chat Models}, 
      author={Hugo Touvron and Louis Martin and Kevin Stone and Peter Albert and Amjad Almahairi and Yasmine Babaei and Nikolay Bashlykov and Soumya Batra and Prajjwal Bhargava and Shruti Bhosale and Dan Bikel and Lukas Blecher and Cristian Canton Ferrer and Moya Chen and Guillem Cucurull and David Esiobu and Jude Fernandes and Jeremy Fu and Wenyin Fu and Brian Fuller and Cynthia Gao and Vedanuj Goswami and Naman Goyal and Anthony Hartshorn and Saghar Hosseini and Rui Hou and Hakan Inan and Marcin Kardas and Viktor Kerkez and Madian Khabsa and Isabel Kloumann and Artem Korenev and Punit Singh Koura and Marie-Anne Lachaux and Thibaut Lavril and Jenya Lee and Diana Liskovich and Yinghai Lu and Yuning Mao and Xavier Martinet and Todor Mihaylov and Pushkar Mishra and Igor Molybog and Yixin Nie and Andrew Poulton and Jeremy Reizenstein and Rashi Rungta and Kalyan Saladi and Alan Schelten and Ruan Silva and Eric Michael Smith and Ranjan Subramanian and Xiaoqing Ellen Tan and Binh Tang and Ross Taylor and Adina Williams and Jian Xiang Kuan and Puxin Xu and Zheng Yan and Iliyan Zarov and Yuchen Zhang and Angela Fan and Melanie Kambadur and Sharan Narang and Aurelien Rodriguez and Robert Stojnic and Sergey Edunov and Thomas Scialom},
      year={2023},
      eprint={2307.09288},
      archivePrefix={arXiv},
      primaryClass={cs.CL},
      url={https://arxiv.org/abs/2307.09288}, 
}

@misc{zhang2019rootmeansquarelayer,
      title={Root Mean Square Layer Normalization}, 
      author={Biao Zhang and Rico Sennrich},
      year={2019},
      eprint={1910.07467},
      archivePrefix={arXiv},
      primaryClass={cs.LG},
      url={https://arxiv.org/abs/1910.07467}, 
}

@misc{penedo2024finewebdatasetsdecantingweb,
      title={The FineWeb Datasets: Decanting the Web for the Finest Text Data at Scale}, 
      author={Guilherme Penedo and Hynek Kydlíček and Loubna Ben allal and Anton Lozhkov and Margaret Mitchell and Colin Raffel and Leandro Von Werra and Thomas Wolf},
      year={2024},
      eprint={2406.17557},
      archivePrefix={arXiv},
      primaryClass={cs.CL},
      url={https://arxiv.org/abs/2406.17557}, 
}

@article{radford2019language,
  title={Language models are unsupervised multitask learners},
  author={Radford, Alec and Wu, Jeffrey and Child, Rewon and Luan, David and Amodei, Dario and Sutskever, Ilya and others},
  year={2019},
  publisher={OpenAI}
}

@misc{zhang2024relu2,
      title={ReLU$^2$ Wins: Discovering Efficient Activation Functions for Sparse LLMs}, 
      author={Zhengyan Zhang and Yixin Song and Guanghui Yu and Xu Han and Yankai Lin and Chaojun Xiao and Chenyang Song and Zhiyuan Liu and Zeyu Mi and Maosong Sun},
      year={2024},
      eprint={2402.03804},
      archivePrefix={arXiv},
      primaryClass={cs.LG},
      url={https://arxiv.org/abs/2402.03804}, 
}

@misc{defazio2025gradientsrapidlyincreasenear,
      title={Why Gradients Rapidly Increase Near the End of Training}, 
      author={Aaron Defazio},
      year={2025},
      eprint={2506.02285},
      archivePrefix={arXiv},
      primaryClass={cs.LG},
      url={https://arxiv.org/abs/2506.02285}, 
}

@misc{vanlaarhoven2017l2regularizationversusbatch,
      title={L2 Regularization versus Batch and Weight Normalization}, 
      author={Twan van Laarhoven},
      year={2017},
      eprint={1706.05350},
      archivePrefix={arXiv},
      primaryClass={cs.LG},
      url={https://arxiv.org/abs/1706.05350}, 
}

@article{Choquette2021,
  title = {NVIDIA A100 Tensor Core GPU: Performance and Innovation},
  volume = {41},
  ISSN = {1937-4143},
  url = {http://dx.doi.org/10.1109/MM.2021.3061394},
  DOI = {10.1109/mm.2021.3061394},
  number = {2},
  journal = {IEEE Micro},
  publisher = {Institute of Electrical and Electronics Engineers (IEEE)},
  author = {Choquette,  Jack and Gandhi,  Wishwesh and Giroux,  Olivier and Stam,  Nick and Krashinsky,  Ronny},
  year = {2021},
  month = Mar,
  pages = {29–35}
}

@inproceedings{Ansel2024,
  series = {ASPLOS ’24},
  title = {PyTorch 2: Faster Machine Learning Through Dynamic Python Bytecode Transformation and Graph Compilation},
  url = {http://dx.doi.org/10.1145/3620665.3640366},
  DOI = {10.1145/3620665.3640366},
  booktitle = {Proceedings of the 29th ACM International Conference on Architectural Support for Programming Languages and Operating Systems,  Volume 2},
  publisher = {ACM},
  author = {Ansel,  Jason and Yang,  Edward and He,  Horace and Gimelshein,  Natalia and Jain,  Animesh and Voznesensky,  Michael and Bao,  Bin and Bell,  Peter and Berard,  David and Burovski,  Evgeni and Chauhan,  Geeta and Chourdia,  Anjali and Constable,  Will and Desmaison,  Alban and DeVito,  Zachary and Ellison,  Elias and Feng,  Will and Gong,  Jiong and Gschwind,  Michael and Hirsh,  Brian and Huang,  Sherlock and Kalambarkar,  Kshiteej and Kirsch,  Laurent and Lazos,  Michael and Lezcano,  Mario and Liang,  Yanbo and Liang,  Jason and Lu,  Yinghai and Luk,  C. K. and Maher,  Bert and Pan,  Yunjie and Puhrsch,  Christian and Reso,  Matthias and Saroufim,  Mark and Siraichi,  Marcos Yukio and Suk,  Helen and Zhang,  Shunting and Suo,  Michael and Tillet,  Phil and Zhao,  Xu and Wang,  Eikan and Zhou,  Keren and Zou,  Richard and Wang,  Xiaodong and Mathews,  Ajit and Wen,  William and Chanan,  Gregory and Wu,  Peng and Chintala,  Soumith},
  year = {2024},
  month = Apr,
  pages = {929–947},
  collection = {ASPLOS ’24}
}

@misc{su2023roformerenhancedtransformerrotary,
      title={RoFormer: Enhanced Transformer with Rotary Position Embedding}, 
      author={Jianlin Su and Yu Lu and Shengfeng Pan and Ahmed Murtadha and Bo Wen and Yunfeng Liu},
      year={2023},
      eprint={2104.09864},
      archivePrefix={arXiv},
      primaryClass={cs.CL},
      url={https://arxiv.org/abs/2104.09864}, 
}

@misc{you2019doeslearningratedecay,
      title={How Does Learning Rate Decay Help Modern Neural Networks?}, 
      author={Kaichao You and Mingsheng Long and Jianmin Wang and Michael I. Jordan},
      year={2019},
      eprint={1908.01878},
      archivePrefix={arXiv},
      primaryClass={cs.LG},
      url={https://arxiv.org/abs/1908.01878}, 
}

@misc{pethick2025trainingdeeplearningmodels,
      title={Training Deep Learning Models with Norm-Constrained LMOs}, 
      author={Thomas Pethick and Wanyun Xie and Kimon Antonakopoulos and Zhenyu Zhu and Antonio Silveti-Falls and Volkan Cevher},
      year={2025},
      eprint={2502.07529},
      archivePrefix={arXiv},
      primaryClass={cs.LG},
      url={https://arxiv.org/abs/2502.07529}, 
}

@misc{Karpathy2022,
  author = {Andrej Karpathy},
  title = {\text{NanoGPT}},
  year = {2022},
  publisher = {GitHub},
  journal = {GitHub repository},
  howpublished = {\url{https://github.com/karpathy/nanoGPT}},
  commit = {325be85d9be8c81b436728a420e85796c57dba7e}
}

@misc{modded_nanogpt_2024,
  author       = {Keller Jordan and Jeremy Bernstein and Brendan Rappazzo and
                  @fernbear.bsky.social and Boza Vlado and You Jiacheng and
                  Franz Cesista and Braden Koszarsky and @Grad62304977},
  title        = {modded-nanogpt: Speedrunning the NanoGPT baseline},
  year         = {2024},
  url          = {https://github.com/KellerJordan/modded-nanogpt}
}

@misc{ahn2025dion2simplemethodshrink,
      title={Dion2: A Simple Method to Shrink Matrix in Muon}, 
      author={Kwangjun Ahn and Noah Amsel and John Langford},
      year={2025},
      eprint={2512.16928},
      archivePrefix={arXiv},
      primaryClass={cs.LG},
      url={https://arxiv.org/abs/2512.16928}, 
}

@misc{ren2026rethinkinglanguagemodelscaling,
      title={Rethinking Language Model Scaling under Transferable Hypersphere Optimization}, 
      author={Liliang Ren and Yang Liu and Yelong Shen and Weizhu Chen},
      year={2026},
      eprint={2603.28743},
      archivePrefix={arXiv},
      primaryClass={cs.LG},
      url={https://arxiv.org/abs/2603.28743}, 
}

@misc{amsel2026polarexpressoptimalmatrix,
      title={The Polar Express: Optimal Matrix Sign Methods and Their Application to the Muon Algorithm}, 
      author={Noah Amsel and David Persson and Christopher Musco and Robert M. Gower},
      year={2026},
      eprint={2505.16932},
      archivePrefix={arXiv},
      primaryClass={cs.LG},
      url={https://arxiv.org/abs/2505.16932}, 
}

@misc{muon2boostingmuonadaptive,
      title={Muon$^2$: Boosting Muon via Adaptive Second-Moment Preconditioning}, 
      author={Ziyue Liu and Ruijie Zhang and Zhengyang Wang and Yequan Zhao and Yupeng Su and Zi Yang and Zheng Zhang},
      year={2026},
      eprint={2604.09967},
      archivePrefix={arXiv},
      primaryClass={cs.LG},
      url={https://arxiv.org/abs/2604.09967}, 
}

@misc{khaled2025muonbpfastermuonblockperiodic,
      title={MuonBP: Faster Muon via Block-Periodic Orthogonalization}, 
      author={Ahmed Khaled and Kaan Ozkara and Tao Yu and Mingyi Hong and Youngsuk Park},
      year={2025},
      eprint={2510.16981},
      archivePrefix={arXiv},
      primaryClass={cs.LG},
      url={https://arxiv.org/abs/2510.16981}, 
}

@inproceedings{shazeer2017outrageously,
  title={Outrageously large neural networks: The sparsely-gated mixture-of-experts layer},
  author={Shazeer, Noam and Mirhoseini, Azalia and Mahdavi, Krzysztof and Davis, Andrew and Le, Quoc and Hinton, Geoffrey and Dean, Jeff},
  booktitle={International Conference on Learning Representations},
  year={2017}
}

@inproceedings{lepikhin2021gshard,
  title={{GS}hard: Scaling giant models with conditional computation and automatic sharding},
  author={Lepikhin, Dmitry and Lee, HyoukJoong and Xu, Yuanzhong and Chen, Dehao and Firat, Orhan and Huang, Yanping and Krikun, Maxim and Shazeer, Noam and Chen, Zhifeng},
  booktitle={International Conference on Learning Representations},
  year={2021}
}

@article{fedus2022switch,
  title={Switch {T}ransformers: scaling to trillion parameter models with simple and efficient sparsity},
  author={Fedus, William and Zoph, Barret and Shazeer, Noam},
  journal={Journal of Machine Learning Research},
  volume={23},
  number={120},
  pages={1--39},
  year={2022}
}

@misc{kravatskiy2025kyfannormsbeyond,
      title={The Ky Fan Norms and Beyond: Dual Norms and Combinations for Matrix Optimization}, 
      author={Alexey Kravatskiy and Ivan Kozyrev and Nikolai Kozlov and Alexander Vinogradov and Daniil Merkulov and Ivan Oseledets},
      year={2025},
      eprint={2512.09678},
      archivePrefix={arXiv},
      primaryClass={math.OC},
      url={https://arxiv.org/abs/2512.09678}, 
}

@misc{dosovitskiy2021imageworth16x16words,
      title={An Image is Worth 16x16 Words: Transformers for Image Recognition at Scale}, 
      author={Alexey Dosovitskiy and Lucas Beyer and Alexander Kolesnikov and Dirk Weissenborn and Xiaohua Zhai and Thomas Unterthiner and Mostafa Dehghani and Matthias Minderer and Georg Heigold and Sylvain Gelly and Jakob Uszkoreit and Neil Houlsby},
      year={2021},
      eprint={2010.11929},
      archivePrefix={arXiv},
      primaryClass={cs.CV},
      url={https://arxiv.org/abs/2010.11929}, 
}

@INPROCEEDINGS{Deng2009,
  author={Deng, Jia and Dong, Wei and Socher, Richard and Li, Li-Jia and Kai Li and Li Fei-Fei},
  booktitle={2009 IEEE Conference on Computer Vision and Pattern Recognition}, 
  title={ImageNet: A large-scale hierarchical image database}, 
  year={2009},
  volume={},
  number={},
  pages={248-255},
  doi={10.1109/CVPR.2009.5206848}}

\newpage 
\appendix
\crefalias{section}{appendix}
\noindent\makebox[\linewidth]{\rule{\linewidth}{1.5pt}}
\section*{\centering \large Technical Appendices and Supplementary Material for\\ ``Anytime Training with Schedule-Free Spectral Optimization''}
\noindent\makebox[\linewidth]{\rule{\linewidth}{1.5pt}}

In this appendix, we provide more thorough preliminaries and background information that supplement the main text. We also state the formal version of the informal theorems from the main text, as well as their proofs. This appendix is organized as follows: 
\begin{itemize}
    \item In \cref{app:app-metrized}, we review metrized deep learning, and derive optimizers from norms on weight matrices. 
    \item In \cref{app:sf-review} , we review the Schedule-free method. In particular, we discuss SF-SGD and its convergence bound, and the algorithm for SF-AdamW. 
    \item In \cref{app:convergence} proves in the detail the convergence of SF-Spectral Descent, for noisy Lipschitz smooth functions in both the Frobenius and spectral norm. 
    In addition, we prove the lemmas used in the steady-state analysis. 
    \item In \cref{app:arch-hypers},  we review the architecture of the language models used in our experiments and present the hyper-parameter selection procedure. 
    \item In \cref{app:ablation}, studies the effect of ablating momentum term for SF-NorMuon (setting $\mu=0$), and ablating the row-normalization. We also compare against NorMuon with a WSD schedule and with EMA weights, and report seed-to-seed variability.  
    \item In \cref{app:vit-hypers}, we train ViT-B/16 on ImageNet-1K and report the architecture, training setup, and hyperparameter sweeps.
    \item In \cref{app:implementation}, we present a PyTorch implementation of SF-NorMuon algorithm, together with a wall-clock comparison of all optimizers. 
\end{itemize}

\section{Optimizers from Weight-Space Matrix Geometry}\label{app:app-metrized}

For the reader's convenience, we collect here the basic derivations of the
update rules discussed in the main text. The starting point is always the same: given a
current iterate \(W_t\) and gradient \(G_t := \nabla f(W_t)\), we choose an update
\(\Delta W_t\) by minimizing the first-order model of the loss subject to a constraint,
or equivalently a quadratic penalty, in the geometry of interest. Our treatment follows the work of Bernstein and Newhouse \cite{bernstein2024oldoptimizernewnorm}. 

\subsection{Steepest descent in a prescribed geometry}

Let \(\|\cdot\|_{\mathcal X}\) be any norm on the space of matrices. The corresponding
norm-constrained steepest-descent step is
\begin{equation}
\Delta W_t
\in
\arg\min_{\|\Delta W\|_{\mathcal X}\le \eta}
\langle G_t,\Delta W\rangle.
\label{eq:general_constrained_step}
\end{equation}
Since the objective is linear, the solution lies on the boundary of the ball and points in
an extremal descent direction for the chosen norm. Writing
\begin{equation}
\|G_t\|_{\mathcal X,*}
:=
\sup_{\|U\|_{\mathcal X}\le 1}\langle G_t,U\rangle
\end{equation}
for the dual norm, \eqref{eq:general_constrained_step} is equivalent to
\begin{equation}
\Delta W_t = -\eta\, U_t,
\qquad
U_t\in\arg\max_{\|U\|_{\mathcal X}\le 1}\langle G_t,U\rangle,
\label{eq:duality_map_general}
\end{equation}
and the optimal value is \(-\eta \|G_t\|_{\mathcal X,*}\).

In practice one often uses the equivalent regularized local model
\begin{equation}
\Delta W_t
\in
\arg\min_{\Delta W}
\left\{
\langle G_t,\Delta W\rangle
+
\frac{1}{2\eta}\|\Delta W\|_{\mathcal X,t}^2
\right\},
\label{eq:general_regularized_step}
\end{equation}
where \(\|\cdot\|_{\mathcal X,t}\) may depend on time through a momentum estimate or a
preconditioner. The constrained form \eqref{eq:general_constrained_step} produces
normalized directions such as sign and polar updates, while the regularized form
\eqref{eq:general_regularized_step} produces the preconditioned-gradient forms used by
practical optimizers.

\subsection{The sign geometry: \(p=1\), \(q=\infty\)}

Take the matrix norm
\begin{equation}
\|A\|_{\ell_1\to\ell_\infty}
=
\sup_{\|u\|_1\le 1}\|Au\|_\infty.
\end{equation}
A direct calculation shows that this is simply the entrywise max norm:
\begin{equation}
\|A\|_{\ell_1\to\ell_\infty}
=
\max_{i,j}|A_{ij}|.
\label{eq:l1_to_linf_equals_max}
\end{equation}
Indeed, for each row \(i\),
\begin{equation}
|(Au)_i|
=
\left|\sum_j A_{ij}u_j\right|
\le
\max_j |A_{ij}|\sum_j |u_j|
\le
\max_j |A_{ij}|,
\end{equation}
so \(\|Au\|_\infty \le \max_{i,j}|A_{ij}|\). Equality is achieved by choosing \(u\) to be
a signed coordinate vector supported on an entry attaining the maximum.

Substituting \eqref{eq:l1_to_linf_equals_max} into
\eqref{eq:general_constrained_step} gives
\begin{equation}
\min_{\max_{i,j}|\Delta W_{ij}|\le \eta}
\sum_{i,j} (G_t)_{ij}(\Delta W)_{ij}.
\end{equation}
This decouples coordinatewise, and for each entry the minimizer is
\begin{equation}
(\Delta W_t)_{ij}
=
-\eta\,\mathrm{sign}\bigl((G_t)_{ij}\bigr).
\end{equation}
Hence the steepest-descent rule is
\begin{equation}
\Delta W_t
=
-\eta\,\mathrm{sign}(G_t),
\label{eq:sign_sgd_derivation}
\end{equation}
which is the base update underlying signSGD:
\begin{equation}
W_{t+1}=W_t-\eta\,\mathrm{sign}(G_t).
\label{eq:signsgd_update_appendix}
\end{equation}

\subsection{Momentum in the sign geometry}
To overcome the noise from the computing the gradient on a mini-batch a standard practical modification is to replace the raw gradient by a smoothed momentum
estimate. Let
\begin{equation}
M_t=\beta M_{t-1}+(1-\beta)G_t,
\qquad 0\le \beta <1.
\label{eq:ema_momentum_appendix}
\end{equation}
Applying the same steepest-descent rule to \(M_t\) instead of \(G_t\) yields
\begin{equation}
W_{t+1}=W_t-\eta\,\mathrm{sign}(M_t),
\label{eq:sign_momentum_update}
\end{equation}
which is the update for signSGD with momentum.

\paragraph{Lion.}
Lion uses a sign update applied to a momentum direction that directly incorporates current gradient as well. In the standard implementation, one first forms
\begin{equation}
C_t = \beta_1 M_{t-1} + (1-\beta_1) G_t,
\label{eq:lion_predictor}
\end{equation}
then updates
\begin{equation}
W_{t+1}=W_t-\eta\,\mathrm{sign}(C_t),
\label{eq:lion_update}
\end{equation}
and finally refreshes the state
\begin{equation}
M_t=\beta_2 M_{t-1}+(1-\beta_2)G_t.
\label{eq:lion_momentum}
\end{equation}
Thus Lion can be viewed as a two-timescale version of the basic sign
steepest-descent rule.

\paragraph{Adam with \(\beta_1=\beta_2=0\).}
Adam augments the gradient with both a momentum buffer and a variance buffer:
\begin{equation}
m_t=\beta_1 m_{t-1}+(1-\beta_1)G_t,
\qquad
v_t=\beta_2 v_{t-1}+(1-\beta_2)G_t^{\odot 2},
\label{eq:adam_moments_appendix}
\end{equation}
where \(G_t^{\odot 2}\) denotes the entrywise square. Its coordinatewise update is
\begin{equation}
W_{t+1}
=
W_t-\eta\,\frac{m_t}{\sqrt{v_t}+\varepsilon}.
\label{eq:adam_update_appendix}
\end{equation}
When \(\beta_1=\beta_2=0\), both buffers become memoryless,
\begin{equation}
m_t = G_t, \qquad v_t = G_t^{\odot 2},
\end{equation}
and so, ignoring \(\varepsilon\),
\begin{equation}
W_{t+1}
=
W_t-\eta\,\mathrm{sign}(G_t).
\label{eq:adam_beta2_zero_appendix}
\end{equation}
Thus, memoryless Adam recovers signSGD exactly \cite{bernstein2024oldoptimizernewnorm}. Note that with \(\beta_2=0\) but \(\beta_1>0\) the update is \(m_t/(|G_t|+\varepsilon)\), which is \emph{not} \(\mathrm{sign}(m_t)\) in general; signSGD with momentum is instead obtained by applying the sign steepest-descent rule directly to the momentum buffer, as derived above. 

\subsection{The spectral geometry: \(p=q=2\)}

Now take the matrix norm
\begin{equation}
\|A\|_{\ell_2\to\ell_2} = \|A\|_{\op},
\end{equation}
the spectral norm. The norm-constrained step becomes
\begin{equation}
\Delta W_t
\in
\arg\min_{\|\Delta W\|_{\op}\le \eta}
\langle G_t,\Delta W\rangle.
\label{eq:spectral_constrained_problem}
\end{equation}
Let the singular value decomposition of \(G_t\) be
\begin{equation}
G_t = U_t \Sigma_t V_t^\top.
\end{equation}
By von Neumann's trace inequality,
\begin{equation}
\langle G_t, U\rangle \le \|G_t\|_* \|U\|_{\op}
\qquad\text{for all }U,
\end{equation}
so over the unit spectral-norm ball the maximum of \(\langle G_t,U\rangle\) is
\(\|G_t\|_*\), attained by
\begin{equation}
U = U_t V_t^\top.
\end{equation}
This matrix is the polar factor of \(G_t\), which we denote by
\(\polar(G_t)\). Therefore
\begin{equation}
\Delta W_t
=
-\eta\,\polar(G_t),
\label{eq:spectral_descent_derivation}
\end{equation}
and the corresponding update for spectral descent is
\begin{equation}
W_{t+1}=W_t-\eta\,\polar(G_t).
\label{eq:spectral_descent_update_appendix}
\end{equation}

\subsection{Momentum in the spectral geometry}
\paragraph{Muon.}
Exactly as in the sign case, we may replace \(G_t\) by a momentum-smoothed matrix
\begin{equation}
M_t=\beta M_{t-1}+(1-\beta)G_t.
\label{eq:spectral_momentum_state}
\end{equation}
Applying spectral steepest descent to \(M_t\) gives the update for Muon
\begin{equation}
W_{t+1}=W_t-\eta\,\polar(M_t).
\label{eq:spectral_momentum_update}
\end{equation}
In practical implementations, this polar factor is not computed by an exact SVD but
is approximated by a small number of Newton-Schulz iterations. This minimizes the computational overhead as they can be implemented using only matrix
multiplications (GEMMs). 

\paragraph{Shampoo.}
Shampoo maintains left and right second-moment buffers for each matrix parameter and
updates by preconditioning with their inverse quarter-powers 
\begin{equation}
M_t=\beta_1 M_{t-1}+(1-\beta_1)G_t,\qquad
L_t=\beta_2 L_{t-1}+(1-\beta_2)M_tM_t^\top,\qquad
R_t=\beta_2 R_{t-1}+(1-\beta_2)M_t^\top M_t,
\end{equation}
followed by the update
\begin{equation}
W_{t+1}=W_t-\eta\,L_t^{-1/4}M_tR_t^{-1/4}.
\end{equation}
In the limiting case \(\beta_2=0\), these buffers
collapse to the instantaneous matrices
\begin{equation}
L_t=M_tM_t^\top,\qquad R_t=M_t^\top M_t.
\end{equation}
If \(M_t=U\Sigma V^\top\) is the singular value decomposition, then
\begin{equation}
L_t^{-1/4}M_tR_t^{-1/4}
=
U\Sigma^{-1/2}U^\top\,(U\Sigma V^\top)\,V\Sigma^{-1/2}V^\top
=
UV^\top
=
\polar(M_t),
\end{equation}
where the identity is exact in the idealized full-rank case (and otherwise
understood with pseudo-inverses on the support of \(M_t\)). Hence, when the second-moment
buffers are collapsed in this way, Shampoo reduces to the same polar transform of the
momentum matrix that underlies spectral descent and Muon. 

\section{Review of SF-SGD and SF-AdamW}\label{app:sf-review}

Schedule-free (SF) methods were introduced to remove the need for a prescribed
training horizon in learning-rate scheduling. The basic idea is to keep a
\emph{fast} sequence \(z_t\), on which the base optimizer runs at a nearly
constant learning rate, while returning an \emph{averaged} sequence \(x_t\)
whose effective learning rate decays automatically through online averaging.
Gradients are evaluated at an interpolation point \(y_t\) between these two
states. In this way, schedule-free optimization replaces explicit time-based
annealing by implicit annealing through online weight averaging
\cite{defazio2024road}.

\paragraph{SF-SGD.}
In its simplest form, Schedule-Free SGD maintains three sequences,
\begin{equation}
y_t=\beta x_t+(1-\beta)z_t,\qquad
g_t=\nabla \mathcal{L}(y_t;\xi_t),\qquad
z_{t+1}=z_t-\eta g_t,
\end{equation}
together with the online average
\begin{equation}
x_{t+1}=(1-c_{t+1})x_t+c_{t+1}z_{t+1},
\qquad c_{t+1}=\frac{1}{t+1}.
\end{equation}
Here \(z_t\) is the base iterate, \(x_t\) is the returned iterate, and \(y_t\)
is the point at which the stochastic gradient is computed. The parameter
\(\beta\in[0,1]\) interpolates between Polyak-Ruppert averaging
(\(\beta=0\)) and primal averaging (\(\beta=1\)).

\paragraph{Implicit decay of the effective learning rate.}
The averaging update can be rewritten as
\begin{equation}
x_{t+1}-x_t
=
c_{t+1}(z_{t+1}-x_t)
=
c_{t+1}(z_t-x_t)-c_{t+1}\eta g_t.
\end{equation}
Using
\begin{equation}
y_t=\beta x_t+(1-\beta)z_t
\qquad\Longrightarrow\qquad
z_t-x_t=\frac{y_t-x_t}{1-\beta},
\end{equation}
we obtain
\begin{equation}
x_{t+1}-x_t
=
\frac{c_{t+1}}{1-\beta}(y_t-x_t)-c_{t+1}\eta g_t.
\end{equation}
Thus the returned iterate \(x_t\) moves as though it were taking a gradient step
with \emph{effective} learning rate 
\begin{equation}
\eta_t^{\mathrm{eff}} \approx c_{t+1}\eta = \frac{\eta}{t+1}.
\end{equation}
Thus, the effective learning rate decays like \(1/t\) even though the base update
on \(z_t\) uses a constant step size. This is the basic mechanism by which
schedule-free methods replace an explicit schedule. In the convex Lipschitz setting, this implicit learning rate decay is sufficient to retain the optimal worst-case convergence rate of \(O(T^{-1/2})\). For a convex \(G\)-Lipschitz function with 
\(D=\|x_1-x_\star\|\), the choice of \(\eta=D/(G\sqrt{T})\) for any \(\beta\in[0,1]\) leads to
\begin{equation}
\mathbb{E}\bigl[F(x_T)-F(x_\star)\bigr]\le \frac{DG}{\sqrt{T}}.
\end{equation}

\paragraph{SF-AdamW.}

For deep-learning applications, the raw gradient is replaced by a preconditioned
update together with warmup and decoupled weight decay. 

\begin{algorithm}[htb]
\caption{Schedule-Free AdamW}
\label{alg:sf_adamw_review}
\begin{algorithmic}[1]
\State \textbf{Input:} base learning rate \(\eta\), interpolation parameter \(\beta_1\), variance parameter \(\beta_2\), weight decay \(\lambda\), warmup steps \(T_{\mathrm{warmup}}\), numerical constant \(\varepsilon\), decay location \(\tilde y_t\in\{y_t,z_t\}\)
\For{$t=1,2,\dots,T$}
    \State \(y_t \gets (1-\beta_1)z_t + \beta_1 x_t\) \Comment{Momentum via interpolation}
    \State \(g_t \gets \nabla \mathcal{L}(y_t;\xi_t)\) \Comment{Gradient is evaluated at \(y_t\)}
    \State \(v_t \gets \beta_2 v_{t-1} + (1-\beta_2) g_t^{\odot 2}\)
    \State \(\eta_t \gets \eta \sqrt{1-\beta_2^t}\,\min\!\left(1,t /T_{\mathrm{warmup}}\right)\) \Comment{Warmup and Adam bias-correction}
    \State \(z_{t+1} \gets z_t - \eta_t\, g_t / (\sqrt{v_t}+\varepsilon) - \eta_t \lambda \tilde y_t\)
    \State \(s_t \gets s_{t-1} + \eta_t^2\)
    \State \(c_{t+1} \gets \eta_t^2 / s_t\)
    \State \(x_{t+1} \gets (1-c_{t+1})x_t + c_{t+1}z_{t+1}\) \Comment{Update weighted iterate average}
\EndFor
\State \textbf{return} \(x_T\)
\end{algorithmic}
\end{algorithm}

The choice \(\tilde y_t=y_t\) corresponds to applying weight decay at the
gradient location, which matches the interpretation of decay as an
\(\ell_2\)-regularizer in the loss, while \(\tilde y_t=z_t\) applies decay at
the fast iterate. A useful practical feature is that SF-AdamW does not require
more memory than AdamW. One stores \(z_t\) and \(y_t\) in memory,
and reconstructs
\begin{equation}
x_t=\frac{y_t-(1-\beta_1)z_t}{\beta_1}
\end{equation}
when evaluation is needed. Note that unlike classical momentum, which introduces a separate momentum buffer, schedule-free momentum is implemented through the interpolation \(y_t=(1-\beta)z_t+\beta x_t\) in weight space.  

\section{Convergence of Schedule-Free Spectral optimization}\label{app:convergence}

We first state the iterations of the schedule-free spectral optimization and the assumptions of Lipschitz smoothness and bounded diameter.

\paragraph{Iterates.}
For parameters \(0\le \beta<1\), \(0\le \mu<1\), stepsize \(\eta>0\),
and batch size \(B\ge 1\), consider
\begin{align}
Y_t &= \beta X_t+(1-\beta)Z_t, \label{eq:Y_def_common}\\
G_t &= \frac1B\sum_{i=1}^B \nabla f(Y_t;\xi_{t,i}), \label{eq:G_def_common}\\
M_0 &= G_0,\qquad
M_t=\mu M_{t-1}+(1-\mu)G_t \quad (t\ge 1), \label{eq:M_def_common}\\
P_t &= \polar(M_t), \label{eq:P_def_common}\\
Z_{t+1} &= Z_t-\eta P_t, \label{eq:Z_def_common}\\
X_{t+1} &= \frac{t}{t+1}X_t+\frac{1}{t+1}Z_{t+1},
\qquad X_0=Z_0. \label{eq:X_def_common}
\end{align}

Assume there exists \(W^\star\in\R^{m\times n}\) such that
\[
f(W^\star)=f^\star,
\qquad
\nabla f(W^\star)=0, 
\]

and let 
\[
r:=\min\{m,n\},
\qquad
\Delta:=f(Y_0)-f^\star.
\]

Throughout, the stochastic oracle is unbiased with bounded Frobenius variance:
\begin{equation}
\EE[G_t\mid \mathcal F_t]=\nabla f(Y_t)=:g_t,
\qquad
\EE\!\left[\|G_t-g_t\|_F^2\mid \mathcal F_t\right]\le \frac{\sigma^2}{B}.
\label{eq:variance_assumption_common}
\end{equation}

\paragraph{Smoothness assumptions.}
For the spectral proof, we assume there exists a constant $L_S$ such that 
\begin{equation}
\|\nabla f(A)-\nabla f(B)\|_*
\le
L_S\|A-B\|_{\op}
\qquad\text{for all }A,B\in\R^{m\times n}.
\label{eq:spectral_smoothness}
\end{equation}

For the Frobenius proof, we assume there exists a constant $L_F$ such that 
\begin{equation}
\|\nabla f(A)-\nabla f(B)\|_F
\le
L_F\|A-B\|_F
\qquad\text{for all }A,B\in\R^{m\times n}.
\label{eq:frobenius_smoothness}
\end{equation}

\paragraph{Bounded diameter assumptions.}
For the spectral proof, we assume there exists a constant $D_S>0$ such that
\begin{equation}
\|W^\star\|_{\op} \le D_S/2,
\qquad
\|Z_t\|_{\op} \le D_S/2
\qquad \text{for all } t,
\label{eq:diameter_spectral}
\end{equation}
while for the Frobenius proof, we assume there exists a constant $D_F>0$ such that
\begin{equation}
\|W^\star\|_F \le D_F/2,
\qquad
\|Z_t\|_F \le D_F/2
\qquad \text{for all } t.
\label{eq:diameter_frobenius}
\end{equation}

Consequently, $\|X_t\|_{\op},\|Y_t\|_{\op}\le D_S/2$ for all $t$, and hence
\[
\|Y_t-W^\star\|_{\op} \le D_S,
\qquad
\|Z_t-W^\star\|_{\op} \le D_S
\qquad \text{for all } t, 
\]
and a similar statement holds for the Frobenius norm. 

These are standard bounded-iterate assumptions (see Theorem 7 in \cite{gupta2018shampoopreconditionedstochastictensor} and Theorem 1 in \cite{an2025asgoadaptivestructuredgradient}), and they can be enforced algorithmically by considering a projected variant that keep the fast iterates $Z_t$ inside the corresponding norm ball. The iterates $X, Y$ are unweighted and weighted averages of current and previous $Z$ iterates and hence remain within the ball. 

\subsection{Reusable identities and bounds}

\begin{lemma}[Iterate identities]
\label{lem:iterate_identities}
Define
\[
S_t:=Z_t-Y_t.
\]
Then \(S_t=\beta(Z_t-X_t)\), and for every \(t\ge 0\),
\begin{align}
S_{t+1}
&=
\frac{t}{t+1}\bigl(S_t-\beta\eta P_t\bigr),
\label{eq:S_recursion_common}\\
Y_{t+1}-Y_t
&=
\frac{S_t}{t+1}
-\eta\left(1-\beta+\frac{\beta}{t+1}\right)P_t.
\label{eq:Y_increment_common}
\end{align}
Consequently, if \(\|\cdot\|_\diamond\) is any norm for which
\(\|Y_t-W^\star\|_\diamond\le D_\diamond\) and
\(\|Z_t-W^\star\|_\diamond\le D_\diamond\), then
\begin{equation}
\|S_t\|_\diamond\le 2D_\diamond,
\qquad
\|Y_{t+1}-Y_t\|_\diamond
\le
\frac{2D_\diamond}{t+1}
+\eta\|P_t\|_\diamond.
\label{eq:Y_increment_generic_norm}
\end{equation}
In particular,
\begin{align}
\|Y_{t+1}-Y_t\|_{\op}
&\le
\frac{2D_S}{t+1}+\eta,
\label{eq:Y_increment_spectral}\\
\|Y_{t+1}-Y_t\|_F
&\le
\frac{2D_F}{t+1}+\eta\sqrt r.
\label{eq:Y_increment_frobenius}
\end{align}
\end{lemma}

\begin{proof}
Since \(Y_t=\beta X_t+(1-\beta)Z_t\), we have
\[
S_t=Z_t-Y_t=\beta(Z_t-X_t).
\]
Using \eqref{eq:Z_def_common} and \eqref{eq:X_def_common},
\begin{align*}
S_{t+1}
&=Z_{t+1}-Y_{t+1}\\
&=\beta(Z_{t+1}-X_{t+1})\\
&=
\beta\left(Z_{t+1}-\frac{t}{t+1}X_t-\frac{1}{t+1}Z_{t+1}\right)\\
&=
\frac{t}{t+1}\beta(Z_{t+1}-X_t)\\
&=
\frac{t}{t+1}\beta\bigl((Z_t-\eta P_t)-X_t\bigr)\\
&=
\frac{t}{t+1}\bigl(S_t-\beta\eta P_t\bigr),
\end{align*}
which proves \eqref{eq:S_recursion_common}. Then
\begin{align*}
Y_{t+1}-Y_t
&=(Z_{t+1}-S_{t+1})-(Z_t-S_t)\\
&=
-\eta P_t-S_{t+1}+S_t\\
&=
\frac{S_t}{t+1}
-\eta\left(1-\beta+\frac{\beta}{t+1}\right)P_t,
\end{align*}
which is \eqref{eq:Y_increment_common}. The norm bounds follow from
\[
\|S_t\|_\diamond
=
\|Z_t-Y_t\|_\diamond
\le
\|Z_t-W^\star\|_\diamond+\|Y_t-W^\star\|_\diamond
\le 2D_\diamond,
\]
together with \(\|P_t\|_{\op}\le 1\) and \(\|P_t\|_F\le \sqrt r\).
\end{proof}

\begin{lemma}[Polar alignment]
\label{lem:polar_alignment}
For every \(t\ge 0\),
\begin{equation}
\langle g_t,P_t\rangle
\ge
\|g_t\|_* - 2\|g_t-M_t\|_*.
\label{eq:polar_alignment_bound}
\end{equation}
\end{lemma}

\begin{proof}
Since \(P_t=\polar(M_t)\),
\[
\langle M_t,P_t\rangle=\|M_t\|_*.
\]
Also,
\[
\langle g_t-M_t,P_t\rangle
\ge
-\|g_t-M_t\|_*\|P_t\|_{\op}
\ge
-\|g_t-M_t\|_*.
\]
Therefore
\[
\langle g_t,P_t\rangle
=
\langle M_t,P_t\rangle+\langle g_t-M_t,P_t\rangle
\ge
\|M_t\|_*-\|g_t-M_t\|_*.
\]
Finally,
\[
\|M_t\|_*
\ge
\|g_t\|_*-\|g_t-M_t\|_*,
\]
which yields \eqref{eq:polar_alignment_bound}.
\end{proof}

\begin{lemma}[Stochastic EMA error]
\label{lem:ema_noise}
Let
\[
C_0=g_0,
\qquad
C_t=\mu C_{t-1}+(1-\mu)g_t,
\qquad
E_t:=C_t-M_t.
\]
Then
\begin{equation}
E_t=\mu E_{t-1}+(1-\mu)(g_t-G_t),
\label{eq:E_recursion_common}
\end{equation}
and
\begin{equation}
\frac1T\sum_{t=0}^{T-1}\EE\|E_t\|_*
\le
\sqrt r\left(
\frac{\sigma}{(1-\mu)T\sqrt B}
+
\sqrt{\frac{1-\mu}{1+\mu}}\frac{\sigma}{\sqrt B}
\right).
\label{eq:E_avg_common}
\end{equation}
\end{lemma}

\begin{proof}
The recursion \eqref{eq:E_recursion_common} is immediate from the definitions.
Using \eqref{eq:variance_assumption_common},
\[
\EE\!\left[\|E_t\|_F^2\mid\mathcal F_t\right]
\le
\mu^2\|E_{t-1}\|_F^2+(1-\mu)^2\frac{\sigma^2}{B}.
\]
Solving this recursion and applying Jensen gives
\[
\EE\|E_t\|_F
\le
\mu^t\frac{\sigma}{\sqrt B}
+
\sqrt{\frac{1-\mu}{1+\mu}}\frac{\sigma}{\sqrt B}.
\]
Since \(\|A\|_*\le \sqrt r\,\|A\|_F\), averaging over \(t=0,\dots,T-1\) yields
\eqref{eq:E_avg_common}.
\end{proof}

\begin{lemma}[Tracking bound under spectral smoothness]
\label{lem:tracking_spectral}
Assume \eqref{eq:spectral_smoothness} and \eqref{eq:diameter_spectral}. Then
\begin{equation}
\frac1T\sum_{t=0}^{T-1}\EE\|g_t-M_t\|_*
\le
\frac{\mu L_S\eta}{1-\mu}
+
\frac{2\mu L_S D_S\log(eT)}{(1-\mu)T}
+
\sqrt r\left(
\frac{\sigma}{(1-\mu)T\sqrt B}
+
\sqrt{\frac{1-\mu}{1+\mu}}\frac{\sigma}{\sqrt B}
\right).
\label{eq:tracking_spectral_final}
\end{equation}
\end{lemma}

\begin{proof}
Write
\[
g_t-M_t=(g_t-C_t)+(C_t-M_t),
\qquad
A_t:=\|g_t-C_t\|_*.
\]
Then
\begin{align*}
A_t
&=
\|g_t-\mu C_{t-1}-(1-\mu)g_t\|_*\\
&=
\mu\|g_t-C_{t-1}\|_*\\
&\le
\mu\|g_t-g_{t-1}\|_*+\mu A_{t-1}\\
&\le
\mu L_S\|Y_t-Y_{t-1}\|_{\op}+\mu A_{t-1}\\
&\le
\mu L_S\left(\frac{2D_S}{t}+\eta\right)+\mu A_{t-1},
\end{align*}
where we used \eqref{eq:Y_increment_spectral}. Iterating and averaging gives
\[
\frac1T\sum_{t=0}^{T-1}A_t
\le
\frac{\mu L_S\eta}{1-\mu}
+
\frac{2\mu L_S D_S\log(eT)}{(1-\mu)T}.
\]
Combining this with Lemma~\ref{lem:ema_noise} proves
\eqref{eq:tracking_spectral_final}.
\end{proof}

\begin{lemma}[Tracking bound under Frobenius smoothness]
\label{lem:tracking_frobenius}
Assume \eqref{eq:frobenius_smoothness} and \eqref{eq:diameter_frobenius}. Then
\begin{equation}
\frac1T\sum_{t=0}^{T-1}\EE\|g_t-M_t\|_*
\le
\frac{\mu r L_F\eta}{1-\mu}
+
\frac{2\mu \sqrt r\,L_F D_F\log(eT)}{(1-\mu)T}
+
\sqrt r\left(
\frac{\sigma}{(1-\mu)T\sqrt B}
+
\sqrt{\frac{1-\mu}{1+\mu}}\frac{\sigma}{\sqrt B}
\right).
\label{eq:tracking_frobenius_final}
\end{equation}
\end{lemma}

\begin{proof}
Again write
\[
g_t-M_t=(g_t-C_t)+(C_t-M_t),
\qquad
A_t:=\|g_t-C_t\|_*.
\]
Then
\begin{align*}
A_t
&\le
\mu\|g_t-g_{t-1}\|_*+\mu A_{t-1}\\
&\le
\mu\sqrt r\,\|g_t-g_{t-1}\|_F+\mu A_{t-1}\\
&\le
\mu\sqrt r\,L_F\|Y_t-Y_{t-1}\|_F+\mu A_{t-1}\\
&\le
\mu\sqrt r\,L_F\left(\frac{2D_F}{t}+\eta\sqrt r\right)+\mu A_{t-1},
\end{align*}
where we used \eqref{eq:Y_increment_frobenius}. Iterating and averaging gives
\[
\frac1T\sum_{t=0}^{T-1}A_t
\le
\frac{\mu r L_F\eta}{1-\mu}
+
\frac{2\mu \sqrt r\,L_F D_F\log(eT)}{(1-\mu)T}.
\]
Combining with Lemma~\ref{lem:ema_noise} proves
\eqref{eq:tracking_frobenius_final}.
\end{proof}

\subsection{Spectral-smooth analysis}

\begin{theorem}[Stationarity at the \(Y\)-iterates under spectral smoothness]
\label{thm:Y_stationarity_spectral_noisy_reorg}
Assume \eqref{eq:variance_assumption_common},
\eqref{eq:spectral_smoothness}, and \eqref{eq:diameter_spectral}. Then for every
\(T\ge 1\),
\begin{align}
\frac1T\sum_{t=0}^{T-1}\EE\|\nabla f(Y_t)\|_*
\le {}&
\frac{\Delta + 2L_S D_S^2\!\left(\log(eT)+\frac{\pi^2}{6}\right)}
{(1-\beta)T\eta}
+ \frac{L_S\eta}{2(1-\beta)}
+ \frac{2\mu L_S\eta}{(1-\beta)(1-\mu)}
\notag\\
&\quad
+ \frac{2L_S D_S\log(eT)}{(1-\beta)T}
\left(1+\frac{2\mu}{1-\mu}\right)
+ \frac{2\sqrt r\,\sigma}{(1-\beta)\sqrt B}\sqrt{\frac{1-\mu}{1+\mu}}
\notag\\
&\quad
+ \frac{2\sqrt r\,\sigma}{(1-\beta)(1-\mu)T\sqrt B}.
\label{eq:Y_stat_bound_spectral_log}
\end{align}
\end{theorem}

\begin{proof}
By \eqref{eq:spectral_smoothness},
\[
f(Y_{t+1})
\le
f(Y_t)
+\langle g_t,Y_{t+1}-Y_t\rangle
+\frac{L_S}{2}\|Y_{t+1}-Y_t\|_{\op}^2.
\]
Using Lemma~\ref{lem:iterate_identities},
\[
Y_{t+1}-Y_t
=
\frac{S_t}{t+1}
-\eta\left(1-\beta+\frac{\beta}{t+1}\right)P_t,
\]
hence
\begin{align*}
f(Y_{t+1})
\le {}&
f(Y_t)
+\frac{1}{t+1}\langle g_t,S_t\rangle
-\eta\left(1-\beta+\frac{\beta}{t+1}\right)\langle g_t,P_t\rangle
+\frac{L_S}{2}\|Y_{t+1}-Y_t\|_{\op}^2.
\end{align*}
Now \(\|g_t\|_*\le L_S D_S\), \(\|S_t\|_{\op}\le 2D_S\),
\(\|Y_{t+1}-Y_t\|_{\op}\le 2D_S/(t+1)+\eta\), and
Lemma~\ref{lem:polar_alignment} gives
\[
\langle g_t,P_t\rangle
\ge
\|g_t\|_* - 2\|g_t-M_t\|_*.
\]
Therefore
\begin{equation}
f(Y_{t+1})
\le
f(Y_t)
-\eta(1-\beta)\|g_t\|_*
+2\eta\|g_t-M_t\|_*
+\frac{2L_S D_S^2}{t+1}
+\frac{L_S}{2}\left(\frac{2D_S}{t+1}+\eta\right)^2.
\label{eq:one_step_spectral_reorg}
\end{equation}
Summing \eqref{eq:one_step_spectral_reorg} from \(t=0\) to \(T-1\),
taking expectations, and using \(f(Y_T)\ge f^\star\), we obtain
\begin{align*}
\eta(1-\beta)\sum_{t=0}^{T-1}\EE\|g_t\|_*
\le {}&
\Delta
+2\eta\sum_{t=0}^{T-1}\EE\|g_t-M_t\|_*
+2L_S D_S^2\sum_{t=0}^{T-1}\frac1{t+1}\\
&\quad
+\frac{L_S}{2}\sum_{t=0}^{T-1}\left(\frac{2D_S}{t+1}+\eta\right)^2.
\end{align*}
Using
\[
\sum_{t=0}^{T-1}\frac1{t+1}\le \log(eT),
\qquad
\sum_{t=0}^{T-1}\frac1{(t+1)^2}\le \frac{\pi^2}{6},
\]
and substituting Lemma~\ref{lem:tracking_spectral} yields
\eqref{eq:Y_stat_bound_spectral_log}.
\end{proof}

\subsection{Frobenius-smooth analysis}

\begin{theorem}[Stationarity at the \(Y\)-iterates under Frobenius smoothness]
\label{thm:Y_stationarity_frobenius_noisy_reorg}
Assume \eqref{eq:variance_assumption_common},
\eqref{eq:frobenius_smoothness}, and \eqref{eq:diameter_frobenius}. Then for every
\(T\ge 1\),
\begin{align}
\frac1T\sum_{t=0}^{T-1}\EE\|\nabla f(Y_t)\|_*
\le {}&
\frac{\Delta + 2L_F D_F^2\!\left(\log(eT)+\frac{\pi^2}{6}\right)}
{(1-\beta)T\eta}
+ \frac{L_F r\,\eta}{2(1-\beta)}
+ \frac{2\mu L_F r\,\eta}{(1-\beta)(1-\mu)}
\notag\\
&\quad
+ \frac{2\sqrt r\,L_F D_F\log(eT)}{(1-\beta)T}
\left(1+\frac{2\mu}{1-\mu}\right)
+ \frac{2\sqrt r\,\sigma}{(1-\beta)\sqrt B}\sqrt{\frac{1-\mu}{1+\mu}}
\notag\\
&\quad
+ \frac{2\sqrt r\,\sigma}{(1-\beta)(1-\mu)T\sqrt B}.
\label{eq:Y_stat_bound_frobenius_log}
\end{align}
\end{theorem}

\begin{proof}
By \eqref{eq:frobenius_smoothness},
\[
f(Y_{t+1})
\le
f(Y_t)
+\langle g_t,Y_{t+1}-Y_t\rangle
+\frac{L_F}{2}\|Y_{t+1}-Y_t\|_F^2.
\]
Using Lemma~\ref{lem:iterate_identities},
\[
Y_{t+1}-Y_t
=
\frac{S_t}{t+1}
-\eta\left(1-\beta+\frac{\beta}{t+1}\right)P_t.
\]
Hence
\begin{align*}
f(Y_{t+1})
\le {}&
f(Y_t)
+\frac{1}{t+1}\langle g_t,S_t\rangle
-\eta\left(1-\beta+\frac{\beta}{t+1}\right)\langle g_t,P_t\rangle
+\frac{L_F}{2}\|Y_{t+1}-Y_t\|_F^2.
\end{align*}
Since \(\|g_t\|_F\le L_F D_F\), \(\|S_t\|_F\le 2D_F\),
\(\|Y_{t+1}-Y_t\|_F\le 2D_F/(t+1)+\eta\sqrt r\), and
Lemma~\ref{lem:polar_alignment} gives
\[
\langle g_t,P_t\rangle
\ge
\|g_t\|_* - 2\|g_t-M_t\|_*,
\]
we obtain
\begin{equation}
f(Y_{t+1})
\le
f(Y_t)
-\eta(1-\beta)\|g_t\|_*
+2\eta\|g_t-M_t\|_*
+\frac{2L_F D_F^2}{t+1}
+\frac{L_F}{2}\left(\frac{2D_F}{t+1}+\eta\sqrt r\right)^2.
\label{eq:one_step_frobenius_reorg}
\end{equation}
Summing \eqref{eq:one_step_frobenius_reorg} from \(t=0\) to \(T-1\),
taking expectations, and using \(f(Y_T)\ge f^\star\), we get
\begin{align*}
\eta(1-\beta)\sum_{t=0}^{T-1}\EE\|g_t\|_*
\le {}&
\Delta
+2\eta\sum_{t=0}^{T-1}\EE\|g_t-M_t\|_*
+2L_F D_F^2\sum_{t=0}^{T-1}\frac1{t+1}\\
&\quad
+\frac{L_F}{2}\sum_{t=0}^{T-1}\left(\frac{2D_F}{t+1}+\eta\sqrt r\right)^2.
\end{align*}
Using
\[
\sum_{t=0}^{T-1}\frac1{t+1}\le \log(eT),
\qquad
\sum_{t=0}^{T-1}\frac1{(t+1)^2}\le \frac{\pi^2}{6},
\]
and substituting Lemma~\ref{lem:tracking_frobenius} yields
\eqref{eq:Y_stat_bound_frobenius_log}.
\end{proof}

\subsection{Corollaries with optimized hyperparameters}

\begin{corollary}[Tuned rates under spectral smoothness]
\label{cor:spectral_tuned_rates}
Assume the hypotheses of
Theorem~\ref{thm:Y_stationarity_spectral_noisy_reorg}, and fix
\(\beta\in[0,1)\).

\begin{enumerate}
\item[(i)] \textbf{Noisy case.}
There exist choices \(\mu_T\in[0,1)\) and \(\eta_T>0\) such that
\begin{equation}
\frac1T\sum_{t=0}^{T-1}\EE\|\nabla f(Y_t)\|_*
=
O\!\left(\left(\frac{\log T}{T}\right)^{1/4}\right).
\label{eq:spectral_rate_noisy}
\end{equation}

\item[(ii)] \textbf{Noiseless case.}
If \(\sigma=0\), then there exist choices \(\mu_T\in[0,1)\) and
\(\eta_T>0\) such that
\begin{equation}
\frac1T\sum_{t=0}^{T-1}\|\nabla f(Y_t)\|_*
=
O\!\left(\sqrt{\frac{\log T}{T}}\right).
\label{eq:spectral_rate_noiseless}
\end{equation}

\item[(iii)] \textbf{Log-free specialization when \(\beta=0\).}
If \(\beta=0\), then the schedule-free drift disappears, and the rates improve to
\begin{equation}
\frac1T\sum_{t=0}^{T-1}\EE\|\nabla f(Y_t)\|_*
=
O\!\left(T^{-1/4}\right)
\qquad\text{in the noisy case,}
\label{eq:spectral_rate_beta_zero_noisy}
\end{equation}
and
\begin{equation}
\frac1T\sum_{t=0}^{T-1}\|\nabla f(Y_t)\|_*
=
O\!\left(T^{-1/2}\right)
\qquad\text{in the noiseless case.}
\label{eq:spectral_rate_beta_zero_noiseless}
\end{equation}
\end{enumerate}
\end{corollary}

\begin{proof}
Let
\[
A_{T,S}
:=
\Delta + 2L_S D_S^2\!\left(\log(eT)+\frac{\pi^2}{6}\right).
\]
Then \(A_{T,S}=O(\log T)\).

For part (i), write
\[
\alpha:=1-\mu\in(0,1].
\]
Then Theorem~\ref{thm:Y_stationarity_spectral_noisy_reorg} becomes
\begin{align}
\frac1T\sum_{t=0}^{T-1}\EE\|\nabla f(Y_t)\|_*
\le {}&
\frac{1}{1-\beta}
\Biggl[
\frac{A_{T,S}}{T\eta}
+
\frac{L_S(4-3\alpha)}{2\alpha}\eta
+
\frac{2L_S D_S\log(eT)}{T}\frac{2-\alpha}{\alpha}
\notag\\
&\hspace{3.25cm}
+
\frac{2\sqrt r\,\sigma}{\sqrt B}\sqrt{\frac{\alpha}{2-\alpha}}
+
\frac{2\sqrt r\,\sigma}{\alpha T\sqrt B}
\Biggr].
\label{eq:spectral_alpha_bound}
\end{align}
For fixed \(\alpha\), the exact optimizer in \(\eta\) is
\begin{equation}
\eta_{T,S}^\star(\alpha)
=
\sqrt{
\frac{2A_{T,S}\alpha}{L_S T(4-3\alpha)}
}.
\label{eq:spectral_eta_star_alpha}
\end{equation}
Substituting \eqref{eq:spectral_eta_star_alpha} into
\eqref{eq:spectral_alpha_bound} gives
\begin{align}
\frac1T\sum_{t=0}^{T-1}\EE\|\nabla f(Y_t)\|_*
\le {}&
\frac{1}{1-\beta}
\Biggl[
\sqrt{
\frac{2A_{T,S}L_S}{T}\cdot \frac{4-3\alpha}{\alpha}
}
+
\frac{2L_S D_S\log(eT)}{T}\frac{2-\alpha}{\alpha}
\notag\\
&\hspace{3.25cm}
+
\frac{2\sqrt r\,\sigma}{\sqrt B}\sqrt{\frac{\alpha}{2-\alpha}}
+
\frac{2\sqrt r\,\sigma}{\alpha T\sqrt B}
\Biggr].
\label{eq:spectral_plugged_alpha}
\end{align}

Choose
\begin{equation}
\alpha_{T,S}
:=
\min\left\{
1,\;
2\sqrt{\frac{A_{T,S}L_S B}{r\sigma^2 T}}
\right\},
\qquad
\mu_{T,S}:=1-\alpha_{T,S},
\qquad
\eta_{T,S}:=\eta_{T,S}^\star(\alpha_{T,S}).
\label{eq:spectral_mu_eta_choice}
\end{equation}
This is the leading-order optimal choice obtained by balancing the first and
third terms in \eqref{eq:spectral_plugged_alpha}. With this choice,
\[
\sqrt{
\frac{2A_{T,S}L_S}{T}\cdot \frac{4-3\alpha_{T,S}}{\alpha_{T,S}}
}
=
O\!\left(\left(\frac{A_{T,S}}{T}\right)^{1/4}\right),
\]
and
\[
\frac{2\sqrt r\,\sigma}{\sqrt B}\sqrt{\frac{\alpha_{T,S}}{2-\alpha_{T,S}}}
=
O\!\left(\left(\frac{A_{T,S}}{T}\right)^{1/4}\right).
\]
The remaining terms are lower order:
\[
\frac{2L_S D_S\log(eT)}{T}\frac{2-\alpha_{T,S}}{\alpha_{T,S}}
=
O\!\left(\sqrt{\frac{\log T}{T}}\right),
\qquad
\frac{2\sqrt r\,\sigma}{\alpha_{T,S}T\sqrt B}
=
O\!\left(\frac{1}{\sqrt{T\log T}}\right).
\]
Hence
\[
\frac1T\sum_{t=0}^{T-1}\EE\|\nabla f(Y_t)\|_*
=
O\!\left(\left(\frac{A_{T,S}}{T}\right)^{1/4}\right)
=
O\!\left(\left(\frac{\log T}{T}\right)^{1/4}\right),
\]
which proves part (i).

For part (ii), when \(\sigma=0\), the stochastic terms vanish from
\eqref{eq:spectral_alpha_bound}. The remaining bound is increasing in \(\mu\),
hence minimized by
\[
\mu_{T,S}=0
\qquad\text{equivalently}\qquad
\alpha_{T,S}=1.
\]
Then the exact optimizer in \(\eta\) is
\begin{equation}
\eta_{T,S}
=
\sqrt{\frac{2A_{T,S}}{L_S T}}.
\label{eq:spectral_eta_star_noiseless}
\end{equation}
Substituting \(\mu_{T,S}=0\) and \eqref{eq:spectral_eta_star_noiseless} into
Theorem~\ref{thm:Y_stationarity_spectral_noisy_reorg} gives
\[
\frac1T\sum_{t=0}^{T-1}\|\nabla f(Y_t)\|_*
\le
\frac{1}{1-\beta}
\left[
\sqrt{\frac{2L_S A_{T,S}}{T}}
+
\frac{2L_S D_S\log(eT)}{T}
\right].
\]
Since \(A_{T,S}=O(\log T)\), this implies
\[
\frac1T\sum_{t=0}^{T-1}\|\nabla f(Y_t)\|_*
=
O\!\left(\sqrt{\frac{\log T}{T}}\right),
\]
which proves part (ii).

For part (iii), suppose \(\beta=0\). Then \(Y_t=Z_t\), hence
\[
S_t=Z_t-Y_t=0
\qquad\text{for all }t,
\]
and therefore
\[
Y_{t+1}-Y_t=-\eta P_t.
\]
Thus the harmonic-drift term disappears from both the descent and tracking
arguments. Repeating the proof of
Theorem~\ref{thm:Y_stationarity_spectral_noisy_reorg} with \(S_t\equiv 0\) gives
\begin{align}
\frac1T\sum_{t=0}^{T-1}\EE\|\nabla f(Y_t)\|_*
\le {}&
\frac{\Delta}{T\eta}
+\frac{L_S\eta}{2}
+\frac{2\mu L_S\eta}{1-\mu}
+\frac{2\sqrt r\,\sigma}{\sqrt B}\sqrt{\frac{1-\mu}{1+\mu}}
+\frac{2\sqrt r\,\sigma}{(1-\mu)T\sqrt B}.
\label{eq:spectral_beta_zero_bound}
\end{align}
Writing \(\alpha:=1-\mu\), this becomes
\[
\frac1T\sum_{t=0}^{T-1}\EE\|\nabla f(Y_t)\|_*
\le
\frac{\Delta}{T\eta}
+\frac{L_S(4-3\alpha)}{2\alpha}\eta
+\frac{2\sqrt r\,\sigma}{\sqrt B}\sqrt{\frac{\alpha}{2-\alpha}}
+\frac{2\sqrt r\,\sigma}{\alpha T\sqrt B}.
\]
For fixed \(\alpha\), the exact optimizer in \(\eta\) is
\[
\eta^\star
=
\sqrt{\frac{2\Delta\,\alpha}{L_S T(4-3\alpha)}}.
\]
Substituting this back and balancing the first two leading terms gives
\[
\alpha_T
\asymp
\min\left\{
1,\;
2\sqrt{\frac{\Delta L_S B}{r\sigma^2 T}}
\right\},
\qquad
\mu_T:=1-\alpha_T,
\qquad
\eta_T:=\sqrt{\frac{2\Delta\,\alpha_T}{L_S T(4-3\alpha_T)}}.
\]
With this choice,
\[
\frac1T\sum_{t=0}^{T-1}\EE\|\nabla f(Y_t)\|_*
=
O(T^{-1/4}),
\]
which proves the noisy claim in \eqref{eq:spectral_rate_beta_zero_noisy}.

If in addition \(\sigma=0\), then \eqref{eq:spectral_beta_zero_bound} reduces to
\[
\frac1T\sum_{t=0}^{T-1}\|\nabla f(Y_t)\|_*
\le
\frac{\Delta}{T\eta}+\frac{L_S\eta}{2}.
\]
Its exact optimizer is
\[
\mu_T=0,
\qquad
\eta_T=\sqrt{\frac{2\Delta}{L_S T}}.
\]
Substituting yields
\[
\frac1T\sum_{t=0}^{T-1}\|\nabla f(Y_t)\|_*
\le
\sqrt{\frac{2L_S\Delta}{T}}
=
O(T^{-1/2}),
\]
which proves \eqref{eq:spectral_rate_beta_zero_noiseless}.
\end{proof}

\begin{corollary}[Tuned rates under Frobenius smoothness]
\label{cor:frobenius_tuned_rates}
Assume the hypotheses of
Theorem~\ref{thm:Y_stationarity_frobenius_noisy_reorg}, and fix
\(\beta\in[0,1)\).

\begin{enumerate}
\item[(i)] \textbf{Noisy case.}
There exist choices \(\mu_T\in[0,1)\) and \(\eta_T>0\) such that
\begin{equation}
\frac1T\sum_{t=0}^{T-1}\EE\|\nabla f(Y_t)\|_*
=
O\!\left(\left(\frac{\log T}{T}\right)^{1/4}\right).
\label{eq:frobenius_rate_noisy}
\end{equation}

\item[(ii)] \textbf{Noiseless case.}
If \(\sigma=0\), then there exist choices \(\mu_T\in[0,1)\) and
\(\eta_T>0\) such that
\begin{equation}
\frac1T\sum_{t=0}^{T-1}\|\nabla f(Y_t)\|_*
=
O\!\left(\sqrt{\frac{\log T}{T}}\right).
\label{eq:frobenius_rate_noiseless}
\end{equation}

\item[(iii)] \textbf{Log-free specialization when \(\beta=0\).}
If \(\beta=0\), then the schedule-free drift disappears, and the rates improve to
\begin{equation}
\frac1T\sum_{t=0}^{T-1}\EE\|\nabla f(Y_t)\|_*
=
O\!\left(T^{-1/4}\right)
\qquad\text{in the noisy case,}
\label{eq:frobenius_rate_beta_zero_noisy}
\end{equation}
and
\begin{equation}
\frac1T\sum_{t=0}^{T-1}\|\nabla f(Y_t)\|_*
=
O\!\left(T^{-1/2}\right)
\qquad\text{in the noiseless case.}
\label{eq:frobenius_rate_beta_zero_noiseless}
\end{equation}
\end{enumerate}
\end{corollary}

\begin{proof}
Let
\[
A_{T,F}
:=
\Delta + 2L_F D_F^2\!\left(\log(eT)+\frac{\pi^2}{6}\right).
\]
Then \(A_{T,F}=O(\log T)\).

For part (i), write
\[
\alpha:=1-\mu\in(0,1].
\]
Then Theorem~\ref{thm:Y_stationarity_frobenius_noisy_reorg} becomes
\begin{align}
\frac1T\sum_{t=0}^{T-1}\EE\|\nabla f(Y_t)\|_*
\le {}&
\frac{1}{1-\beta}
\Biggl[
\frac{A_{T,F}}{T\eta}
+
\frac{L_F r(4-3\alpha)}{2\alpha}\eta
+
\frac{2\sqrt r\,L_F D_F\log(eT)}{T}\frac{2-\alpha}{\alpha}
\notag\\
&\hspace{3.25cm}
+
\frac{2\sqrt r\,\sigma}{\sqrt B}\sqrt{\frac{\alpha}{2-\alpha}}
+
\frac{2\sqrt r\,\sigma}{\alpha T\sqrt B}
\Biggr].
\label{eq:frobenius_alpha_bound}
\end{align}
For fixed \(\alpha\), the exact optimizer in \(\eta\) is
\begin{equation}
\eta_{T,F}^\star(\alpha)
=
\sqrt{
\frac{2A_{T,F}\alpha}{L_F r\,T(4-3\alpha)}
}.
\label{eq:frobenius_eta_star_alpha}
\end{equation}
Substituting \eqref{eq:frobenius_eta_star_alpha} into
\eqref{eq:frobenius_alpha_bound} gives
\begin{align}
\frac1T\sum_{t=0}^{T-1}\EE\|\nabla f(Y_t)\|_*
\le {}&
\frac{1}{1-\beta}
\Biggl[
\sqrt{
\frac{2A_{T,F}L_F r}{T}\cdot \frac{4-3\alpha}{\alpha}
}
+
\frac{2\sqrt r\,L_F D_F\log(eT)}{T}\frac{2-\alpha}{\alpha}
\notag\\
&\hspace{3.25cm}
+
\frac{2\sqrt r\,\sigma}{\sqrt B}\sqrt{\frac{\alpha}{2-\alpha}}
+
\frac{2\sqrt r\,\sigma}{\alpha T\sqrt B}
\Biggr].
\label{eq:frobenius_plugged_alpha}
\end{align}

Choose
\begin{equation}
\alpha_{T,F}
:=
\min\left\{
1,\;
2\sqrt{\frac{A_{T,F}L_F B}{\sigma^2 T}}
\right\},
\qquad
\mu_{T,F}:=1-\alpha_{T,F},
\qquad
\eta_{T,F}:=\eta_{T,F}^\star(\alpha_{T,F}).
\label{eq:frobenius_mu_eta_choice}
\end{equation}
This is the leading-order optimal choice obtained by balancing the first and
third terms in \eqref{eq:frobenius_plugged_alpha}. With this choice,
\[
\sqrt{
\frac{2A_{T,F}L_F r}{T}\cdot \frac{4-3\alpha_{T,F}}{\alpha_{T,F}}
}
=
O\!\left(\left(\frac{A_{T,F}}{T}\right)^{1/4}\right),
\]
and
\[
\frac{2\sqrt r\,\sigma}{\sqrt B}\sqrt{\frac{\alpha_{T,F}}{2-\alpha_{T,F}}}
=
O\!\left(\left(\frac{A_{T,F}}{T}\right)^{1/4}\right).
\]
The remaining terms are lower order:
\[
\frac{2\sqrt r\,L_F D_F\log(eT)}{T}\frac{2-\alpha_{T,F}}{\alpha_{T,F}}
=
O\!\left(\sqrt{\frac{\log T}{T}}\right),
\qquad
\frac{2\sqrt r\,\sigma}{\alpha_{T,F}T\sqrt B}
=
O\!\left(\frac{1}{\sqrt{T\log T}}\right).
\]
Hence
\[
\frac1T\sum_{t=0}^{T-1}\EE\|\nabla f(Y_t)\|_*
=
O\!\left(\left(\frac{A_{T,F}}{T}\right)^{1/4}\right)
=
O\!\left(\left(\frac{\log T}{T}\right)^{1/4}\right),
\]
which proves part (i).

For part (ii), when \(\sigma=0\), the stochastic terms vanish from
\eqref{eq:frobenius_alpha_bound}. The remaining bound is increasing in \(\mu\),
hence minimized by
\[
\mu_{T,F}=0
\qquad\text{equivalently}\qquad
\alpha_{T,F}=1.
\]
Then the exact optimizer in \(\eta\) is
\begin{equation}
\eta_{T,F}
=
\sqrt{\frac{2A_{T,F}}{L_F r\,T}}.
\label{eq:frobenius_eta_star_noiseless}
\end{equation}
Substituting \(\mu_{T,F}=0\) and \eqref{eq:frobenius_eta_star_noiseless} into
Theorem~\ref{thm:Y_stationarity_frobenius_noisy_reorg} gives
\[
\frac1T\sum_{t=0}^{T-1}\|\nabla f(Y_t)\|_*
\le
\frac{1}{1-\beta}
\left[
\sqrt{\frac{2L_F r\,A_{T,F}}{T}}
+
\frac{2\sqrt r\,L_F D_F\log(eT)}{T}
\right].
\]
Since \(A_{T,F}=O(\log T)\), this implies
\[
\frac1T\sum_{t=0}^{T-1}\|\nabla f(Y_t)\|_*
=
O\!\left(\sqrt{\frac{\log T}{T}}\right),
\]
which proves part (ii).

For part (iii), suppose \(\beta=0\). Then \(Y_t=Z_t\), so
\[
S_t=0
\qquad\text{and}\qquad
Y_{t+1}-Y_t=-\eta P_t.
\]
Thus the harmonic-drift term disappears from both the descent and tracking
arguments. Repeating the proof of
Theorem~\ref{thm:Y_stationarity_frobenius_noisy_reorg} with \(S_t\equiv 0\) gives
\begin{align}
\frac1T\sum_{t=0}^{T-1}\EE\|\nabla f(Y_t)\|_*
\le {}&
\frac{\Delta}{T\eta}
+\frac{L_F r\,\eta}{2}
+\frac{2\mu L_F r\,\eta}{1-\mu}
+\frac{2\sqrt r\,\sigma}{\sqrt B}\sqrt{\frac{1-\mu}{1+\mu}}
+\frac{2\sqrt r\,\sigma}{(1-\mu)T\sqrt B}.
\label{eq:frobenius_beta_zero_bound}
\end{align}
Writing \(\alpha:=1-\mu\), this becomes
\[
\frac1T\sum_{t=0}^{T-1}\EE\|\nabla f(Y_t)\|_*
\le
\frac{\Delta}{T\eta}
+\frac{L_F r(4-3\alpha)}{2\alpha}\eta
+\frac{2\sqrt r\,\sigma}{\sqrt B}\sqrt{\frac{\alpha}{2-\alpha}}
+\frac{2\sqrt r\,\sigma}{\alpha T\sqrt B}.
\]
For fixed \(\alpha\), the exact optimizer in \(\eta\) is
\[
\eta^\star
=
\sqrt{\frac{2\Delta\,\alpha}{L_F r\,T(4-3\alpha)}}.
\]
Substituting this back and balancing the first two leading terms gives
\[
\alpha_T
\asymp
\min\left\{
1,\;
2\sqrt{\frac{\Delta L_F B}{\sigma^2 T}}
\right\},
\qquad
\mu_T:=1-\alpha_T,
\qquad
\eta_T:=\sqrt{\frac{2\Delta\,\alpha_T}{L_F r\,T(4-3\alpha_T)}}.
\]
With this choice,
\[
\frac1T\sum_{t=0}^{T-1}\EE\|\nabla f(Y_t)\|_*
=
O(T^{-1/4}),
\]
which proves \eqref{eq:frobenius_rate_beta_zero_noisy}.

If in addition \(\sigma=0\), then \eqref{eq:frobenius_beta_zero_bound} reduces to
\[
\frac1T\sum_{t=0}^{T-1}\|\nabla f(Y_t)\|_*
\le
\frac{\Delta}{T\eta}+\frac{L_F r\,\eta}{2}.
\]
Its exact optimizer is
\[
\mu_T=0,
\qquad
\eta_T=\sqrt{\frac{2\Delta}{L_F r\,T}}.
\]
Substituting yields
\[
\frac1T\sum_{t=0}^{T-1}\|\nabla f(Y_t)\|_*
\le
\sqrt{\frac{2L_F r\,\Delta}{T}}
=
O(T^{-1/2}),
\]
which proves \eqref{eq:frobenius_rate_beta_zero_noiseless}.
\end{proof}

\subsection{Steady-State Analysis for Weight Decay at $Z$}
\label{app:steady-state}

In this subsection, we prove the boundedness and steady-state characterization lemmas stated in \cref{sec:sf-wd}. For convenience, we restate the lemmas here.

\subsubsection{Proof of \cref{lem:Z-bounded}}

\textbf{Lemma~\ref{lem:Z-bounded}} (Restated).
\textit{Consider the update 
$Z_{t+1} \gets Z_t -\eta \lambda Z_t - \hat{\eta}_t\widehat{P}_t$
and assume $0 < \eta \lambda < 1$ with $\hat{\eta} = \kappa\, \eta \sqrt{mn}/\|\widehat{P}_t\|_F$ for a constant $\kappa>0$. Then for all $t$, we have 
\[
\|Z_t\|_F \leq  \|Z_0\|_F (1- \eta \lambda)^{t}+  \frac{\kappa \sqrt{mn}}{\lambda}  \implies 
\|Z_t\|_F \leq \frac{\kappa \sqrt{mn}}{\lambda} ~\textrm{as} ~ t \to \infty.
\]
Furthermore, since $X_t$ and $Y_t$ are convex combinations of $\{Z_s\}_{s \leq t}$, they satisfy the same bound.}
\begin{proof}
Triangle inequality gives
\[
\|Z_{t+1}\|_F \leq (1-\eta\lambda)\|Z_t\|_F + \hat{\eta}_t\|\widehat{P}_t\|_F = (1-\eta\lambda)\|Z_t\|_F + \kappa\,\eta\sqrt{mn}.
\]
Unrolling and bounding the geometric sum:
\[
\|Z_t\|_F \leq (1-\eta\lambda)^t\|Z_0\|_F + \kappa\,\eta\sqrt{mn}\sum_{k=0}^{t-1}(1-\eta\lambda)^k \leq (1-\eta\lambda)^t\|Z_0\|_F + \frac{\kappa\sqrt{mn}}{\lambda}.
\]
\end{proof}

\subsubsection{Proof of \cref{lem:Z-steady-state}}

\textbf{Lemma~\ref{lem:Z-steady-state}} (Quasi-Steady-State of $Z$, Restated).
\textit{Under the same conditions as \cref{lem:Z-bounded}, for $\eta\lambda \ll 1$ the steady-state satisfies
\[
\lambda\rho_t = \frac{\kappa}{2}\Bigl[-\alpha_t + \sqrt{\alpha_t^2 + 2\eta\lambda}\Bigr],
\]
where $\rho_t := \|Z_t\|_F / \sqrt{mn}$ is the RMS norm and $\alpha_t := \langle \widehat{P}_t, Z_t \rangle / (\|\widehat{P}_t\|_F \|Z_t\|_F)$ is the alignment.}

\begin{proof}
Squaring the update $Z_{t+1} = (1-\eta\lambda)Z_t - \hat{\eta}_t\widehat{P}_t$ gives
\begin{align}
\|Z_{t+1}\|_F^2 &= (1-\eta\lambda)^2\|Z_t\|_F^2 - 2(1-\eta\lambda)\hat{\eta}_t\langle Z_t, \widehat{P}_t \rangle + \hat{\eta}_t^2\|\widehat{P}_t\|_F^2. \label{eq:Z-squared-expansion}
\end{align}
The NorMuon learning rate $\hat{\eta}_t = \kappa\,\eta\sqrt{mn}/\|\widehat{P}_t\|_F$ ensures that $\hat{\eta}_t \|\widehat{P}_t\|_F = \kappa\,\eta\sqrt{mn}$ is constant. Substituting into \eqref{eq:Z-squared-expansion}:
\[
\|Z_{t+1}\|_F^2 = (1-\eta\lambda)^2\|Z_t\|_F^2 - 2\kappa\,\eta(1-\eta\lambda)\alpha_t\sqrt{mn}\|Z_t\|_F + \kappa^2\eta^2 mn.
\]
Dividing by $mn$:
\begin{align}
\rho_{t+1}^2 &= (1-\eta\lambda)^2\rho_t^2 - 2\kappa\,\eta(1-\eta\lambda)\alpha_t\rho_t + \kappa^2\eta^2. \label{eq:rho-dynamics}
\end{align}
In steady state, $\rho_{t+1} = \rho_t =: \rho_t$. Rearranging \eqref{eq:rho-dynamics}:
\[
\rho_t^2 \bigl[1 - (1-\eta\lambda)^2\bigr] = -2\kappa\,\eta(1-\eta\lambda)\alpha_t\rho_t + \kappa^2\eta^2.
\]
For $\eta\lambda \ll 1$, we have $1 - (1-\eta\lambda)^2 \approx 2\eta\lambda$ and $1 - \eta\lambda \approx 1$, giving the quadratic $\lambda \rho_t^2 + \kappa\,\alpha_t\rho_t - \tfrac{\kappa^2}{2}\eta = 0$. Taking the positive root:
\[
\lambda\rho_t = \frac{\kappa}{2}\Bigl[-\alpha_t + \sqrt{\alpha_t^2 + 2\eta\lambda}\Bigr].
\]
\end{proof}

\section{Architecture and Hyperparameter Details}\label{app:arch-hypers}

This appendix provides complete details of the model architectures and hyperparameter sweeps used in our experiments. All experiments were conducted on 8$\times$ NVIDIA A100-40GB GPUs \cite{Choquette2021} using PyTorch 2 \cite{Ansel2024}.

\paragraph{Model Architecture}

\begin{table}[htb]
\centering
\caption{Small and large Model architecture details.}
\label{tab:model-arch}
\begin{tabular}{@{}lcc@{}}
\toprule
\textbf{Hyperparameter} & \textbf{Small Model (125M)} & \textbf{Large Model (772M)} \\
\midrule
Vocabulary size & 50,304 & 50,304 \\
Number of layers & 12 & 36 \\
Hidden dimension ($d_{\text{model}}$) & 768 & 1,280 \\
Number of attention heads & 6 & 10 \\
Head dimension & 128 & 128 \\
MLP hidden dimension & 3,072 & 5,120 \\
Context length & 1,024 & 1,024 \\
Total parameters & $\sim$124M & $\sim$772M \\
\quad 1D parameters (embedding/head) & 38,633,472 & 64,389,120 \\
\quad 2D parameters (attention/MLP) & 84,934,656 & 707,788,800 \\
Batch size (global) & 512 & 512 \\
Tokens per batch & 524,288 & 524,288 \\
\bottomrule
\end{tabular}
\end{table}

We use a Llama-style GPT architecture with rotary positional embeddings (RoPE), RMSNorm, and QK-normalization in attention. The MLP uses a squared ReLU activation. Weight tying is applied between the token embedding and output head.

\paragraph{Training Regimes}

\begin{table}[htb]
\centering
\caption{Training budgets in steps and tokens for each model and regime.}
\label{tab:training-budget}
\begin{tabular}{@{}llcccc@{}}
\toprule
\textbf{Model} & & \textbf{$1\times$} & \textbf{$2\times$} & \textbf{$4\times$} & \textbf{$8\times$} \\
\midrule
Small (125M) & Steps & 5,000 & 10,000 & 20,000 & 40,000 \\
             & Tokens & 2.6B & 5.2B & 10.5B & 21.0B \\
\midrule
Large (772M) & Steps & 30,000 & 60,000 & 120,000 & --- \\
             & Tokens & 15.7B & 31.5B & 62.9B & --- \\
\bottomrule
\end{tabular}
\end{table}

We follow Chinchilla-optimal training budgets \cite{hoffmann2022}, scaling the number of training steps proportionally with model size. For the small model, we evaluate four training horizons denoted $1\times$, $2\times$, $4\times$, and $8\times$ Chinchilla-optimal. For the large model, we train at $1\times$, $2\times$, and $4\times$ Chinchilla-optimal. Table~\ref{tab:training-budget} summarizes the training budgets. All runs use a fixed global batch size of 512 sequences ($\sim$0.5M tokens), standard for models up to 1B parameters \cite{brown2020language}, and a fixed gradient-clipping threshold of 1.0; neither was swept.

\paragraph{Hyperparameter Search}

We conduct extensive hyperparameter sweeps for each optimizer using the 125M model. In particular, for AdamW we tuned the all hyper parameters except learning rate at $2\times$ Chinchilla but the \textit{learning rate was tuned for each horizon}, while for SF-AdamW we tuned at $2\times$ Chinchilla and used the same configuration for runs up to $8\times$ Chinchilla. Using the learning rate adjustment described in Section~\ref{sec:sf-spec}, SF-NorMuon was able to use an identical learning rate and weight decay as SF-AdamW. Table~\ref{tab:hyperparam-space} summarizes the search space.

\begin{table}[htb]
\centering
\caption{Hyperparameter search space for each optimizer. Values in \textbf{bold} indicate the best configuration found. The learning rate in bold is for AdamW corresponds to the $2\times$ training duration.}
\label{tab:hyperparam-space}
\begin{tabular}{@{}llp{7cm}@{}}
\toprule
\textbf{Optimizer} & \textbf{Hyperparameter} & \textbf{Values Searched} \\
\midrule
AdamW & Learning rate & \{0.002, 0.004, 0.006, \textbf{0.008}, 0.01\} \\
      & $\beta_1$ & \{\textbf{0.9}, 0.95, 0.99, 0.995\} \\
      & $\beta_2$ & \{0.9, \textbf{0.95}, 0.99, 0.995\} \\
      & Weight decay $\lambda$ & \{0, 0.05, \textbf{0.1}\} \\
      & Warmup steps & \{1000, 1500, 2000, \textbf{2500}\} \\
      & Scheduler & Cosine decay to 0 \\
\midrule
SF-AdamW & Learning rate & \{0.005, 0.008, \textbf{0.01}, 0.012, 0.015\} \\
          & $\beta_1$ & \{0.9, \textbf{0.95}, 0.99, 0.995\} \\
          & $\beta_2$ & \{0.9, 0.95, \textbf{0.99}, 0.995\} \\
         & Weight decay $\lambda$ & \{0, \textbf{0.05}, 0.1\} \\
         & Warmup steps & \{1500, \textbf{2000}, 2500\} \\
         & Decay location & \{$\mathbf{y}$, $z$\} \\
\midrule
SF-NorMuon & Learning rate & \{0.005, 0.008, \textbf{0.01}, 0.012, 0.015\} \\
            & $\beta_1$ & \{\textbf{0.9}, 0.95, 0.99, 0.995\} \\
            & $\beta_2$ & \{0.9, \textbf{0.95}, 0.99, 0.995\} \\
           & Momentum $\mu$ & \{\textbf{0.8}, 0.9, 0.95, 0.99\} \\
           & Weight decay $\lambda$ & \{0, \textbf{0.05}, 0.1\} \\
           & Warmup steps & \{1500, \textbf{2000}, 2500\} \\
           & Decay location & \{$y$, $\mathbf{z}$\} \\
\bottomrule
\end{tabular}
\end{table}

Table~\ref{tab:best-hypers} reports the best hyperparameter configuration for each optimizer on the 125M model at $8\times$ Chinchilla-optimal training. These configurations were used for all experiments conducted for the large model. 

\begin{table}[htb]
\centering
\caption{Best hyperparameter configurations for each optimizer on the small model (learning rate corresponds to $8\times$ Chinchilla for AdamW).}
\label{tab:best-hypers}
\begin{tabular}{@{}lccc@{}}
\toprule
\textbf{Hyperparameter} & \textbf{AdamW} & \textbf{SF-AdamW} & \textbf{SF-NorMuon} \\
\midrule
Learning rate & 0.004 & 0.01 & 0.01 \\
$(\beta_1, \beta_2)$ for 1D params & (0.9, 0.95) & (0.95, 0.99) & (0.95, 0.99) \\
$(\beta_1, \beta_2)$ for 2D params & (0.9, 0.95) & (0.95, 0.99) & (0.9, 0.95) \\
Momentum $\mu$ & --- & --- & 0.8 \\
Weight decay $\lambda$ & 0.1 & 0.05 & 0.05 \\
Warmup steps & 2500 & 2000 & 2000 \\
Decay location & --- & $y$ & $z$ \\
Scheduler & Cosine & None & None \\
\bottomrule
\end{tabular}
\end{table}

\paragraph{Remarks.} We make the following observations from our hyperparameter sweeps:

\begin{itemize}[leftmargin=*,itemsep=2pt,topsep=2pt] 

\item \textbf{Momentum parameters:} For SF-AdamW, $(\beta_1, \beta_2) = (0.95, 0.99)$ consistently outperforms the SF-AdamW default of $(0.9, 0.95)$. For SF-NorMuon, we use separate beta values for 1D and 2D parameters, with the NorMuon-specific momentum $\mu=0.8$ providing the best results. 

\item \textbf{Decay location:} For SF-AdamW, applying weight decay at the $y$ iterate (gradient evaluation point) works best. For SF-NorMuon, applying decay at the $z$ iterate (base sequence) achieves the lowest validation loss. 
\end{itemize}

\paragraph{Learning Rate Sweep Results}

Figures~\ref{fig:lr-sweep-adamw}--\ref{fig:lr-sweep-sf} show the validation loss curves across different learning rates for each optimizer. For all experiments, other hyperparameters are fixed at their optimal values from Table~\ref{tab:best-hypers}.

\begin{figure}[htb]
\centering
\includegraphics[width=\textwidth]{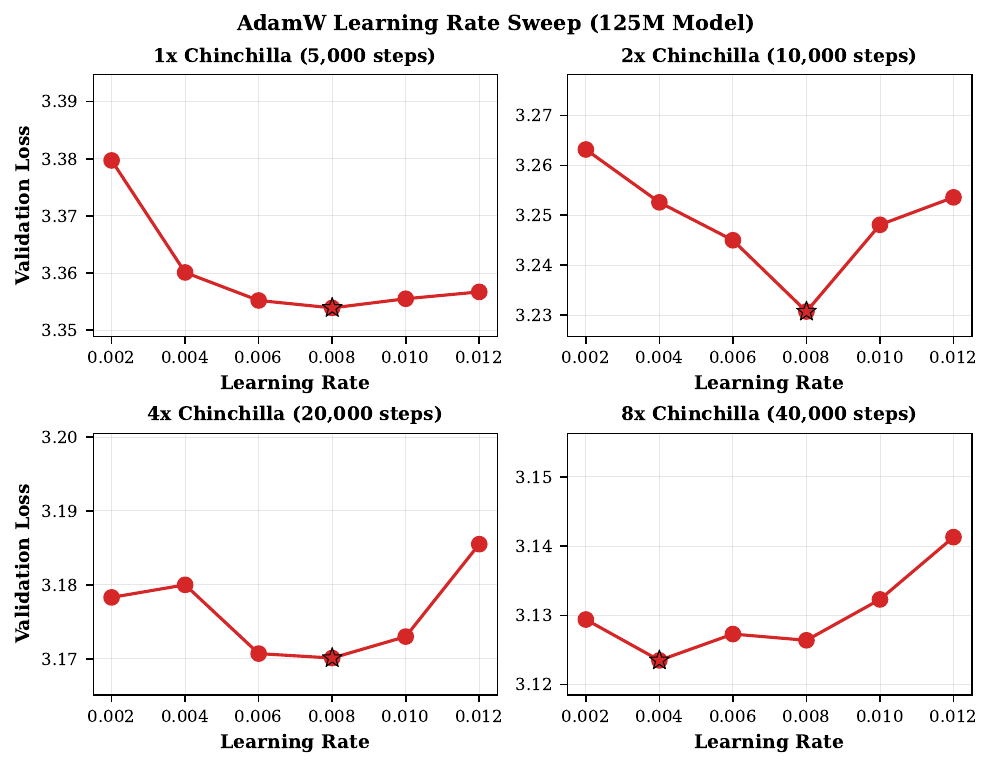}
\caption{Learning rate sweep for AdamW with cosine scheduling on the 125M model across four training horizons: $1\times$ (5K steps), $2\times$ (10K steps), $4\times$ (20K steps), and $8\times$ (40K steps) Chinchilla-optimal. The optimal learning rate is $\eta=0.008$ for $1\times$-- $4\times$ Chinchilla, and $\eta=0.004$ for $8\times$ Chinchilla, indicating that longer training benefits from smaller learning rates.}
\label{fig:lr-sweep-adamw}
\end{figure}

\section{Ablation of SF-NorMuon and Additional Results}\label{app:ablation}

We conduct ablation experiments to validate the key design choices in SF-NorMuon, compare against NorMuon with a WSD schedule and with EMA weights, and report seed-to-seed variability.

\subsection{Ablation of SF-NorMuon Components}

In \cref{fig:ablation}, we conduct ablation experiments to validate the two key modifications that distinguish SF-NorMuon from schedule-free spectral descent: the explicit momentum buffer and the row-wise adaptive normalization. All experiments use the 125M parameter model trained for 8$\times$ Chinchilla tokens ($\approx$21B tokens) with decay at $Z$.

\paragraph{Explicit momentum.} As discussed in \S\ref{sec:sf-spec}, the polar update treats all singular directions uniformly, which can be more aggressive than coordinate-wise methods. We therefore introduce an explicit momentum buffer $M_t = \mu M_{t-1} + (1-\mu)G_t$ to smooth the gradient before computing its polar factor. The left panel of \cref{fig:ablation} compares SF-NorMuon ($\mu=0.8$) against the ablation with no explicit momentum ($\mu=0$). Removing momentum increases the final validation loss from 3.14 to 3.28, a rather large degradation. This confirms that explicit momentum is essential for stable, effective training with spectral updates.

\paragraph{Row-wise normalization.} The polar factor can produce updates with high variance across neurons, as different rows of a weight matrix may have vastly different typical magnitudes. Row-wise adaptive normalization addresses this by maintaining a running average of squared row norms and normalizing accordingly. The right panel of \cref{fig:ablation} compares SF-NorMuon against the variant without row normalization on a tighter y-axis scale to reveal the small but consistent improvement. SF-NorMuon achieves a final loss of 3.136 compared to 3.142 for SF-Muon, and reaches the same loss approximately 12\% faster in terms of tokens processed. This speedup is quite similar to the speedup achieved by NorMuon relative to Muon \cite{li2025normuonmakingmuonefficient}. 

\subsection{Comparison with NorMuon under a WSD Schedule}\label{app:wsd}

Warmup-stable-decay (WSD) schedules \cite{hu2024WSD, hagele2024scalinglawscomputeoptimaltraining} allow training to be continued from a checkpoint on the stable (constant learning rate) branch, and are therefore a natural horizon-flexible alternative to cosine decay. We train the 125M model with NorMuon under WSD, following standard configurations \cite{dremov2025trainingdynamicscooldownstage, ma2026wsqdhorizonfreelearningrate} with a 300-step linear warmup and a linear cooldown over the final 20\% of each target horizon. We evaluate peak learning rates $\eta=0.008$, the best value in our cosine NorMuon sweep, and $\eta=0.004$, motivated by prior evidence that WSD favors a peak learning rate roughly half the optimal cosine one \cite{hagele2024scalinglawscomputeoptimaltraining}. For each learning rate, we branch from the stable checkpoints at steps 4,000 and 8,000 to obtain the $1\times$ and $2\times$ Chinchilla endpoints. Results are reported in \cref{tab:wsd}.

\begin{table}[htb]
\centering
\caption{Final validation loss for the 125M model including NorMuon with a WSD schedule. WSD endpoints are obtained by cooling down from a single stable branch.}
\label{tab:wsd}
\begin{tabular}{lcccccc}
\toprule
Horizon & SF-NorMuon & SF-AdamW & NorMuon & AdamW & NorMuon-WSD & NorMuon-WSD \\
 & & & (cosine) & (cosine) & ($\eta=0.004$) & ($\eta=0.008$) \\
\midrule
1$\times$ (2.6B tokens) & 3.318 & 3.399 & \textbf{3.292} & 3.354 & 3.310 & 3.317 \\
2$\times$ (5.2B tokens) & 3.229 & 3.282 & \textbf{3.200} & 3.231 & 3.237 & 3.247 \\
\bottomrule
\end{tabular}
\end{table}

At both horizons, WSD underperforms the tuned cosine schedule, consistent with the independent finding in \cite{kimiteam2026kimik3openfrontier} that cosine decay achieves lower final loss than WSD under separately tuned hyperparameters. At $2\times$ Chinchilla, the cooled-down WSD weights are outperformed not only by cosine NorMuon but also by SF-NorMuon. WSD is thus not a fully satisfying anytime solution: every deployable checkpoint still requires a separate cooldown branch, and, more importantly, the optimal peak learning rate itself depends on the horizon, so a value tuned for a shorter run becomes suboptimal once training is extended \cite{ma2026wsqdhorizonfreelearningrate, kimiteam2026kimik3openfrontier}. In contrast, SF-NorMuon returns the deployable averaged iterate at every step with a single configuration.

\subsection{Comparison with NorMuon and EMA Weights}\label{app:ema}

A constant learning rate combined with an exponential moving average (EMA) of the weights, $\bar{W}_{t} = \omega \bar{W}_{t-1} + (1-\omega) W_t$, is another horizon-free alternative, and uses the same $\sim 3mn$ memory as SF-NorMuon (\cref{rem:memory}). Good EMA performance requires adjusting the decay to the batch size and training duration \cite{zhang2025doescriticalbatchsize}, so we test the learning rates and decays recommended for our batch size in \cite{zhang2025doescriticalbatchsize} on the 125M model at $1\times$ Chinchilla (\cref{tab:ema}).

\begin{table}[htb]
\centering
\caption{Final EMA validation loss for NorMuon with a constant learning rate after warmup (125M model, $1\times$ Chinchilla). For reference, cosine NorMuon reaches 3.292 and SF-NorMuon 3.318.}
\label{tab:ema}
\begin{tabular}{ccc}
\toprule
Learning rate & EMA decay $\omega$ & Validation loss \\
\midrule
0.00158 & 0.998 & 3.349 \\
0.00158 & 0.99 & 3.352 \\
0.00316 & 0.998 & 3.346 \\
0.00316 & 0.99 & 3.347 \\
\bottomrule
\end{tabular}
\end{table}

The best EMA configuration (3.346) is behind both cosine NorMuon and SF-NorMuon, indicating that EMA averaging alone does not remove the need for horizon-sensitive tuning in this setting.

\subsection{Reproducibility Across Seeds}\label{app:seeds}

A single 772M run at $4\times$ Chinchilla requires 120,000 optimizer steps, making seed repetitions for every configuration prohibitively expensive. To quantify run-to-run variability, we repeat all four optimizers with two additional seeds for the 125M model at $1\times$ Chinchilla (\cref{tab:seeds}).

\begin{table}[htb]
\centering
\caption{Validation loss for the 125M model at $1\times$ Chinchilla across three seeds.}
\label{tab:seeds}
\begin{tabular}{lccccc}
\toprule
Optimizer & Seed 1337 & Seed 1338 & Seed 1339 & Mean & Std. \\
\midrule
AdamW & 3.354 & 3.348 & 3.348 & 3.350 & $\pm 0.003$ \\
NorMuon & 3.292 & 3.305 & 3.304 & 3.300 & $\pm 0.006$ \\
SF-AdamW & 3.399 & 3.398 & 3.403 & 3.400 & $\pm 0.003$ \\
SF-NorMuon & 3.318 & 3.320 & 3.316 & 3.318 & $\pm 0.002$ \\
\bottomrule
\end{tabular}
\end{table}

The within-setting standard deviation is at most $0.006$ nats, whereas the gaps between SF-NorMuon and AdamW ($0.032$) and between SF-NorMuon and NorMuon ($0.018$) are several times larger, and the ordering of the optimizers is preserved on every seed. The reported differences are therefore not artifacts of validation-loss noise.  

\newpage 

\section{Training Vision Transformers}\label{app:vit-hypers}

To evaluate SF-NorMuon beyond language modeling, we train a ViT-B/16 \cite{dosovitskiy2021imageworth16x16words} on ImageNet-1K \cite{Deng2009} for 30K steps, comparing all three optimizer families under matched conditions. The architecture and full hyperparameter search are detailed below.

\begin{figure}[t]
\centering
\begin{minipage}[b]{0.5\textwidth}
\centering
\includegraphics[width=\textwidth]{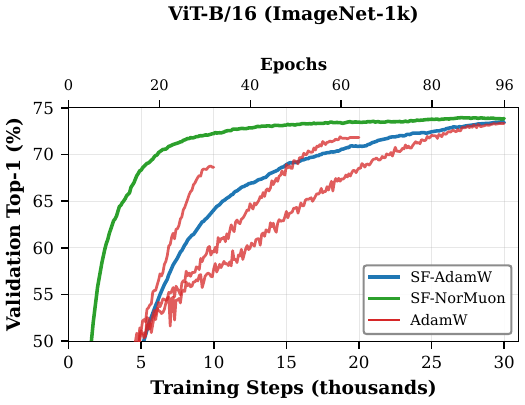}
\captionof{figure}{Best validation accuracy curve for each optimizer family on ImageNet-1K (ViT-B/16, 30K steps).}
\label{fig:vit-milestones}
\end{minipage}\hfill
\begin{minipage}[b]{0.47\textwidth}
\centering
\footnotesize
\setlength{\tabcolsep}{4pt}
\begin{tabular}{@{}c ccc c@{}}
\toprule
Milestone & SF-NorMuon & SF-AdamW & AdamW & Speedup \\
 (steps)  & (decay@$Z$) & (decay@$Y$) & (cosine) & \\
\midrule
10K & \textbf{72.22} & 64.05 & 68.75 & 65\% \\
20K & \textbf{73.46} & 70.87 & 71.89 & 65\% \\
30K & \textbf{73.84} & 73.44 & 73.32 & 40\% \\
\bottomrule
\end{tabular}
\captionof{table}{\textbf{Top-1 validation accuracy (\%) at training milestones on ImageNet-1K (ViT-B/16, 30K steps).} The best accuracy at each milestone is shown in \textbf{bold}. \emph{Speedup} is the percentage of steps saved by SF-NorMuon to reach the final accuracy of SF-AdamW at that milestone.}
\label{tab:vit-milestones}
\end{minipage}
\end{figure}

\textbf{Optimizer split and hyperparameter tuning.}
Following the same two-optimizer paradigm as the language model experiments, \emph{opt1} (scalars, biases, class token, position embeddings, and the classification head) always uses SF-AdamW, while \emph{opt2} (all remaining 2D matrix parameters) uses the spectral update. We sweep learning rate, weight decay, warmup, and matrix momentum for all three families; the search space and best configurations are reported below.

\textbf{Results.}
As shown in \cref{fig:vit-milestones} and \cref{tab:vit-milestones}, SF-NorMuon is the strongest optimizer, reaching \textbf{73.84\%} top-1 against 73.44\% for SF-AdamW and 73.32\% for AdamW, and it leads at every milestone without any learning rate schedule. This mirrors the language modeling findings, with weight decay at the fast iterate $Z$ again best for SF-NorMuon. To confirm the correctness of our implementation, we exactly replicated the original ViT-B/16 recipe \cite{dosovitskiy2021imageworth16x16words}, recovering its standard 77.9\% top-1 accuracy at 90K steps (following the data processing, optimizer settings, training settings and post training higher resolution fine-tuning mentioned in this paper).

Below we provide complete details of the ViT-B/16 architecture, training setup, and hyperparameter sweeps used in the ImageNet experiments. All runs were conducted on 8$\times$ NVIDIA A100 GPUs using PyTorch 2 \cite{Ansel2024} with bf16 mixed precision and \texttt{torch.compile}.

\paragraph{Model Architecture}

\begin{table}[htb]
\centering
\caption{ViT-B/16 architecture details.}
\label{tab:vit-arch}
\begin{tabular}{@{}lc@{}}
\toprule
\textbf{Hyperparameter} & \textbf{ViT-B/16} \\
\midrule
Image size & $224 \times 224$ \\
Patch size & $16 \times 16$ \\
Hidden dimension ($d_{\text{model}}$) & 768 \\
Number of layers & 12 \\
Number of attention heads & 12 \\
MLP ratio & 4 \\
MLP activation & GELU \\
Attention & Scaled dot-product (SDPA) \\
Position embedding & Learned \\
Classification token & \texttt{[CLS]} \\
Normalization & Pre-norm (LayerNorm) \\
Classification head init & Zero \\
Dropout & 0.1 \\
Total parameters & 86.57M \\
\bottomrule
\end{tabular}
\end{table}

\paragraph{Optimizer Split}

As in the language model experiments, we split the parameters into two optimizer groups. \emph{opt1} receives all scalar parameters, biases, the class token, position embeddings, and the classification head; \emph{opt2} receives all remaining 2D matrix parameters (patch embedding convolution, attention QKV and output projections, and MLP layers). For schedule-free runs, opt1 always uses SF-AdamW regardless of which optimizer is applied to opt2. Schedule-free optimizers are placed in train mode during parameter updates and eval mode for validation; model selection uses the validation accuracy at the $x$ iterate.

\paragraph{Dataset and Data Processing}

We use the full ImageNet-1K dataset in torchvision \texttt{ImageFolder} layout. Training augmentation consists of \texttt{RandomResizedCrop} to $224 \times 224$ (scale $(0.05, 1.0)$, bilinear interpolation) and random horizontal flip, with symmetric normalization $(x - 0.5)/0.5$. Validation uses resize to $256$, center crop to $224$, and the same normalization. Color jitter is disabled for all reported runs. Mixup is available in the codebase but is also disabled for the reported runs. We use label smoothing of 0.1 throughout. The global batch size is 4096 with a per-device batch size of 128, yielding 4 gradient-accumulation microsteps on 8 GPUs. Validation is performed every 100 steps on the full validation set.

\paragraph{Training Duration}

All runs in the tuning batch train for 30K gradient steps. The AdamW baseline uses linear warmup followed by cosine decay to zero; the schedule-free methods use no explicit learning rate schedule. Warmup length is part of the sweep, ranging over $\{2\text{K}, 3\text{K}, 4\text{K}\}$ steps for all three families. Validation is performed on the full ImageNet validation set every 100 steps, and schedule-free model selection uses the validation accuracy at the $x$ iterate.

\paragraph{Hyperparameter Search}

We tune all three optimizer families with a focused 30K-step sweep, summarized in \cref{tab:vit-hp-space}. The AdamW baseline uses a short-warmup grid over learning rate, $\beta_2$, and warmup length. The schedule-free methods sweep learning rate, weight decay, and warmup jointly, with SF-NorMuon inheriting the learning rate, weight decay, and warmup of the selected SF-AdamW parents and additionally sweeping the matrix momentum $\mu$.

\begin{table}[htb]
\centering
\caption{Hyperparameter search space for each optimizer in the 30K-step ViT-B/16 tuning batch. Values in \textbf{bold} indicate the best configuration found.}
\label{tab:vit-hp-space}
\begin{tabular}{@{}llp{7cm}@{}}
\toprule
\textbf{Optimizer} & \textbf{Hyperparameter} & \textbf{Values Searched} \\
\midrule
AdamW & Learning rate & \{\textbf{0.003}, 0.005, 0.010\} \\
      & $\beta_1$ & \{\textbf{0.9}\} \\
      & $\beta_2$ & \{0.95, \textbf{0.99}, 0.999\} \\
      & Weight decay $\lambda$ & \{\textbf{0.30}\} \\
      & Warmup steps & \{2,000, 3,000, \textbf{4,000}\} \\
      & Scheduler & Linear warmup then cosine decay to 0 \\
      & Dropout & \textbf{0.1} \\
\midrule
SF-AdamW & Learning rate & \{0.005, 0.008, 0.010, 0.015, \textbf{0.018}, 0.020, 0.025\} \\
         & $\beta_1$ & \{\textbf{0.95}\} \\
         & $\beta_2$ & \{\textbf{0.99}\} \\
         & Weight decay $\lambda$ & \{\textbf{0.15}, 0.20, 0.25, 0.30\} \\
         & Warmup steps & \{\textbf{2,000}, 3,000, 4,000\} \\
         & Decay location & \{$\mathbf{y}$, $z$\} \\
         & Dropout & \textbf{0.1} \\
\midrule
SF-NorMuon & Learning rate & \{0.008, \textbf{0.015}, 0.018, 0.025\} \\
           & $\beta_1$ & \{\textbf{0.95}\} \\
           & $\beta_2$ & \{\textbf{0.99}\} \\
           & Momentum $\mu$ & \{0.8, \textbf{0.9}\} \\
           & Weight decay $\lambda$ & \{0.15, 0.20, \textbf{0.25}\} \\
           & Warmup steps & \{2,000, \textbf{4,000}\} \\
           & Decay location & \{$y$, $\mathbf{z}$\} \\
           & Dropout & \textbf{0.1} \\
\bottomrule
\end{tabular}
\end{table}

\cref{tab:vit-best-hypers} reports the best hyperparameter configuration for each optimizer in the 30K-step batch. Notably, SF-NorMuon uses an equal learning rate for both opt1 and opt2 groups, consistent with the rescaling approach described in \cref{sec:sf-spec}.

\begin{table}[htb]
\centering
\caption{Best hyperparameter configurations for each optimizer in the 30K-step ViT-B/16 tuning batch (ImageNet-1K), with the corresponding validation top-1 at 30K steps.}
\label{tab:vit-best-hypers}
\begin{tabular}{@{}lccc@{}}
\toprule
\textbf{Hyperparameter} & \textbf{AdamW} & \textbf{SF-AdamW} & \textbf{SF-NorMuon} \\
\midrule
Learning rate $\eta$ & 0.003 & 0.018 & 0.015 \\
$(\beta_1, \beta_2)$ & (0.9, 0.99) & (0.95, 0.99) & (0.95, 0.99) \\
Momentum $\mu$ (opt2) & --- & --- & 0.9 \\
Weight decay $\lambda$ & 0.3 & 0.15 & 0.25 \\
Warmup steps & 4,000 & 2,000 & 4,000 \\
Decay location & --- & $y$ & $z$ \\
Scheduler & Cosine & None & None \\
Dropout & 0.1 & 0.1 & 0.1 \\
Label smoothing & 0.1 & 0.1 & 0.1 \\
Val top-1 at 30K (\%) & 73.32 & 73.44 & 73.84 \\
\bottomrule
\end{tabular}
\end{table}

\paragraph{Remarks.} We highlight several observations from the 30K-step ViT tuning batch:

\begin{itemize}[leftmargin=*,itemsep=2pt,topsep=2pt]

\item \textbf{Best family:} SF-NorMuon is the strongest optimizer family, reaching 73.84\% top-1 at 30K steps, ahead of SF-AdamW (73.44\%) and the short-warmup AdamW grid (73.32\%). Both schedule-free methods match or exceed the AdamW reference at 30K steps without a learning rate schedule.

\item \textbf{Learning rate:} The AdamW baseline performs best at $\eta = 0.003$, consistent with standard ViT-B/16 training recipes. Schedule-free methods benefit from substantially higher learning rates ($\eta = 0.015$ to $0.018$), which is expected given the absence of a decaying schedule.

\item \textbf{Weight decay and Location:} AdamW uses $\lambda = 0.3$, consistent with the original ViT training recipe \cite{dosovitskiy2021imageworth16x16words}. Schedule-free methods prefer moderate weight decay ($\lambda = 0.15$ to $0.25$). As in the language model experiments, SF-AdamW performs best with weight decay at the $y$ iterate, while SF-NorMuon achieves the best accuracy with decay at the $z$ iterate. This confirms that the decay location finding generalizes across modalities.


\end{itemize}

\newpage

\section{Reference PyTorch Implementation}
\label{app:implementation}

Below we provide a self-contained, distributed PyTorch implementation of SF-NorMuon, closely following \cref{algo:sf-normuon}. Our implementation is based on the code from \cite{defazio2024road,jordan2024muon_repo}. Because the Newton-Schulz orthogonalization is sharded across data-parallel ranks, the spectral update adds little cost per step. \cref{tab:wallclock} reports the mean time per optimizer step for the 125M model; SF-NorMuon incurs only $\sim 2\%$ overhead relative to AdamW.

\begin{table}[htb]
\centering
\caption{Mean time per optimizer step for the 125M model on 8$\times$A100-40GB (batch size 512, sequence length 1024, eager mode, averaged over 200 steps).}
\label{tab:wallclock}
\begin{tabular}{lcc}
\toprule
Optimizer & Time / step (ms) & Overhead vs.\ AdamW \\
\midrule
AdamW & 769.1 & --- \\
NorMuon & 771.1 & +0.26\% \\
SF-AdamW & 772.4 & +0.43\% \\
SF-NorMuon & 785.5 & +2.13\% \\
\bottomrule
\end{tabular}
\end{table}

\paragraph{Implementation notes.}
\begin{itemize}[leftmargin=*]
    \item The \texttt{zeropower\_via\_newtonschulz5} function computes an approximation to the polar factor $P = \polar(M)$ using 5 iterations of the Newton-Schulz method with coefficients $(a, b, c) = (3.4445, -4.7750, 2.0315)$. The function handles both tall and wide matrices by transposing when necessary. There are other alternative methods to accomplish the polar transformation which either use randomized methods \cite{ahn2025dion,ahn2025dion2simplemethodshrink} or differ in the choice of coefficients \cite{amsel2026polarexpressoptimalmatrix}. 
    
    \item In addition to the live weights \texttt{p} (the gradient evaluation sequence $Y_t$), the optimizer maintains three state variables per parameter: \texttt{z} (the fast iterate $Z_t$), \texttt{v} (the row-wise second moment estimate $v_t$), and \texttt{mom} (the momentum buffer $M_t$).
    
    \item The \texttt{eval()} method computes $X_t$ on the fly using $Z_t$ and $Y_t$. The \texttt{train()} method restores \texttt{p} to the $Y_t$ iterate for further training. 
    
    \item Weight decay is applied to the fast iterate $Z_t$, consistent with our analysis in \cref{sec:sf-wd}, and the learning rate scaling factor \texttt{eta\_scale} (default 0.2) normalizes the effective step size to be comparable to Adam's RMS-normalized updates.

    \item Following the NorMuon \cite{li2025normuonmakingmuonefficient}, the Newton-Schulz orthogonalization is sharded across the data-parallel ranks which minimizes wall-clock overhead of SF-NorMuon compared to AdamW.
    
\end{itemize}

\clearpage
\newgeometry{margin=0.835in}
\begin{lstlisting}[language=Python, basicstyle=\ttfamily\scriptsize]
import math, torch
import torch.distributed as dist
@torch.compile
def zeropower_via_newtonschulz5(G, steps=5, eps=1e-7):
    assert len(G.shape) == 2
    a, b, c = (3.4445, -4.7750, 2.0315)
    X = G.bfloat16(); X /= (X.norm() + eps)
    transpose = G.size(0) > G.size(1)
    if transpose: X = X.T
    for _ in range(steps):
        A = X @ X.T; B = A @ X
        X = a * X + b * B + c * A @ B
    return X.T if transpose else X

class NorMuonScheduleFree(torch.optim.Optimizer):
    def __init__(self, params, lr=0.005, betas=(0.9, 0.95), momentum=0.8,
                 eps=1e-8, weight_decay=0.1, warmup_steps=2000, eta_scale=0.2):
        params = sorted(list(params), key=lambda x: x.numel(), reverse=True)  # align gather rounds
        defaults = dict(lr=lr, betas=betas, momentum=momentum, eps=eps,
                        weight_decay=weight_decay, warmup_steps=warmup_steps,
                        eta_scale=eta_scale, k=0, train_mode=False, weight_sum=0.0)
        super().__init__(params, defaults)
    @torch.no_grad()
    def eval(self):  # X_t = (Y_t - (1-beta) Z_t) / beta
        for group in self.param_groups:
            if group["train_mode"]:
                beta = group["betas"][0]
                for p in group["params"]:
                    if "z" in self.state[p]: p.lerp_(end=self.state[p]["z"], weight=1.0 - 1.0 / beta)
                group["train_mode"] = False
    @torch.no_grad()
    def train(self):  # restore Y_t
        for group in self.param_groups:
            if not group["train_mode"]:
                beta = group["betas"][0]
                for p in group["params"]:
                    if "z" in self.state[p]: p.lerp_(end=self.state[p]["z"], weight=1.0 - beta)
                group["train_mode"] = True
    @torch.no_grad()
    def step(self, closure=None):
        loss = closure() if closure else None
        on = dist.is_available() and dist.is_initialized() and dist.get_world_size() > 1
        world, rank = (dist.get_world_size(), dist.get_rank()) if on else (1, 0)
        for group in self.param_groups:
            beta, beta2 = group["betas"]
            mu, eps, eta_scale, decay = group["momentum"], group["eps"], group["eta_scale"], group["weight_decay"]
            k, warmup_steps = group["k"], group["warmup_steps"]
            lr = group["lr"] * ((k + 1) / warmup_steps if k < warmup_steps else 1.0)
            weight = lr * lr; group["weight_sum"] += weight; ckp1 = weight / group["weight_sum"]
            params = group["params"]
            # momentum update + lazy state init (replicated on every rank)
            for p in params:
                if p.grad is None: continue
                state = self.state[p]
                if "z" not in state:
                    state["z"], state["mom"] = p.clone(), torch.zeros_like(p)
                    state["v"] = torch.zeros(p.shape[0], device=p.device, dtype=torch.float32)
                state["mom"].mul_(mu).add_(p.grad, alpha=1.0 - mu)
            # sharded polar factor: each rank does 1/world of the matrices, then all_gather
            P = [torch.empty_like(p) for p in params]
            pad = (world - len(params) % world) % world
            pp, Pp = params + [params[-1]] * pad, P + [torch.empty_like(params[-1]) for _ in range(pad)]
            for b in range(0, len(pp), world):
                i = b + rank
                if i < len(params) and pp[i].grad is not None:
                    Pp[i].copy_(zeropower_via_newtonschulz5(self.state[pp[i]]["mom"]).to(pp[i].dtype))
                else: Pp[i].zero_()
                if on: dist.all_gather(Pp[b:b + world], Pp[i])
            # schedule-free update (replicated; uses the gathered polar factor)
            for i, p in enumerate(params):
                if p.grad is None: continue
                z, v = self.state[p]["z"], self.state[p]["v"]
                v.mul_(beta2).add_((P[i] * P[i]).mean(dim=1).float(), alpha=1.0 - beta2)
                Phat = P[i] / (v.sqrt() + eps).to(P[i].dtype).unsqueeze(1)
                m, n = p.shape
                eta_hat = eta_scale * lr * math.sqrt(m * n) / Phat.float().norm().clamp_min(1e-12)
                x_t = (p - (1.0 - beta) * z) / beta  # recover X_t
                if decay != 0.0: z.mul_(1.0 - lr * decay)
                z.sub_(eta_hat * Phat)  # z is now Z_{t+1}
                p.copy_((1.0 - beta) * z + beta * ((1.0 - ckp1) * x_t + ckp1 * z))
            group["k"] = k + 1
        return loss
\end{lstlisting}
\restoregeometry
\clearpage

\end{document}